\documentclass[mnsc,nonblindrev]{informs3_homepage}
\usepackage[utf8]{inputenc}
\usepackage[T1]{fontenc}
\usepackage[english]{babel}
\usepackage{microtype}
\usepackage{mathrsfs,mathtools,accents}
\usepackage{xifthen,enumitem,comment}
\setlist[enumerate,1]{label=\normalfont{(\Roman*)},leftmargin=*}
\usepackage{xcolor}
\usepackage{etoolbox}
\makeatletter
\patchcmd{\env@cases}{1.2}{0.96}{}{}
\makeatother
\usepackage{algorithm,algorithmicx,algpseudocode}
\usepackage[hypertexnames=false]{hyperref}
\hypersetup{%
  colorlinks=true,%
  linkcolor={blue!66!black},%
  linkbordercolor=red,%
  citecolor={blue!66!black},
}
\usepackage{tikz,pgfplots,subfigure}
\usetikzlibrary{matrix,decorations.pathreplacing}

\ifx\argmax\undefined\DeclareMathOperator*{\argmax}{arg\,max}\fi
\ifx\argmin\undefined\DeclareMathOperator*{\argmin}{arg\,min}\fi

\providecommand{\inner}[2]{\left\langle#1,#2\right\rangle}

\newcommand*{\QED}{%
\leavevmode\unskip\penalty9999 \hbox{}\nobreak\hfill
    \quad\hbox{$\square$}%
}
\newcommand*{\QEG}{%
\leavevmode\unskip\penalty9999 \hbox{}\nobreak\hfill
    \quad\hbox{$\clubsuit$}%
}

\providecommand{\1}{\boldsymbol{1}}

\renewcommand{\Pr}{\mathbb{P}}
\providecommand{\E}{\mathbb{E}}

\providecommand{\N}{\mathbb{N}}

\providecommand{\cE}{\mathcal{E}}
\providecommand{\cF}{\mathcal{F}}
\providecommand{\cG}{\mathcal{G}}
\providecommand{\cH}{\mathcal{H}}

\providecommand{\cN}{\mathcal{N}}

\providecommand{\cX}{\mathcal{X}}
\providecommand{\cY}{\mathcal{Y}}

\OneAndAHalfSpacedXI

\usepackage{lmodern}
\usepackage{tikz}
\usepackage{caption}
\usepackage{graphicx}
\numberwithin{equation}{section}
\usepackage{algorithm,algpseudocode}

\usepackage{natbib}
\bibpunct[, ]{(}{)}{,}{a}{}{,}%
\def\bibfont{\small}%

\TheoremsNumberedThrough     %
\ECRepeatTheorems

\EquationsNumberedBySection %

\renewcommand{\Pr}{\mathbb{P}}
\providecommand{\E}{\mathbb{E}}

\providecommand{\N}{\mathbb{N}}

\providecommand{\bX}{\boldsymbol{X}}
\providecommand{\bvartheta}{\boldsymbol{\vartheta}}
\providecommand{\balpha}{\boldsymbol{\alpha}}
\providecommand{\bTheta}{\boldsymbol{\Theta}}

\providecommand{\cG}{\mathcal{G}}

\providecommand{\cE}{\mathcal{E}}
\providecommand{\cF}{\mathcal{F}}
\providecommand{\cX}{\mathcal{X}}
\providecommand{\cY}{\mathcal{Y}}
\providecommand{\cH}{\mathcal{H}}

\providecommand{\cN}{\mathcal{N}}

\providecommand{\by}{\boldsymbol{y}}
\providecommand{\bz}{\boldsymbol{z}}
\providecommand{\bZ}{\boldsymbol{Z}}

\providecommand{\btheta}{\boldsymbol{\theta}}

\providecommand{\botheta}{\mathfrak{r}}

\providecommand{\bz}{\boldsymbol{z}}

\newcommand{\htheta}{\hat{\theta}}

\providecommand{\lowp}{\underline{p}}
\providecommand{\lowbeta}{\underline{\beta}_{\Theta}}

\providecommand{\compgran}{comparison granularity}
\providecommand{\searchint}{search interval}
\providecommand{\twostage}{\textsf{SL+LHF}}
\providecommand{\ftpb}{\textsf{MAPB}}

\MANUSCRIPTNO{}

\begin{document}

	\RUNAUTHOR{Junyu Cao and Mohsen Bayati}

	\RUNTITLE{A Probabilistic Approach for Model Alignment with Human Comparisons}

	\TITLE{A Probabilistic Approach for Model Alignment with Human Comparisons}

	\ARTICLEAUTHORS{%
		\AUTHOR{Junyu Cao}

		\AFF{McCombs School of Business, The University of Texas at Austin, \EMAIL{junyu.cao@mccombs.utexas.edu}}

		\AUTHOR{Mohsen Bayati}
		\AFF{
			Graduate School of Business, Stanford University, \EMAIL{bayati@stanford.edu}}
	}

	\ABSTRACT{
		A growing trend involves integrating human knowledge into learning frameworks, leveraging subtle human feedback to refine AI models. While these approaches have shown promising results in practice, the theoretical understanding of when and why such approaches are effective remains limited. This work takes steps toward developing a theoretical framework for analyzing the conditions under which human comparisons can enhance the traditional supervised learning process. Specifically, this paper studies the effective use of noisy-labeled data and human comparison data to address challenges arising from noisy environment and high-dimensional models. We propose a two-stage ``Supervised Learning+Learning from Human Feedback'' (SL+LHF) framework that connects machine learning with human feedback through a probabilistic bisection approach. The two-stage framework first learns low-dimensional representations from noisy-labeled data via an SL procedure and then uses human comparisons to improve the model alignment. To examine the efficacy of the alignment phase, we introduce a concept, termed the ``label-noise-to-comparison-accuracy'' (LNCA) ratio. This paper identifies from a theoretical perspective the conditions under which the ``SL+LHF'' framework outperforms the pure SL approach; we then leverage this LNCA ratio to highlight the advantage of incorporating human evaluators in reducing sample complexity. We validate the framework on a real high-dimensional crowdfunding-prediction task: under a fixed query budget, trading labels for comparisons improves accuracy precisely when labels are scarce, and the findings hold when the evaluator is replaced by real large language models. A study conducted via Amazon Mechanical Turk (MTurk) further validates the model primitives.

	}%


	\maketitle

	\section{Introduction}

	In light of the increasing availability of data, the demand for data-driven algorithmic solutions has surged. Machine learning has emerged as the standard technique for numerous tasks, including natural language processing and computer vision, among others. Various streams of operations research also demonstrate its importance.%

	To model increasingly complex environments, machine learning models have grown in complexity, often leading to over-parameterization. While this increased complexity allows models to capture intricate patterns, it has revealed fundamental limitations in the standard training and testing pipeline. Specifically, the traditional supervised learning (SL) approach of training on noisy-labeled data faces significant challenges due to model misspecification \citep{d2022underspecification}. This issue becomes particularly acute in high-complexity scenarios, such as when the feature dimension is comparable to or exceeds the number of training data points. In such high-dimensional settings, fundamentally different models may achieve similar performance metrics during testing, despite having  different underlying structures and behaviors. For instance, as demonstrated by \cite{overman2024aligning}, the trained model that minimizes the empirical risk may violate fundamental principles, such as the natural expectation of negative demand elasticity—indicating that the model is not \emph{aligned} with its expected behavior. Consequently, decision makers must consider strategic approaches to achieve better model alignment across different environments.

	One promising approach has been the incorporation of human feedback into the learning framework \citep{li2016dialogue,christiano2017deep, alurhuman}. These studies demonstrate that using human feedback for alignment can effectively address the limitations of traditional training approaches. For instance, in the design of chatbots, performance can be adaptively improved through the chatbots' interactions with humans \citep{ziegler2019fine,ouyang2022training}. \cite{ouyang2022training} exemplify this approach in their work on fine-tuning language models, where InstructGPT is fine-tuned with human feedback. The effectiveness of human feedback stems from evaluators' domain expertise that extends beyond what can be encoded into a dataset or estimation method. Humans' deep understanding of the task domain enables them to make informed comparisons that consider context, relevance, and real-world applicability \citep{deng2020integrating}.

	Among various forms of human feedback, comparisons between alternatives have emerged as a particularly effective approach due to their cognitive simplicity. As Carl Jung observes, ``Thinking is difficult, therefore let the herd pronounce judgment''. Similarly, humans find it easier to select between alternatives than to generate absolute evaluations. Recent work has demonstrated the practical value of such comparative feedback: \cite{zhou2020learning} use pair-wise comparisons for story generation, while \cite{xu2023shattering} leverage human comparisons between large language models (LLMs) for fine-tuning. Although these applications show promising results, we lack a theoretical framework to understand when and why human comparisons enhance traditional supervised learning. This gap motivates our paper's central question:
	\begin{quote}
		\centering{
			\emph{How can we effectively combine supervised learning and human feedback to achieve model alignment?}}
	\end{quote}

	To make progress towards answering this question, we consider a two-stage probabilistic model that helps us analyze and understand the interplay between supervised learning and human feedback. In the first stage, noisy labeled data are fed into the SL oracle to obtain low-dimensional representations (or embeddings). For example, selecting a small set of features through Lasso is a way to represent high-dimensional feature spaces; similarly, low-rank matrix estimation is an example of representation of high-dimensional matrices. The second stage uses human comparisons to improve model alignment within this low-dimensional space. Our framework uses the random utility model to describe humans' comparison behavior. Practically, humans may provide incorrect responses, so we develop a probabilistic bisection approach for the alignment phase. We call this two-stage process the Supervised Learning+Learning from Human Feedback (SL+LHF) model.

	The rationale behind this model is rooted in compelling evidence that representation learning techniques, such as deep learning architectures \citep{zhou2017anomaly,akrami2019robust}, can effectively learn low-dimensional features even from highly noisy data \citep{li2021learning,taghanaki2021robust,mousavi2024robust}. For instance, \cite{tsou2023foundation} recently showed that linear probes trained on the final embedding of pre-trained GPT-2 are surprisingly robust, sometimes even outperforming full-model fine-tuning. However, while supervised learning excels at learning these representations, the challenge lies in effectively fine-tuning the final layer, the stage where the model's predictions are aligned with desired behaviors. For this alignment phase, we leverage human expertise through comparisons rather than direct labels, building on the cognitive advantage of comparative judgments discussed above. Our approach uses a sequential design of comparisons and active data acquisition to maximize information gain and minimize sample complexity, particularly effective when label noise is high.

	To describe the effectiveness of the refinement phase involving human comparison, we introduce a new concept called the ``label-noise-to-comparison-accuracy'' (LNCA) ratio. We offer a theoretical characterization of the  condition in which the SL+LHF framework performs better than the pure SL algorithm based on this LNCA ratio. In a large-noise scenario, human evaluators with high selection accuracy can significantly reduce the sample complexity.

	It is important to clarify the scope of our framework relative to the reinforcement learning from RLHF pipeline that has become standard for aligning large language models. The standard RLHF pipeline consists of three stages: (1) supervised fine-tuning on curated demonstrations, (2) training a reward model from human preference data, and (3) optimizing the policy via reinforcement learning against the reward model.
		A popular alternative, direct preference optimization (DPO) \citep{rafailov2023direct}, merges stages (2) and (3) into a single supervised objective fit directly to the preference data; this streamlines the pipeline but does not change its nature, since the preference likelihood implicitly encodes the same reward model.
		Our framework addresses a fundamentally different setting. Rather than learning a reward function and then optimizing a generative policy, we study the problem of directly refining model parameters in a low-dimensional space via pairwise human comparisons. This is most relevant in settings where: (a) the model is parameterized by an interpretable, finite-dimensional vector (e.g., coefficients in a pricing model, risk scores, or treatment effect estimates) and (b) the primary challenge is noisy labeled data and model underspecification. The distinction is also structural. RLHF is a reward maximization problem: human preferences are collected upfront to train a reward model, and then reinforcement learning searches for a policy that maximizes the expected reward. In contrast, our framework is a active learning and optimal stopping problem: human comparisons are acquired adaptively (each query is designed based on previous responses), and the central question is how to acquire human data and when to stop querying to guarantee that the estimated parameter lies within $\varepsilon$ of the truth with confidence $1-\delta$. The key theoretical object in RLHF is regret or reward suboptimality; in our framework, it is sample complexity for the $(\varepsilon,\delta)$-alignment problem. Our theoretical framework provides precise conditions, formalized through the LNCA ratio, under which incorporating human comparisons reduces the total sample complexity compared to pure supervised learning, a question that the RLHF literature does not address.

	\subsection{Our Contributions}

	This paper makes contributions in four areas: theoretical modeling, characterization of sample complexity, practical implementation, and empirical validation.

	\paragraph{Modeling.} We develop a theoretical framework that explains when and why human comparisons can substantially enhance performance of SL in high-noise environments. Our model captures the interaction between SL and human feedback through a utility-based comparative judgment approach.

	\paragraph{Theoretical contributions.}  Our main theoretical results consist of two components. First, we develop a probabilistic bisection method, an efficient sequential decision-making process that alternates between ``vertical moves'' (multiple evaluations of the same comparison) and ``horizontal moves'' (selection of new comparison pairs) to acquire human comparisons. Second, focusing on settings where the SL stage utilizes Lasso for learning low-dimensional representations, we introduce the label noise-to-comparison-accuracy ratio (LNCA) that quantifies the trade-off between label noise and human comparison accuracy. Our goal is to refine the model within an $\varepsilon$-distance with confidence at least  $1-\delta$, which $\varepsilon$ represents the learning precision and $1-\delta$ denotes the confidence level. When human selection accuracy uniformly exceeds random guessing, we prove the refinement stage has sample complexity $O(s\log(1/(\delta\varepsilon)))$, where $s$ is the number of non-zero coefficients in Lasso.  We characterize conditions under which our SL+LHF framework outperforms pure SL.

		\paragraph{Towards practical implementations.} While our theoretical analysis examines human comparisons of model coefficients for analytical tractability, direct coefficient comparison is impractical in real applications. To bridge this gap, we develop an Active Comparison Selection (ACS) framework that translates our theoretical insights into implementable procedures by enabling humans to compare meaningful data points rather than abstract model parameters. This framework provides a systematic approach for selecting the most informative samples for human comparison, applicable across a broad set of data distributions.

		\paragraph{Empirical analysis.} We validate the framework end to end on a real high-dimensional task, predicting Kickstarter crowdfunding outcomes from $d\approx9{,}800$ covariates under a fixed query budget. We find that trading part of the label budget for comparisons improves test accuracy precisely when labels are scarce, and the advantage crosses over as labels accumulate. An RLHF-style preference fit given the identical comparisons never matches our active refinement, and the findings persist when the simulated evaluator is replaced by real large language models whose comparison accuracy is measured directly from their answers. Synthetic simulations (Appendix~\ref{appendix: synthetic}) trace the roles of the noise level, the expertise level, and the feature dimension, and an Amazon Mechanical Turk study (Appendix~\ref{appendix: mturk}) validates the model primitives with real human feedback, estimating the accuracy-decay rate and the label noise directly from human responses.

		\paragraph{Managerial implications.} Our framework provides actionable guidance for practitioners deciding how to allocate data collection resources between traditional labeled data and expert comparisons. The LNCA ratio gives a concrete, computable criterion: when label noise $\sigma$ is high relative to the comparison accuracy parameter $\kappa$, investing in human comparisons yields a better return than collecting additional noisy labels. We demonstrate this in the context of demand estimation and pricing under limited data, where experts can correct model misspecification, such as incorrectly estimated price elasticities, that purely data-driven approaches cannot resolve within practical sample budgets. More broadly, our framework applies to any decision-making setting where models are parameterized by interpretable coefficients and experts can meaningfully compare model predictions.

		\section{Literature Review}\label{sec: literature}

		\emph{Alignment and Learning from Human Feedback.}
		The alignment challenge seeks to synchronize human values with machine learning systems and direct learning algorithms toward the objectives and interests of humans. Machine learning models frequently demonstrate unforeseen deficiencies in their performance when they are implemented in practical contexts; this phenomenon has been identified as \emph{underspecification} \citep{d2022underspecification}. Predictors that perform similarly in the training dataset can exhibit significant variation when they are deployed in other domains. The presence of ambiguity in a model might result in instability and suboptimal performance in real-world scenarios. Recent papers have focused on aligning models with human intentions, including training simple robots in simulated environments and Atari games \citep{christiano2017deep,ibarz2018reward} and fine-tuning language models to summarize text \citep{stiennon2020learning} and to optimize dialogue \citep{jaques2019way}. Studies have demonstrated that using reinforcement learning from human feedback (RLHF) significantly boosts performance \citep{bai2022training,glaese2022improving,stiennon2020learning,dwaracherla2024efficient}. \cite{ouyang2022training} use RLHF \citep{stiennon2020learning} to fine-tune GPT-3 to follow a broad class of written instructions. In particular, \cite{zhou2020learning} use pair-wise comparisons for story generation. They find that, when comparing two natural language generation (NLG) models, asking a human annotator to assign scores separately for samples generated by different models, which resembles the case in the ADEM model \citep{lowe2017towards}, is more complicated. The much easier approach is for human annotators to directly compare one sample generated by the first model against another sample from the second model in a pair-wise fashion and then to compute the win/loss rate. \cite{zhu2023principled} provide a reinforcement learning framework that includes human feedback. Specifically, for function approximation in $K$-wise comparisons, with policy learning as the target. \cite{xu2023shattering} propose a fine-tuning algorithm by inducing a complementary distribution over the two sampled LLM responses to model human comparison processes, in accordance with the Bradley-Terry model for pairwise comparisons. %
		In contrast to all these works, we provide a theoretical, two-stage framework for integrating supervised learning with actively acquired human comparison labels, and we characterize the conditions where the two-stage framework outperforms pure SL algorithms.

		While both our framework and RLHF leverage human comparisons, the technical machinery differs substantially. In RLHF, human preferences are collected in batch to train a reward model \citep{ouyang2022training,christiano2017deep,ji2024reinforcement,muldrew2024active}, and the resulting reward function is then optimized via policy gradient methods \citep{zhu2023principled}; human preferences shape the final model only through the intermediary of the learned reward function. In our framework, each human comparison directly updates the posterior distribution over $\btheta^*$ and immediately informs the selection of the next query. There is no intermediate reward model: comparisons adaptively refine the parameter estimate, and the algorithm solves an optimal stopping problem to determine when sufficient precision has been reached. Moreover, our algorithm \emph{actively} determines which samples to present to the human evaluator, selecting the most informative data points for comparison at each step. This adaptive, sequential design, both in what to query and when to stop, is what enables our sample complexity guarantees. Furthermore, RLHF does not address the question of \emph{when} human feedback improves upon pure supervised learning, whereas our LNCA ratio provides a precise answer.

		\emph{Transfer Learning and Task Alignment.} Model alignment generally refers to the process of ensuring that machine learning models are tailored to meet specific tasks and objectives. A related paradigm is transfer learning, where a model trained on one task is adapted for use in another task through additional fine-tuning \citep{pan2009survey,xu2021group,bastani2022meta,du2024transfer}. While our work shares some conceptual similarities with transfer learning, we focus specifically on the impact of noise in both supervised learning and human feedback phases. Our framework's applicability to transfer learning scenarios depends crucially on the relationship between the initial and target tasks' representations. When the target task shares the same low-dimensional representation as the initial task, requiring only updates to the final layer or coefficients, our theoretical results apply directly: the human comparison phase can effectively refine these parameters regardless of whether the comparison objective exactly matches the initial task. However, when the target task requires different representations, transfer learning typically demands an additional supervised learning phase with new labeled data to first refine these representations before any human-guided fine-tuning can begin. In such cases, our framework's human comparison stage remains relevant for the final alignment when label noise is high, but it must be preceded by this representation learning phase.

		\emph{Active Learning.} The core concept behind active learning in machine learning theory is the ability of a model to interactively query a human or an oracle to obtain labels for new data points. Instead of passively learning from a fixed dataset, an active learning algorithm intelligently selects the most informative learning instances to query for labels. We refer readers to \cite{settles2009active} for a comprehensive review of the active learning algorithm literature. The literature has concentrated on studying label complexity to demonstrate the benefit of active learning, which refers to the number of labels requested to attain a specific accuracy. For example, \cite{cohn1994improving} demonstrate that active learning can have a substantially lower label complexity than supervised learning in the noiseless binary classification issue. In the presence of noise, other papers focus on the classification problems \citep{hanneke2007bound,balcan2006agnostic,hanneke2011rates} and the regression problems \citep{sugiyama2009pool,beygelzimer2009importance}. A related line studies learning from pairwise comparison queries in classification. \citet{kane2017active} and \citet{xu2017noise} use comparisons between two instances to reduce the label complexity of learning halfspaces, \citet{zeng2022efficient} extend this to crowdsourced PAC learning with adversarial annotators, and \citet{yona2022active} compare two candidate labels of the same instance in multiclass active learning. These works trade labels for instance-level or label-level comparisons in classification; we instead compare two models' predictions in noisy regression after a supervised representation stage, and characterize when that trade pays off through the LNCA ratio. Although our study has a similar goal of using the fewest possible samples to achieve the target accuracy, we concentrate on the human-AI collaboration framework and try to reduce the sample complexity of human comparisons.

		\emph{Data-Driven Decision Making and Humans in the Loop.} A large body of recent literature on operations research or operations management has studied how to effectively use data to improve decisions \citep{mivsic2020data}, including pricing, assortment, and recommendation. Motivated by the observation that humans may sometimes have ``private'' information to which an algorithm does not have access but that can improve performance, \cite{balakrishnan2022improving} examine new theory that describes algorithm overriding behavior. They designed an experiment to test feature transparency as an implementable approach that can help people better identify and use their private information. \cite{ibrahim2021eliciting} study how to address the shared information problem between humans and algorithms in a setting where the algorithm makes the final decision using the human's prediction as input.  \cite{dietvorst2018overcoming,dietvorst2015algorithm} study the algorithm aversion phenomenon, in which people fail to use algorithms after learning that they are imperfect, even though evidence-based algorithms consistently outperform human forecasters. Other papers study human-artificial intelligence (AI) collaboration on a large variety of topics \citep{bastani2021improving,wu2022survey,deng2020integrating,chen2022algorithmic, caro2023believing,ye2024lola}, including ride-hailing \citep{benjaafar2022human} and health care \citep{reverberi2022experimental}.

		\emph{Probabilistic Bisection.} Another stream of literature related to our work, which is independent of the previously identified human-AI ML topics, is on probabilistic bisection algorithms (PBAs) \citep{horstein1963sequential}. Algorithms using the PBA framework have been used in several contexts.   Notably, including the tasks of target localization \citep{tsiligkaridis2014collaborative}, scanning electron microscopy \citep{sznitman2013optimal}, topology estimation \citep{castro2008active}, edge detection \citep{golubev2003sequential}, and value function approximation for optimal stopping problems \citep{rodriguez2015information}. From a computational perspective, \cite{frazier2019probabilistic} have shown that probabilistic bisection converges almost as quickly as stochastic approximation. In this work, we design the model alignment algorithm by using the PBA approach, where  randomness is involved because of the inherent variability in the human judgment process. Our algorithm inherits the vertical- and horizontal-move scaffolding of \cite{frazier2019probabilistic}, however, the problem we solve with it differs in three respects. First, \cite{frazier2019probabilistic} study the asymptotic convergence rate of probabilistic bisection, whereas our object is the finite sample complexity. Second, in our setting the comparison accuracy is generated by the random utility model and decays toward $1/2$ as the two candidates approach the true parameter, at a task-specific rate $\kappa$, a regime that standard PBA does not model and that changes the sample-complexity analysis. Third, the probabilistic-bisection literature is one-dimensional; we go beyond it through the Gram-Schmidt-based Active Comparison Selection, which decomposes the multi-dimensional alignment problem into a sequence of one-dimensional refinements carried out over comparisons of model predictions on actively selected real samples. A related adaptive-elicitation method is the Fast and Simple Elicitation procedure of \cite{bertani2025fast}, which measures a decision model
			by maintaining bounds on the target function and iteratively halving the gap between them, narrowing the feasible parameter set through linear programming and approximating splines. Its bound-halving shares the bisection spirit of our approach, but it addresses a different problem: efficiently measuring a single individual's preferences in decision analysis from essentially noise-free responses. By contrast, we run a \emph{probabilistic} bisection that absorbs \emph{noisy} comparisons through a posterior and yields $(\varepsilon,\delta)$ sample-complexity guarantees for learning a high-dimensional model, and our central object is the label-noise-to-comparison-accuracy tradeoff rather than the elicitation of a preference function.

		\section{Model}\label{sec: model}
		In this section, we introduce our model setup.
		We first introduce the supervised learning oracle that potentially is used during the \emph{initial learning stage}, and then we present the fine-tuning procedure involving human feedback during the \emph{refinement stage}. The two alternative ways of collecting human feedback are either to collect human labels directly or to ask humans to compare two answers. The former approach, which continues the label collecting procedure following the initial learning stage, is called \emph{pure supervised learning}. As previously noted, we call the latter approach the \emph{Supervised Learning+Learning from Human Feedback} framework.

		\subsection{Supervised Learning and Underspecifications}

		Given the data $\{(X_i,Y_i)\}_{i\leq n}$, where $X_i\in \cX$ and $Y_i\in \cY\subseteq \mathbb{R}$, the goal is to learn the true function that is parameterized by $\boldsymbol{\vartheta}^*\subseteq\mathbb{R}^d$, $h_{\boldsymbol{\vartheta}^*}\in \cH:\cX\rightarrow \cY$, where $\cH$ is the function class. Each data point is generated from a stochastic response:
		\[
		Y_i = h_{\boldsymbol{\vartheta}^*}(X_i) + \epsilon_i\,,
		\]
		where $\epsilon_i$ is the noise with mean 0 and variance $\sigma^2$, and where $\sigma$ could potentially be very large.

		We specifically consider a high-dimensional problem that can be represented by the low-dimensional generalized linear model; that is,
		\[
		h_{\boldsymbol{\vartheta}^*}(x)=f(\inner{\btheta^*}{\varphi(x)})\,,
		\]
		where $f(\cdot)$ is a link function and $\inner{\cdot}{\cdot}$ is the inner product. Here, $\btheta^*\in \Theta\in \mathbb{R}^s$ is the unknown parameter\footnote{We denote the parameter vector using the bold symbol.}; and $\varphi(x)\in \mathbb{R}^s$ is the kernel function, which is the low-dimensional representation of feature $x$. Our goal is to learn both the parameter $\btheta^*$ and representation $\varphi(x)$. We assume that $s\ll d$.

		We provide the following two examples as illustrations.
		\begin{example}[Sparse Linear Models]\label{example: sparse linear}
			Suppose $x$, $\boldsymbol{\vartheta}^*\in \mathbb{R}^d$, and $h_{\boldsymbol{\vartheta}^*}(x)=\inner{\boldsymbol{\vartheta}^*}{x}$.
			The parameter $\boldsymbol{\vartheta}^*$ has nonzero entries in the first $s$ dimensions and has zero entries in the rest of the dimensions. In this case, $h_{\boldsymbol{\vartheta}^*}(x)=f(\inner{\btheta^*}{\varphi(x)})$ where $\btheta^*=(\vartheta^*_1,\cdots, \vartheta^*_s)$, $\varphi(x)=(x_1,\cdots, x_s)$, and $f(\cdot)$ is an identity function. \QEG
		\end{example}

		\begin{example}[Generalized Low-rank Models]
			Suppose $x\in \mathbb{R}^{d\times d}$.
			In the generalized low-rank models, we assume that
			$h_{\boldsymbol{\vartheta}^*}(x)=f(\inner{\bTheta^*}{x})$
			where $\bTheta^*\in \mathbb{R}^{d\times d}$ is a low-rank matrix (i.e., $\bTheta^*=\sum_{i=1}^s \sigma_i u_i v_i^\top$ where $u_i$ and $v_i \in \mathbb{R}^{d})$; the trace inner product is defined as $\inner{\bTheta^*}{x}=tr(\bTheta^* x^\top)$. Thus,
			\[
			h_{\boldsymbol{\vartheta}^*}(x)=f(\sum_{i=1}^s \sigma_i \cdot tr(u_i v_i^\top x^\top))=f(\inner{\btheta^*}{\varphi(x)}),
			\]
			where $\varphi(X)=[tr(u_i v_i^\top x^\top)]_{i=1}^s$ is the projection of $x$ onto the low-rank space, and $\btheta^*=(\sigma_1,\cdots, \sigma_s)$ is the $s$-dimensional vector of singular values.
			\QEG
		\end{example}

		Suppose that $h_{\bvartheta}$ can be estimated through a statistical loss function $\ell: \cY\times \mathbb{R}\rightarrow\mathbb{R}$ and that the total loss is
		\[
		L_n(h_{\bvartheta}) = \sum_{i=1}^{n} \ell(Y_i, h_{\bvartheta}(X_i)).
		\]
		For example, $\ell$ can represent the squared loss and the regularized squared loss, among others. To deal with the high-dimensional issue, the algorithm learns the low-dimensional representation $\varphi(x)$, as well as the parameter $\btheta^*$.
		For example, to learn the sparse linear models, one can apply Lasso to select important features and to learn the model parameters. To learn the low-rank models, one can use trace regression to learn the low-dimensional representations.

		Define $h_{\hat{\bvartheta}_n}$ as the function contained in $\cH$ that minimizes empirical risk $L_n$. In the existing literature, the decision maker very commonly employs the predictor $h_{\hat{\bvartheta}_n}$ directly. However, some issues may arise. In high-dimensional problems, where significant noise is present, approaches striving to balance the variance--bias trade-off may violate fundamental principles, such as
		the natural expectation of negative demand elasticity in economics.
		We use the following example to illustrate this potential issue.

		\begin{example}\label{example: pricing}
			Consider a linear pricing model with high-dimensional contexts. That is, $Y_i = \vartheta_1 \cdot p_i + \vartheta_2^\top X_i + \epsilon_i$, where $Y_i$ is the demand, $p_i$ is the price, and the vector $X_i$ represents features of product $i$. Also, assume that $\vartheta_1<0$. However, in a high-noise regime, we may incorrectly estimate a positive value for $\vartheta_1$. Please see Appendix \ref{appendix: pricing} for numerical illustrations.
			\QEG
		\end{example}

		Example \ref{example: pricing} illustrates a general phenomenon in high-dimensional settings: when the number of features is large relative to the sample size and noise is substantial, regularized estimators like Lasso may produce coefficients with incorrect signs. In this example, the sign constraint (negative price elasticity) is economically meaningful and universally understood. A human expert can immediately correct this violation through comparison.

		In the face of the challenge that underspecification poses in a potentially very noisy environment, we take a critical step toward refinement; we call this step ``model alignment through human comparisons.'' Humans can do comparisons with much greater precision. For example, on the basis of their  knowledge, humans can immediately correct the negative correlation between demand and price. This particular stage comes into play after our machine learning pipelines generate and produce their output $\hat{\varphi}(\cdot)$ and $\hat{\btheta}_0$ (which denotes the estimator of $\varphi$ and $\btheta$) by training on the noisy-labeled data. During the second stage, we refine the model in a lower dimensional space. Throughout the paper, we  focus our analysis on sparse linear models.

		\subsection{Utility Model of Human Comparisons}\label{subsec: utility model}
		In this section, we introduce the human comparison model. Given two potential models, $f_{\btheta_1}$ and $f_{\btheta_2}$, the system asks humans to compare and select the better prediction model given sample $x$; we call this step the sample-level comparison. If $f_{\btheta_1}(x)$ is preferable to $f_{\btheta_2}(x)$, we denote the result as $f_{\btheta_1}(x)\succeq f_{\btheta_2}(x)$.

			The preference over $f_{\btheta}$ at sample $x$ is described by a utility $u_x(\btheta)$, which we represent, after subtracting a candidate-independent constant, as
			\[
			u_x(\btheta)=-d_x(\btheta,\btheta^*),
			\]
			where the discrepancy $d_x\geq0$, with $d_x(\btheta,\btheta)=0$, measures how far the model $\btheta$ is from the truth as judged on sample $x$; the true model attains the highest utility. The leading example in this paper is the prediction discrepancy
			$u_x(\btheta)=-\big|x^\top\btheta-x^\top\btheta^*\big|,$
			under which the evaluator judges a model by how far its prediction on $x$ is from the true prediction. %
		{}
	Suppose that the human selection process follows a random utility model (RUM).
	In the process of querying, the human's opinion regarding $f_{\btheta}(x)$ is
	\begin{equation}\label{eq: rum}
		U_x(\btheta) = u_x(\btheta) +  \xi,
	\end{equation}
	where
	$\xi$ is the unobserved noise.
	We assume that the noise $\xi$ follows a Gumbel distribution with parameter $\gamma$, which represents the expert level of the human. We say human evaluators make the correct choice if they select $\btheta_1$ ($\btheta_2$) when $u_x(\btheta_1)$ ($u_x(\btheta_2)$) is higher than $u_x(\btheta_2)$ ($u_x(\btheta_1)$). A smaller value of $\gamma$ corresponds to a higher certainty of the selection's correctness.
	When $\gamma$ approaches 0, humans can select the better one with almost 100\% accuracy. Although we phrase the model in terms of human evaluators, nothing in the framework requires the comparator to be a person. Any agent that selects between two candidate models, for instance a large language model prompted to judge predictions or a stronger pretrained model, can play the same role, with $\gamma$ capturing its expertise level, and the analysis that follows applies verbatim. Section~\ref{appendix: kickstarter llm} instantiates the evaluator with real large language models.

	\begin{remark}
		Our model can also incorporate the human bias in the following way: $U_x(\btheta) = u_x(\btheta) + \beta_h+ \xi$, where the human bias $\beta_{h}$ can vary between individuals. For simplicity, we focus on the utility model \eqref{eq: rum}.
	\end{remark}

	\begin{lemma}[Precision]\label{lemma: precision}
		For any two models $f_{\btheta_1}$ and $f_{\btheta_2}$ evaluated at sample $x$, the probability that a human will make the right selection is
		\[
		\frac{1}{1+ \exp(-|u_x(\btheta_2)-u_x(\btheta_1)|/\gamma)}\,,
		\]
		which is strictly greater than 1/2 when $u_x(\btheta_1)\neq u_x(\btheta_2)$.
	\end{lemma}
	{}
		Lemma~\ref{lemma: precision} says that the accuracy of a comparison is governed by the utility gap between the two candidates: it increases with the gap, approaches one as the gap grows, and approaches $1/2$ as the gap vanishes. The gap also depends on the input $x$ on which the models are evaluated, so the same pair of models can be easy to compare on one input and uninformative on another.

	We next give a detailed example that illustrates how the model primitives introduced above (the utility function, the RUM, and the comparison probability) come together in a concrete operations management application.

		\begin{example}[Willingness-to-Pay for Electric Vehicles]\label{example: wtp}
			Consider a manufacturer estimating consumer willingness-to-pay (WTP) as a sparse linear function of product features:
			\[
			\text{WTP}_i = \inner{\boldsymbol{\vartheta}^*}{X_i} + \epsilon_i,
			\]
			where $X_i\in \mathbb{R}^d$ encodes attributes such as battery range, charging speed,  and safety rating ($d\approx 50$), of which only $s\approx 5$ to $8$ carry real signal. Labels come from conjoint surveys or transaction records.

			\emph{Large label noise.} Eliciting an absolute WTP is genuinely hard, so the label is a noisy report of the true attribute-driven value. Respondents round to salient figures (\$40K, \$45K) and anchor on recently seen prices; they may misread a feature such as ``300-mile range'' and price the configuration they imagined; and transaction prices absorb promotions, financing, and trade-in terms that do not track the attributes. These deviations are large and roughly mean-zero conditional on $X_i$, entering as high-variance noise $\epsilon_i$ with variance $\sigma^2$.
			With $\sigma$ this large, the fit is underspecified: many coefficient vectors explain the noisy labels equally well, and even the sign of an important coefficient can come out wrong, as in Example~\ref{example: pricing}.

			\emph{Human comparison.} The second stage instead shows a manager two models' predictions on a fixed configuration and asks only which is better. For instance:
			\emph{``For a 300-mile-range, 25-minute-charge, mid-tier, 5-star EV, Model A predicts WTP $=\$52$K and Model B predicts $\$41$K; which is more plausible?''}
			This removes the costliest noise sources: the manager need not commit to a round number, so rounding and anchoring drop out. When the two predictions are far apart ($\$52$K versus $\$41$K), the better one is easy to pick; when they are close ($\$46$K versus $\$44$K), accuracy drops toward $1/2$. The manager is asked only for a relative judgment, which is far easier and far less noisy than producing a precise dollar figure from scratch. Not all error vanishes, since a manager who misreads ``300-mile range'' misreads it the same way when comparing, so the comparison is easier but not noise-free, which is why we treat it probabilistically. In our model, the utility of model $\btheta$ on $X$ is $u_X(\btheta)=-\|X^\top\btheta-X^\top\btheta^*\|$, and the manager picks the model with higher perceived utility $u_X(\btheta)+\xi$ with $\xi\sim$ Gumbel($\gamma$), giving the logistic accuracy $[1+\exp(-|u_X(\btheta_A)-u_X(\btheta_B)|/\gamma)]^{-1}$ (Lemma~\ref{lemma: precision}). %
			\QEG
		\end{example}

	Fix a base model $\btheta_0\in\mathbb{R}^d$, a direction $\boldsymbol v\in\mathbb{R}^d$, and an input $x$, we write the search line as $\btheta(t)=\btheta_0+t\boldsymbol v$, where $t\in\mathbb{R}$ is an additive step along the line. One specific example is $\boldsymbol v = e_i$, which corresponds to coordinate searches. Fix a granularity $\Delta>0$ and compare, at query $t$, the two candidates
	\begin{equation}\label{eq: scalar comparison pairs}
		c_\Delta^-(t)=t-\Delta/2,\qquad c_\Delta^+(t)=t+\Delta/2.
	\end{equation}
	We next introduce Assumption \ref{assump: utility}  that relates the utility with the model correctness. Specifically, it assumes that the preferred candidate must point toward the target, with a tie exactly at the target.

	\begin{assumption}[Scalar comparison oracle]\label{assump: utility}
		The utility gap $D_\Delta(t) := u(c_\Delta^+(t))- u(c_\Delta^-(t))$ satisfies
		\begin{equation}\label{eq: scalar comparison direction}
			\operatorname{sgn}\big(D_\Delta(t)\big)=\operatorname{sgn}(t^*-t)\qquad\text{for all } t,
		\end{equation}
		where $\operatorname{sgn}(0)=0$, and $|D_\Delta(\cdot)|$ is lower semicontinuous.
	\end{assumption}

	Assumption~\ref{assump: utility} says that a preference between the two candidates is always informative about the direction of the target: the preferred candidate lies on the side of $t^*$. %
	{}The prediction discrepancy above satisfies Assumption~\ref{assump: utility}. With $x^\top\boldsymbol v\neq0$, the utility along the line is
		$u(t)=-\big|x^\top\btheta_0+t\,x^\top\boldsymbol v-x^\top\btheta^*\big|=-|x^\top\boldsymbol v|\,|t-t^*|$ and  $
		t^*=\frac{x^\top\btheta^*-x^\top\btheta_0}{x^\top\boldsymbol v},$
		so that $D_\Delta(t)=|x^\top\boldsymbol v|\big(|t-\Delta/2-t^*|-|t+\Delta/2-t^*|\big)$ is continuous in $t$, has the sign of $t^*-t$, and vanishes exactly at $t=t^*$: the candidate nearer to $t^*$ is strictly preferred, and the two candidates tie at the target.%

	\section{Human Comparison Along Single Direction}\label{sec: human}

	The low-dimensional representation $\hat{\varphi}$ is established in the initial stage, as we discuss in Section \ref{sec: human-AI}. Following this stage, our framework calls for model alignment through human comparisons. At this point, the parameter $\btheta$ may lie in a multi-dimensional space.
	{}To illustrate our algorithm, we first consider the special case where the evaluation sample satisfies $\varphi(x) =
		e_i$, which reduces the problem to a one-dimensional human comparison, and then extend to a general sample space.

	The primary challenge in designing the framework lies in strategically selecting pairwise comparisons to minimize the number of comparisons (or sample complexity), which resembles the principles of active learning. First, we consider a special case in which humans can select the better model with 100\% accuracy (deterministic selection). Second, we relax this to a probabilistic selection. For the probabilistic selection, we propose an algorithm we call Model Alignment through Probabilistic Bisection (MAPB) and characterize its sample complexity.

	\subsection{Deterministic Bisection}\label{subsec: deterministic bisection}

	We first illustrate the algorithm in a one-dimensional space with 100\% selection accuracy. This scenario corresponds to $\gamma\to 0$ in RUM. In the one-dimensional case, suppose $\theta^*\in \Theta \subseteq\mathbb{R}$ and $|\theta|\leq \beta_{\Theta}$ for all $\theta\in \Theta$.

	The overall goal is to use the least number of samples to find a $\theta$ such that $d_\varrho(\theta,\theta^*)\leq\varepsilon$, where $d_\varrho(\boldsymbol\theta,\boldsymbol\theta')=\|\boldsymbol\theta-\boldsymbol\theta'\|_\varrho$. The interval is initialized by $[-\beta_{\Theta},\beta_{\Theta}]$, and at round $k$, the center of this interval with respect to the distance $d_\varrho(\cdot,\cdot)$ is used as the \emph{query point}. Specifically, after defining the query point $\theta_k$ to be the single point in the set $H(\theta^-,\theta^+)=\{\theta: d_\varrho(\theta,\theta^-)=d_\varrho(\theta,\theta^+)\}$, we ask the human to compare the utility of two points, denoted by $c_{\Delta}^-(\theta_k)$ and $c_{\Delta}^+(\theta_k)$. These two points have equal distance $\Delta/2$ from $\theta_k$ in opposite directions, where $\Delta$ is called \compgran{} and is a tuning parameter. In one dimension $d_\varrho$ is the absolute value for every $\varrho$, so $H(\theta^-,\theta^+)$ is the arithmetic midpoint: at each round $\theta_k=(\theta^-+\theta^+)/2$,
	$c_{\Delta}^-(\theta_k)=\theta_k-\Delta/2$, and $c_{\Delta}^+(\theta_k)=\theta_k+\Delta/2$.
	If the human selects $c_{\Delta}^-(\theta_k)$, then we eliminate the upper half of the \searchint{} by updating $\theta^+$ to be equal to $\theta_k$; otherwise, we eliminate the bottom half of the \searchint{} by setting $\theta^-$ to be equal to $\theta_k$. After this bisection step, we continue the process by selecting $\theta_{k+1}$ to be $H(\theta^-,\theta^+)$ again, and repeat the above step until the length of the \searchint{} is at most $2\varepsilon$, so that its midpoint is within $\varepsilon$ of every point of the interval, in particular of $\theta^*$. The pseudo-code is presented in Algorithm \ref{Alg: DB}.

	~\\

	\begin{center}
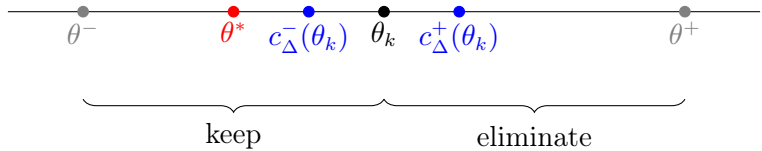

		\begin{tikzpicture}
			\draw (-5,0) -- (5,0);
			\filldraw[black] (0,0) circle (2pt) node[anchor=north]{$\theta_k$};
			\filldraw[red] (-2,0) circle (2pt) node[anchor=north]{$\theta^*$};
			\filldraw[gray] (-4,0) circle (2pt) node[anchor=north]{$\theta^-$};
			\filldraw[gray] (4,0) circle (2pt) node[anchor=north]{$\theta^+$};
			\filldraw[blue] (1,0) circle (2pt) node[anchor=north]{$c_\Delta^+(\theta_k)$};
			\filldraw[blue] (-1,0) circle (2pt) node[anchor=north]{$c_\Delta^-(\theta_k)$};

			\draw [decorate,decoration={brace,amplitude=5pt,mirror,raise=4ex}]
			(-4,-0.5) -- (0,-0.5) node[midway,yshift=-3em]{keep};

			\draw [decorate,decoration={brace,amplitude=5pt,mirror,raise=4ex}]
			(0,-0.5) -- (4,-0.5) node[midway,yshift=-3em]{eliminate};

		\end{tikzpicture}
		\captionof{figure}{An illustrative example. Among two choices $c_{\Delta}^-(\theta_k)$ and $c_{\Delta}^+(\theta_k)$, $c_{\Delta}^-(\theta_k)$ would be selected because it is closer to the true parameter. In this case, the interval $(\theta_k, \theta^+]$ is eliminated.}\label{fig: outline}
	\end{center}

	\begin{algorithm}
		\caption{\textnormal{\textsf{DB}} (Deterministic Bisection)}\label{Alg: DB}
		\begin{algorithmic}[1]
			\State \textbf{input}:  precision $\varepsilon$; \compgran{} $\Delta$; $k=0$; $\theta_k = 0$; $\theta^+ =\beta_{\Theta}$; $\theta^- = -\beta_{\Theta}$;
			\Repeat
			\State Ask human to evaluate $c_\Delta^-(\theta_k)$ and $c_\Delta^+(\theta_k)$
			\If{$c_\Delta^+(\theta_k)$ is chosen}
			\State $\theta^- = \theta_k$;
			\Else{ $\theta^+=\theta_k$;}
			\EndIf
			\State $k=k+1$; $\theta_k=H(\theta^-,\theta^+)$;
			\Until{$d_\varrho(\theta^+,\theta^-)\leq 2\varepsilon$.}
		\end{algorithmic}
	\end{algorithm}

	We can show the following result for the performance of Algorithm \textsf{DB}.

	\begin{proposition}\label{prop: bisection}
		{}Under Assumption~\ref{assump: utility}, suppose the noiseless oracle chooses a candidate with maximal utility ($\gamma\rightarrow 0$) and breaks ties arbitrarily. After at most $\lceil\log_2(\beta_\Theta/\varepsilon)\rceil$ comparisons, the query point $\theta_k$ of Algorithm~\ref{Alg: DB} satisfies $d_\varrho(\theta_k,\theta^*)\leq\varepsilon$; that is, $O(\log(\beta_\Theta/\varepsilon))$ comparisons suffice.
	\end{proposition}

	Note that the comparison granularity has no impact on Algorithm \ref{Alg: DB} and its computational complexity because the deterministic model assumes 100\% accuracy for any $\Delta$. However, this deterministic model of human selection as described exhibits two significant limitations: First, in practical scenarios, humans may not always be able to identify the correct choice, implying that $\gamma>0$; second, the accuracy of human selections is influenced by the \compgran{} $\Delta$. To address these limitations, we introduce a more general probabilistic model in the following section.

	\subsection{Probabilistic Bisection}\label{subsec: PB}

	First, we enrich the model to capture the probabilistic nature of human comparisons by noting that, in the RUM, when $\gamma$ is strictly positive, the human's selection is random, with a certain distribution that favors the correct answer, parameterized by $\gamma$. Precisely, as shown in Lemma \ref{lemma: precision}, the human selects $c_\Delta^-(\theta)$ with probability
	\[
	p^-_{\gamma}(\theta) := \left[1+\exp\left\{\frac{u\Big(c_\Delta^+(\theta)\Big)-u\Big(c_\Delta^-(\theta)\Big)}{\gamma}\right\}\right]^{-1}\,,
	\]
	and the human selects $c_\Delta^+(\theta)$ with probability $p^+_{\gamma}(\theta)$, which is equal to $1-p^-_{\gamma}(\theta)$.

	We note that, in this equation, the selection probability depends on the \compgran{}. This dependence naturally captures the difficulty of the selection for the human. On the one hand, when $\Delta$ is too small, $c_\Delta^-(\theta)$ and $c_\Delta^+(\theta)$ converge toward each other; hence, $p^-_{\gamma}(\theta)$ and $p^+_{\gamma}(\theta)$ converge to $1/2$, representing the difficulty of distinguishing between two highly similar choices. On the other hand, when $\Delta$ is too large, the selection can  become difficult as well because both options are far away from $\theta^*$. To illustrate, this difficulty can be captured when the distance function converges to a finite value for far away points from $\theta^*$. In this case, $p^-_{\gamma}(\theta)$ and $p^+_{\gamma}(\theta)$ could also converge to 1/2 (see Example \ref{example: far-away answer}). Since $\Delta$ is fixed throughout a run, for brevity we exclude $\Delta$ in the notations $p^-_{\gamma}(\theta)$ and $p^+_{\gamma}(\theta)$ in the remainder of the paper.

	\begin{example}\label{example: far-away answer}
		{}Consider the scalar target $\theta^*=0$ and normalized utility
			\[
			u(\theta)=\frac12-\frac{1}{1+\exp(-|\theta|)}.
			\]
			Its discrepancy is a continuous, strictly increasing, bounded function of $|\theta|$, so the candidate nearer to $\theta^*$ has the strictly higher utility and the two candidates tie only at $\theta=\theta^*$; the gap is continuous, and Assumption~\ref{assump: utility} holds. For fixed $\theta$ and large admissible $\Delta$, both candidate utilities approach $-1/2$, their difference vanishes, and both selection probabilities approach $1/2$. {}
		\QEG
	\end{example}

	The goal is accurate estimation of the scalar coordinate or of the parameter vector in a chosen norm, as specified in the following definition. A small discrepancy on one prediction does not alone identify a parameter vector, since distinct parameters can produce the same prediction.

	\begin{definition}[$(\varepsilon,\delta)$-alignment problem]\label{definition: precision and confidence}
		{} For $\varepsilon>0$ and $0<\delta<1$, an output $\widehat\btheta$ solves the $(\varepsilon,\delta)$-alignment problem if
			$\Pr\big(\|\widehat\btheta-\btheta^*\|_\varrho\leq\varepsilon\big)\geq1-\delta$.
		{}
	\end{definition}

	If $\theta<\theta^*$, then $p_\gamma^-(\theta) <1/2$ and $p_\gamma^+(\theta)>1/2$; and vice versa.
	We introduce several notations for simplicity. Let $Y(\theta)$ denote the selection: $Y(\theta)=c_\Delta^-(\theta)$ with probability $p_\gamma^-(\theta) $ and $Y(\theta)=c_\Delta^+(\theta)$ with probability $p_\gamma^+(\theta)$. Define $\tilde{Z}(\theta) = 2\cdot\1(Y(\theta)=c_\Delta^+(\theta))-1$, so that $\tilde{Z}(\theta)=1$ if $Y(\theta)=c_\Delta^+(\theta)$ (when we believe $\theta^*$ is more likely to be to the right of $\theta$) and $\tilde{Z}(\theta)=-1$ if $Y(\theta)=c_\Delta^-(\theta)$ (when we believe $\theta^*$ is more likely to be to the left of $\theta$). Note that we have $\tilde{Z}(\theta)=1$ with probability $p_\gamma^+(\theta)$. Also, define
	\[
	\tilde{p}(\theta) = \Pr\big(\tilde{Z}(\theta)\text{ correctly identifies the direction}\big) = \begin{cases} \Pr(\tilde{Z}(\theta)=1), & \mbox{if } \theta \leq \theta^*  \\ \Pr(\tilde{Z}(\theta)=-1), & \mbox{if } \theta>\theta^* \end{cases}\,.
	\]
	{}With $D_\Delta(\theta)=u(c_\Delta^+(\theta))-u(c_\Delta^-(\theta))$ the utility gap of the two candidates, Lemma~\ref{lemma: precision} gives
		\begin{equation}\label{eq: scalar drift identity}
			p_\gamma^+(\theta)=\frac{1}{1+\exp(-D_\Delta(\theta)/\gamma)},\qquad
			b(\theta):=\E[\tilde{Z}(\theta)]=\tanh\!\left(\frac{D_\Delta(\theta)}{2\gamma}\right),
		\end{equation}
		so that $\tilde p(\theta)=[1+\exp(-|D_\Delta(\theta)|/\gamma)]^{-1}$ and $|b(\theta)|=2\tilde p(\theta)-1$. By \eqref{eq: scalar comparison direction}, the drift $b(\theta)$ is positive to the left of the target, negative to its right, and zero at $\theta^*$, where $\tilde p(\theta^*)=1/2$.

	Figure \ref{fig: proofoutline} shows the structure of introducing our main algorithm. Algorithm \ftpb{}, described as follows, involves two types of moves: the \emph{horizontal moves} and the \emph{vertical moves}. For any fixed $\theta$, we ask the human to repeatedly evaluate $c_\Delta^-(\theta)$ and $c_\Delta^+(\theta)$.
	We refer to this as the vertical moves. After the vertical stopping criteria are satisfied, the algorithm decides the next $\theta$ to evaluate according to a certain distribution that will be defined. We call this step the horizontal move. The horizontal moves continue until the horizontal stopping criteria are satisfied. In the following two sections, we discuss the vertical moves and horizontal moves, as well as their stopping criteria, in sequence. Then, as shown in Figure 2, we introduce the two-stage Human-AI collaboration framework for multiple dimensions in Section \ref{sec: human-AI} and practical implementations in Section \ref{sec: practical implementations}.\\

	\begin{center}
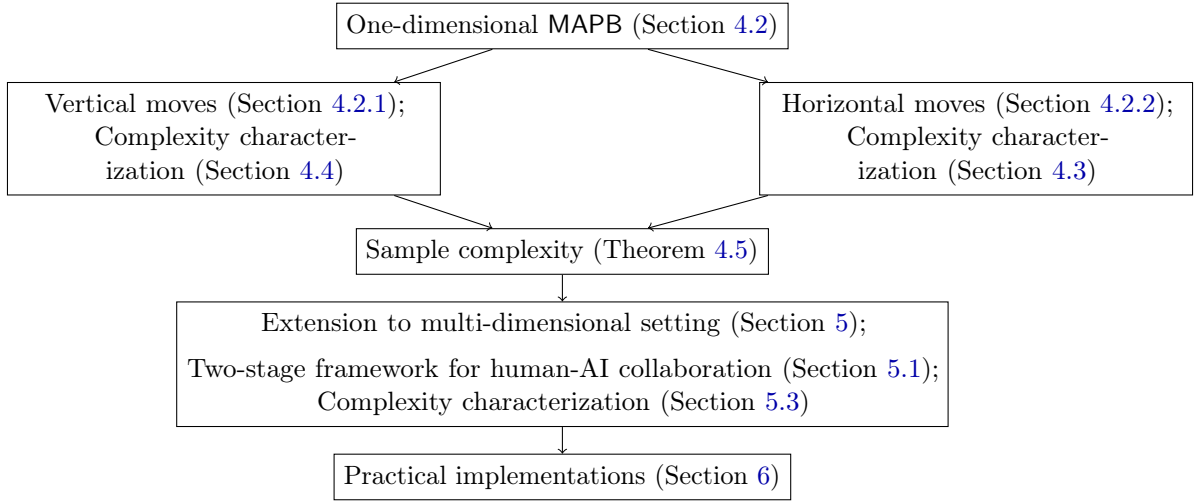

		\begin{tikzpicture}
			\node (1) [draw] at (6.5,0)    { {\small One-dimensional \ftpb{} (Section \ref{subsec: PB})}};
			\node (2) [draw] at (12,-1.5)  [text width= 5.5cm, align=center]    {{\small Horizontal moves (Section \ref{subsubsec: horizontal moves});\\ Complexity characterization (Section \ref{subsec: horizontal move})}};
			\node (3) [draw] at (6.5, -3)    {{\small Sample complexity (Theorem \ref{subsec: sample complexity})}};
			\node (4) [draw] at (6.5, -4.5) [text width=10cm, align=center]    {{\small
					Extension to multi-dimensional setting (Section \ref{sec: human-AI});\\
					Two-stage framework for human-AI collaboration (Section \ref{subsec: two-stage});\\ Complexity characterization (Section \ref{subsec: SFT+FTPB})}};
			\node (5) [draw] at (6.5, -6)    {{\small Practical implementations (Section \ref{sec: practical implementations})}};

			\node (6) [draw] at (2, -1.5) [text width= 5.5cm, align=center]   {{\small Vertical moves (Section \ref{subsubsec: vertical moves});\\Complexity characterization (Section \ref{subsec: vertical move})}};
			\draw [->] (1) -- node [sloped, above] {} (2);
			\draw [->] (1) -- node [sloped, above] {} (6);
			\draw [->] (2) -- node [sloped, above] {} (3);
			\draw [->] (3) -- node [sloped, above] {} (4);
			\draw [->] (6) -- node [sloped, above] {} (3);
			\draw [->] (4) -- node [sloped, above] {} (5);
		\end{tikzpicture}
		\captionof{figure}{Roadmap for introducing Algorithm \ref{Alg: RTB}.}\label{fig: proofoutline}
	\end{center}

	\subsubsection{Vertical moves}\label{subsubsec: vertical moves}
	Fix some $\theta$.
	We repeatedly ask for human evaluations of the query point $\theta$, achieved through comparisons of $c^+_\Delta(\theta)$ and $c^-_\Delta(\theta)$.
	{}By \eqref{eq: scalar drift identity}, signs at a fixed query have mean $2\tilde p(\theta)-1>0$ when $\theta<\theta^*$, its negative when $\theta>\theta^*$, and zero at the target. Their partial sums $S_m(\theta)=\sum_{i=1}^m\widetilde Z_i(\theta)$ define the vertical random walk. Sequential test of power one indicates whether the drift $\omega$ of a random walk satisfies the hypothesis $\omega<0$ versus $\omega>0$.

	Define the stopping time $\tau^{\uparrow}_1(\theta)= \inf\{m\in \mathbb{N}: |S_m(\theta)|\geq \hbar_m\}$, where $\{\hbar_m\}_m$ is a positive sequence. Here, ``$\uparrow$'' and ``$\rightarrow$'' represent vertical and horizontal, respectively. The test decides that $\theta<\theta^*$, if $S_{\tau_1^\uparrow(\theta)}(\theta)\geq \hbar_{\tau^{\uparrow}_1(\theta)}$; it decides that $\theta>\theta^*$ if $S_{\tau_1^\uparrow(\theta)}(\theta)\leq -\hbar_{\tau^{\uparrow}_1(\theta)}$; and the test does not make a decision if $\tau^{\uparrow}_1(\theta)=\infty$.

	Assume now that the algorithm queries a point $\theta_k$ (not equal to $\theta^*$) at the $(k+1)^{th}$ iteration. We then observe a random walk with the $m^{th}$ term $S_{k,m}=S_{k,m}(\theta_k)=\sum_{i=1}^m Z_{k,i}(\theta_k)$ until we reach the end of the test, which is called the power-one test \citep{robbins1974expected}.
	Define the new signal:
	\[
	Z_k(\theta_k) = \begin{cases} +1, & \mbox{if } S_{k,\tau_{1}^{\uparrow}(\theta_k)}>0\,,  \\ -1, & \mbox{if } S_{k,\tau_{1}^{\uparrow}(\theta_k)}<0 \,. \end{cases}
	\]

	\begin{lemma}\label{lemma: detection accuracy}
		{} Define $\hbar_m=\sqrt{2m\log((m+1)/\eta)}$. Under Assumption~\ref{assump: utility}, at a fixed query $\theta\ne\theta^*$, the test terminates almost surely and returns the wrong direction with probability at most $\eta/2$, conditionally on the history before that query. Hence its directional accuracy is at least $p_c=1-\eta/2$. %
		{}
	\end{lemma}

		Lemma~\ref{lemma: detection accuracy} lower bounds the directional detection accuracy by $p_c=1-\eta/2$. As $\theta$ approaches $\theta^*$, the drift of the comparison walk decreases and the uncapped directional test may require many comparisons; its expected length grows like the inverse squared drift, so a query that happens to land very close to $\theta^*$ would be  expensive. Our objective is instead to return a point within distance $\varepsilon$ with confidence $1-\delta$. We therefore cap every directional test at a deterministic number $m=\tau^{\uparrow}(\delta/2,\varepsilon;K)$ of comparisons, where $K$ is the horizontal query budget; both are defined in Theorem~\ref{thm: micro stop}. Failure to cross the directional boundary within the cap is evidence that the query is close to $\theta^*$, and the algorithm then returns it.

	\subsubsection{Horizontal moves}\label{subsubsec: horizontal moves}
	Algorithm \ref{Alg: RTB} begins with a prior density $\rho_0$ on $[-\beta_{\Theta}, \beta_{\Theta}]$ that is positive everywhere. For example, $\rho_0$ can be chosen as a uniform distribution. Choose some $p\in (1/2, p_c)$, and let $q=1-p$. When the algorithm queries at the point $\theta_k$ to obtain $Z_k$, we update the density $\rho_k$:
	\begin{enumerate}
		\item If $Z_k(\theta_k)=+1$, then
		\[
		\rho_{k+1}(\theta') = \begin{cases} 2p \rho_{k}(\theta'), & \mbox{if } \theta'\geq\theta_k;\\ 2q \rho_{k}(\theta'), & \mbox{if } \theta'<\theta_k. \end{cases}
		\]
		\item If $Z_k(\theta_k)=-1$, then
		\[
		\rho_{k+1}(\theta') = \begin{cases} 2q \rho_{k}(\theta'), & \mbox{if } \theta'\geq\theta_k;\\ 2p \rho_{k}(\theta'), & \mbox{if } \theta'<\theta_k. \end{cases}
		\]
	\end{enumerate}

	The next query $\theta_{k+1}$ is the median of $\rho_{k+1}$. Throughout, $\mu_k$ denotes the distribution with density $\rho_k$, so the two describe the same posterior; Algorithm~\ref{Alg: RTB} and the analysis are written in terms of $\mu_k$, and the median of $\mu_k$ is the median of $\rho_k$. Throughout the paper, we assume that $\theta_k\neq \theta^*$ with probability 1. Define $\tau^\rightarrow(\delta,a)$ as the number of horizontal moves to reach $a$-neighborhood of $\theta^*$ with confidence level $1-\delta$ (which will be specified in Theorem~\ref{thm: complexity}). That is, parameter $\theta$ moves to the local $a$-neighborhood of $\theta^*$ with probability of at least $1-\delta/2$ when the number of steps exceeds $\tau^\rightarrow(\delta/2,a)$. The horizontal move stops when the total number of moves reaches $K:=\lceil\tau^{\rightarrow}(\delta/2,\varepsilon)\rceil$ because at this point the precision level $\varepsilon$ has been achieved. %

		At each query the algorithm collects comparisons until the directional boundary is crossed or the deterministic vertical cap $m:=\tau^{\uparrow}(\delta/2,\varepsilon;K)$ (defined in Theorem~\ref{thm: micro stop}) is reached. Boundary crossing takes priority, including at sample $m$: it produces a directional update. Only if the boundary has not been crossed through sample $m$ does the algorithm terminate and return the current query. Reaching the cap without a crossing is evidence that the current query is already within $\varepsilon$ of $\theta^*$ since the utility of $\theta$ gets close to that of $\theta^*$. If $K$ horizontal moves are completed first, the algorithm returns the median $\theta_K$.%

	Algorithm \ref{Alg: RTB} is described as follows. We initialize the distribution as $\mu_0$. Without any historical data, $\mu_0$ can be initialized as the uniform distribution over the confidence region. Alternatively, $\mu_0$ can be approximated by the empirical distribution of the estimator using the bootstrap method and applying it to the first phase.
	{}
		At step $k$, the algorithm collects comparison signs $\widetilde Z_{k,s}$ at $\theta_k$ and forms $S_{k,s}=\sum_{i=1}^s\widetilde Z_{k,i}$. It checks $|S_{k,s}|\geq\hbar_s$ before applying the early-return condition. %
		If this inequality holds, it updates $\mu_{k+1}$ using $Z_k$; otherwise the algorithm terminates and returns $\theta_k$.

		\begin{algorithm}
			\caption{\textsf{MAPB} (Model Alignment through Probabilistic Bisection)}\label{Alg: RTB}
			\begin{algorithmic}[1]

				\State \textbf{input}: $\delta\in(0,1)$, $\varepsilon>0$, initial distribution $\mu_0$, $\eta\in(0,1)$;  set $K=\max\{1,\lceil\tau^{\rightarrow}(\delta/2,\varepsilon)\rceil\}$ and $m=\tau^{\uparrow}(\delta/2,\varepsilon;K)$; set $k=0$.\;%
				{}
				\While{$k<K$}
				\State Pick $\theta_k$ as the median of $\mu_k$; set $s=0$ and $S_{k,0}=0$
				\Repeat
				\State $s=s+1$; collect $\widetilde Z_{k,s}$ at $\theta_k$; set $S_{k,s}=S_{k,s-1}+\widetilde Z_{k,s}$
				\Until{$|S_{k,s}|\geq\hbar_s$ or $s=m$}
				\If{$|S_{k,s}|\geq\hbar_s$}
				\State Set $Z_k=\operatorname{sgn}(S_{k,s})$; update $\mu_{k+1}$ using $Z_k$; $k=k+1$
				\Else
				\State \Return $\widehat\theta=\theta_k$
				\EndIf
				\EndWhile
				\State Pick $\theta_K$ as the median of $\mu_K$; \Return $\widehat\theta=\theta_K$
			\end{algorithmic}
		\end{algorithm}

	\subsection{Complexity of Horizontal Moves}\label{subsec: horizontal move}

	Let $T^{\uparrow}(\theta)$ denote the number of comparisons that are collected at $\theta$ and let $T^{\rightarrow}$ denote the number of horizontal moves; let $T_k^+$ denote the total number of comparisons that are collected after $k$ horizontal moves, where ``$+$'' denotes both horizontal and vertical moves. Note that $T^\uparrow(\theta)$, $T^\rightarrow$, and $T_k^+$ are all random variables.
	In the following paragraphs, we discuss the complexity of vertical and horizontal moves, and we characterize its sample complexity.

	The horizontal analysis follows similarly from \cite{frazier2019probabilistic}: the last-exit time after which the medians stay within an $a$-neighborhood of $\theta^*$ is stochastically bounded by a stopping time at which the posterior mass concentrates near $\theta^*$ plus a random variable that does not depend on $a$ (Lemma~\ref{lemma: T(a) and tau(a)} in Appendix~\ref{appendix: proof for human sec}), and the tail of that stopping time is controlled by the prior mass near $\theta^*$ (Lemma~\ref{lemma: prior}). %
	{} Algorithm \ref{Alg: RTB} starts with an initial distribution $\mu_0$. Without any prior information, $\mu_0$ can be set as a uniform distribution; otherwise, we can use the prior information to update $\mu_0$. First, we make the following mild assumption regarding the initial prior distribution.

	\begin{assumption}[Prior distribution]\label{assump: prior}
		{}Let $B=(\theta^*-a,\theta^*]$ and $B'=[\theta^*,\theta^*+a)$ denote the left and right $a$-neighborhoods of $\theta^*$. There exist constants $c_0>0$, $\varsigma\in(0,1]$, and $a_0\in\big(0,\min\{\theta^*+\beta_\Theta,\beta_\Theta-\theta^*\}\big]$ such that the initial distribution satisfies
			\[
			\mu_0(B)\geq c_0\,a^{\varsigma}\quad\text{and}\quad \mu_0(B')\geq c_0\,a^{\varsigma}\qquad\text{for all } 0<a\leq a_0.
			\]
	\end{assumption}
	{}The uniform distribution on $[-\beta_\Theta,\beta_\Theta]$ satisfies Assumption \ref{assump: prior} with $\varsigma=1$ and $c_0=1/(2\beta_\Theta)$, because $\mu_0(B)=\mu_0(B')=a/(2\beta_\Theta)$. More generally, any prior whose density is bounded below by $\rho_{\min}>0$ on the $a_0$-neighborhood of $\theta^*$ satisfies it with $\varsigma=1$ and $c_0=\rho_{\min}$.%

	We next characterize the complexity of horizontal moves. The constants $r_1,r_2,\alpha,\beta_1,\beta_2,r_R>0$ below are the rate and prefactor constants, specified in Lemmas~\ref{lemma: T(a) and tau(a)} and~\ref{lemma: prior} of Appendix~\ref{appendix: proof for human sec}; they depend only on the update parameter $p$ and on the initial prior mass on each side of $\theta^*$.

	\begin{theorem}[Sample complexity of horizontal moves]\label{thm: complexity}
		{}Under Assumptions~\ref{assump: utility} and~\ref{assump: prior}, by setting
		\[
		{}\tau^\rightarrow(\delta,\varepsilon) = \max \left\{\frac{4\log\big(\frac{8}{c_0\,\delta\, \varepsilon^{\varsigma}}\big)}{(1+\varsigma)\,r_2\alpha},  \frac{2\log(8\beta_1/\delta)}{r_1}, \frac{\log(8 \beta_2/\delta)}{r_R}\right\}.
		\]
		we have that
		\[
		\Pr(|\theta_{k}-\theta^*|\leq \varepsilon)>1-\delta, \quad \forall k\geq  \tau^{\rightarrow}(\delta,\varepsilon)\,.
		\]
	\end{theorem}

	Theorem \ref{thm: complexity} states that after $\tau^\rightarrow(\delta,\varepsilon)$ horizontal moves, $\theta_k$ reaches the $\varepsilon$-neighborhood of the true parameter with probability at least $1-\delta$. The complexity of this process depends on $\log(1/\delta)$, $\log(1/\varepsilon)$, and on the prior through $\log(1/c_0)$ and the factor $1/(1+\varsigma)$. When the initial prior distribution is more closely centered around the true parameter $\theta^*$, fewer horizontal moves are needed.

	\subsection{Complexity of Vertical Moves}\label{subsec: vertical move}

		The vertical test uses two stopping criteria at every query: directional boundary crossing and the cap $m$.%
	{}
	{}Two regimes arise for the vertical moves. If the comparison accuracy stays above a fixed floor at every query other than the target, a cap of logarithmic length certifies every query, and the total cost is logarithmic in $1/(\delta\varepsilon)$ up to a further logarithmic factor; this case is settled in Corollary~\ref{thm: complexity pc} below as a special case of the general result. For the prediction discrepancy and the parameter distance, however, $D_\Delta$ is continuous with $D_\Delta(\theta^*)=0$, so $\tilde p(\theta)\to1/2$ as $\theta\to\theta^*$ and comparisons become uninformative near the target.

	Next, we first address the scenario in which the selection accuracy tends toward $1/2$ as the tested parameter approaches the true parameter.

	\subsubsection{When the selection accuracy approaches $1/2$.}

	As the comparison accuracy approaches $1/2$, the vertical move (based on the power-one test) requires the collection of more samples. The rate at which the vertical sample complexity increases depends on the rate at which the utility changes as $\theta$ converges to $\theta^*$. To account for this effect, we introduce the following assumption.

	\begin{assumption}\label{assump: lip}
		{}There exist fixed $a>0$, $\lambda_\Delta>0$, and $0<\kappa\leq1$ such that, for every $\theta\in I$ with $|\theta-\theta^*|\leq a$, it holds that
			$|D_\Delta(\theta)|\geq\lambda_\Delta|\theta-\theta^*|^\kappa.$
		{}
	\end{assumption}
	The parameter $\kappa$ characterizes the speed of convergence, with a smaller value indicating a faster convergence rate. Assuming that $\kappa\leq 1$ is natural for several reasons. First, human decision making should not be inferior to learning from a random sample; otherwise, collecting random samples would suffice. We demonstrate this by examining the selection accuracy of learning from a single sample.  Suppose the sample $(x,y)$ is generated from $y=\theta^* x +\epsilon$ where $\epsilon\sim \cN(0,\sigma^2)$. When comparing $\theta+\Delta$ and $\theta-\Delta$ (without loss of generality, assuming $\theta<\theta^*$), $\theta+\Delta$ is preferable to $\theta-\Delta$ if the residual is smaller. That is, the probability that $\theta+\Delta$ is better is
	\[
	\begin{aligned}
		\Pr(\theta+\Delta \succeq \theta-\Delta) &= \Pr\left((\theta^*x + \epsilon-(\theta+\Delta)x)^2\leq (\theta^*x + \epsilon-(\theta-\Delta)x)^2\right)\\
		&=\Pr(\epsilon\geq (\theta-\theta^*)|x|)
		= \frac{1}{2}+ \Omega((\theta^*-\theta)|x|).
	\end{aligned}
	\]

	In instances where evaluation is conducted solely on a random sample, $\kappa$ is equal to 1. Considering that humans possess or have access to additional side information and thus can make better comparisons, we anticipate that $\kappa$ is, at most, 1.

	\begin{lemma}[How selection accuracy converges]\label{lemma: convergence}
		{}Under Assumptions~\ref{assump: utility} and~\ref{assump: lip}, it holds that
		\[
		\tilde{p}(\theta)-1/2\geq c\|\theta-\theta^*\|^\kappa\,, \forall \theta\in [\theta^*-a,\theta^*+a],
		\]
		where $c=\min\{\lambda_\Delta/(8\gamma \ln(2)), 1/(6a^\kappa)\}$.
	\end{lemma}

		Lemma~\ref{lemma: convergence} converts a small comparison drift into a small parameter error within the $a$-neighborhood, and Lemma~\ref{lemma: finite cover} bounds the drift away from zero outside it. Together they allow a capped directional test to certify closeness at \emph{every} query, provided the  errors are controlled simultaneously over the at most $K$ queries.
		{}The uncapped reference of Section~\ref{subsubsec: horizontal moves} is well defined. Lemma~\ref{lemma: horizontal coupling} in Appendix~\ref{appendix: proof for human sec} couples Algorithm~\ref{Alg: RTB} to this reference so that the two coincide at every reached query, and thereby transfers the guarantee of Theorem~\ref{thm: complexity} to the final horizontal output.%

		Let $c$ be the constant of Lemma~\ref{lemma: convergence} and let
		\begin{equation}\label{eq: adaptive outer drift}
			\underline{\mu}_a:=\inf\{\,2\tilde p(\theta)-1:\ \theta\in[-\beta_\Theta,\beta_\Theta],\ |\theta-\theta^*|\geq a\,\}>0
		\end{equation}
		be the minimal comparison drift outside the $a$-neighborhood, positive by Lemma~\ref{lemma: finite cover}.

		\begin{theorem}[Vertical moves under adaptive termination]\label{thm: micro stop}
			Under Assumptions~\ref{assump: utility}-\ref{assump: lip}, by setting
			\begin{equation}\label{eq: adaptive vertical budget}
				m=\tau^{\uparrow}(\delta/2,\varepsilon;K)=\left\lceil\max\left\{\tau_0,\ \frac{8}{g^{2}}\log\frac{2K}{\delta}\right\}\right\rceil
				=\widetilde O\big(\varepsilon^{-2\kappa}\big),
				\qquad g=\min\{c\varepsilon^\kappa,\underline{\mu}_a\},
			\end{equation}
			where $\tau_0=\min\{n\in\N:\hbar_s/s\leq g/2\text{ for all }s\geq n\}$ with $\hbar_s=\sqrt{2s\log((s+1)/\eta)}$, the algorithm makes at most $Km$ comparisons on every sample path, and
			\[
			\Pr\bigl(d_\varrho(\widehat\theta,\theta^*)\leq\varepsilon\bigr)\geq1-\delta.
			\]
		\end{theorem}

		When  Assumptions~\ref{assump: utility}-\ref{assump: lip} hold, the sample complexity of the vertical move would not exceed $O(\varepsilon^{-2\kappa})$, where $\varepsilon$ is the precision. The vertical sample complexity increases as the value of $\kappa$ increases. The faster the selection accuracy converges to 1/2, the harder the selection problem, and thus the higher the vertical sample complexity. In the extreme case, where $\kappa=0$, $\tilde{p}(\theta)$ is bounded away from 1/2 so $\tau^\uparrow(\delta,\varepsilon)$ is bounded. We characterize the special case where the drift is bounded below by a constant in Corollary \ref{thm: complexity pc}.

				\begin{corollary}[Sample complexity when $\lowp$ exists]\label{thm: complexity pc}
					{}Suppose Assumptions~\ref{assump: utility}-\ref{assump: lip} hold and $\tilde p(\theta)\geq\lowp>1/2$ for all $\theta\in I\setminus\{\theta^*\}$. Then $\Pr(d_\varrho(\widehat\theta,\theta^*)\leq\varepsilon)\geq1-\delta$, and for fixed $\lowp$ and $\eta$ the number of comparisons satisfies
						\[
						N=\widetilde O\!\left(\log\frac{1}{\delta\varepsilon}\right),
						\]
						where the suppressed factor is $\log\big(\log(1/(\delta\varepsilon))/\delta\big)$.
				\end{corollary}
				{}Corollary~\ref{thm: complexity pc} indicates that, when the selection accuracy stays above a fixed floor $\lowp$, the comparison drift never falls below $2\lowp-1$, so a cap of order $g_0^{-2}\log(2K/\delta)$ comparisons, with $g_0=2\lowp-1$, shows every query, and the total cost is the horizon $K=O(\log(1/(\delta\varepsilon)))$ times that cap. The dependence on the precision is then logarithmic, as in noiseless bisection. The $\varepsilon^{-2\kappa}$ term of Theorem~\ref{thm: micro stop} arises only because, for the prediction discrepancy and the parameter distance, the accuracy approaches $1/2$ as the query approaches the target, so determining that a query is within $\varepsilon$ of $\theta^*$ requires of order $\varepsilon^{-2\kappa}$ comparisons at that query.

		\subsection{Sample Complexity}\label{subsec: sample complexity}
		Combining the sample complexity of horizontal moves and vertical moves, we conclude the total
		sample complexity of \ftpb{}.

			\begin{theorem}[Sample complexity of \ftpb{}]\label{thm: total complexity}
				Under Assumptions~\ref{assump: utility}-\ref{assump: lip}, the total number $N$ of comparisons made by Algorithm~\ref{Alg: RTB} satisfies, on every sample path,
				\begin{equation}\label{eq: adaptive total cost}
					N\leq H(\delta,\varepsilon;a):=K\,\tau^{\uparrow}(\delta/2,\varepsilon;K)
					=\widetilde O\big(\varepsilon^{-2\kappa}\big).
				\end{equation}
			\end{theorem}

			Theorem \ref{thm: total complexity} analyzes the (human) sample complexity in a one-dimensional space, which is in the order of $\tilde{O}(\varepsilon^{-2\kappa})$.  Explicitly, $H=O\big(K\tau_0+Kg^{-2}\log(2K/\delta)\big)$ with $g=\min\{c\varepsilon^\kappa,\underline{\mu}_a\}$, and $K=O(\log(1/(\varepsilon\delta))$.  The complexity dependence of $\underline{\mu}_a^{-2}$ is the cost of comparisons at queries outside the $a$-neighborhood, and $\widetilde O(\varepsilon^{-2\kappa})$ is the cost of certifying an $\varepsilon$-accurate query inside it. %
		{}
		If humans have the ability to differentiate between any two parameters $\theta_1$ and $\theta_2$, with the probability strictly higher than $1/2 + \delta'$ as long as $\theta_1\neq \theta_2$, then the total complexity is in the logarithm order of both confidence $1/\delta$ and precision $1/\varepsilon$.

		In a $d$-dimensional space, if naively refining parameters along each dimension to make the parameter within an $\varepsilon$-ball (in $\ell_{\varrho}$-norm), then the total complexity is $dH(\delta/d,\varepsilon/d^{1/\varrho};a)$, which scales as $\tilde{O}(d^{1+\frac{2\kappa}{\varrho}}/\varepsilon^{2\kappa})$.  However, in high-dimensional spaces, $d$ can be significant even though the true parameter may lie within a low-dimensional space. In the following section, we explore leveraging a supervised learning oracle to reduce sample complexity.

		\section{Model Alignment: Human-AI Interaction}\label{sec: human-AI}

		In the previous section, we discussed the framework of using human feedback in a one-dimensional space. In practice, a noisy labeled dataset usually can be acquired easily. A noisy labeled dataset potentially can be used for learning the low-dimensional representation $\varphi(\cdot)$ and thus for reducing the complexity of the refining step. In this section, we propose a two-stage framework where the first stage learns the low-dimensional representation using noisy labels and the second stage refines the model using human feedback.

		\subsection{Two-stage Framework}\label{subsec: two-stage}

		Suppose the true parameter lies in an $s$-dimensional space where $s\ll d$. Given confidence $\delta$ and precision $\varepsilon$, we introduce a two-stage framework \twostage{}:
		\begin{enumerate}
			\item In Stage 1, the algorithm feeds $N_1^{HAI}(\delta,\varepsilon)$ noisy labeled data $(\bX,Y)$ to the supervised learning oracle (e.g., Lasso), where $\bX\in \mathbb{R}^{N_1^{HAI}\times d}$ is the design matrix and $Y$ is the noisy response. At the end, we learn the $s$-dimensional embedding $\varphi(\cdot)$.

			\item In Stage 2, the algorithm asks a human to refine the estimator along $s$ dimensions. This stage requires $N_2^{HAI}(\delta,\varepsilon)$ data points.
		\end{enumerate}

			\begin{algorithm}
				\caption{\textsf{SL+LHF} (Supervised Learning+Learning from Human Feedback)}\label{Alg: SFT+FTPB}
				\begin{algorithmic}[1]
					\State \textbf{Input}: confidence $\delta$, precision $\varepsilon$, norm index $\varrho$; support size $s$; half-width $\botheta$;  $\eta$
					\State \textsf{Stage 1}: Feed $N_1^{HAI}(\delta,\varepsilon)$ noisy labeled data $(\bX,Y)$ to the SL oracle (e.g., Lasso); obtain $\htheta^{(1)}$ and its selected support $\hat S$; set $\widehat\theta=\htheta^{(1)}$
					\State \textsf{Stage 2}: \textbf{for} each $i\in\hat S$: run Algorithm~\ref{Alg: RTB} on $\Theta_i=[\htheta^{(1)}_i-\botheta,\htheta^{(1)}_i+\botheta]$ with the uniform prior, confidence $\delta/s$, precision $\varepsilon/s^{1/\varrho}$, and prescribed budgets $\bar K,\bar m$ (Section~\ref{subsec: SFT+FTPB}); set $\widehat\theta_i$ to its output
					\State \Return $\widehat\theta$
				\end{algorithmic}
			\end{algorithm}

		The framework design is grounded in the insight that certain representation learning techniques, like autoencoders or deep learning architectures, are inherently robust to noisy labels \citep{li2021learning,taghanaki2021robust}. Stage 1 is dedicated to uncovering underlying patterns and structures within the data. Given the effectiveness of human expertise in model refinement, Stage 2 serves to enhance model alignment through targeted sample querying. In what follows, we specifically discuss sparse linear models (Example \ref{example: sparse linear}), where Stage 1 applies Lasso\footnote{OLS and thresholding may serve the same purpose as Lasso.} to select important features.

		\subsection{Illustration on Linear Models}
		To demonstrate the sample complexity reduction for the two-stage framework, compared to a framework trained purely by supervised learning, and to provide theoretical justifications, we focus on sparse linear models:
		\[
		Y_i = \inner{\boldsymbol{\vartheta^*}}{X_i}+\epsilon_i=\inner{\btheta^*}{\varphi(X_i)}+\epsilon_i,
		\]
		where $\varphi(\cdot)$ projects $X_i$ to all important features; $S$ is the index set of all important features; and the parameter $\btheta^*$ satisfies that $|\theta^*_i|\geq \lowbeta>0$ for all $i\in S$.
		In the initial stage, we use Lasso for the selection of significant features, followed by human comparisons for model alignment. To assess the sample complexity of our two-stage framework, we sequentially quantify the size of samples needed for the two stages.
		First, we establish the following standard assumptions for the analysis of Lasso.

		\begin{assumption}[sub-Gaussian noise]\label{assump: subGaussain}
			The observational noise is zero-mean i.i.d. sub-Gaussian with parameter $\sigma$.
		\end{assumption}

		\begin{assumption}\label{assump: three assumptions}
			{}Let $V=\frac{1}{n}\bX^\top\bX$ and $V_S=\frac{1}{n}\bX_S^\top\bX_S$ denote the (normalized) Gram matrices of the design $\bX\in\mathbb{R}^{n\times d}$ used in the first stage, and let $x_j$ denote the $j$th column of $\bX$. We assume that the realized design satisfies the following conditions:
				\begin{enumerate}
					\item Mutual incoherence:
					$\max_{j\in S^c} \ \|(\bX_S^\top\bX_S)^{-1}\bX_S^\top x_j\|_1 \leq \alpha_1$ for some $0\leq\alpha_1<1$;\label{assum: MI}
					\item Lower eigenvalue:
					$\lambda_{\min} \left(V_S\right)\geq c_{\min}>0$ (which requires $n\geq s$);
					\label{assum: LE}
					\item $\ell_\infty$-curvature condition:
					$\|V z\|_\infty \geq \alpha_2 \|z\|_\infty\,, \text{ for all } z\in C_{\alpha'}(S),$ for some $\alpha_2>0$,
					where $C_{\alpha'}(S) = \{\Delta\in \mathbb{R}^d|\|\Delta_{S^c}\|_1\leq \alpha' \|\Delta_S\|_1\}$;
					\label{assum: CC}
					\item Column normalization: $\max_{1\leq j\leq d}\|x_j\|_2/\sqrt{n}\leq C$ for some $C>0$.\label{assum: CN}
				\end{enumerate}
			{}
		\end{assumption}

		Assumption \ref{assump: three assumptions} is standard for the analysis of Lasso. For example, consider a random design matrix $\boldsymbol{X}\in \mathbb{R}^{n\times d}$ with i.i.d. $\mathcal{N}(0,1)$ entries; we can easily verify that it satisfies mutual incoherence, lower eigenvalue, and $\ell_{\infty}$-curvature condition with high probability.
		Under Assumptions \ref{assump: subGaussain} and \ref{assump: three assumptions} and based on Corollary 7.22 in \cite{wainwright2019high}, we derive the following theorem that bounds the infinite distance between vector $\htheta_S$ and $\theta^*_S$.

		\begin{theorem}\label{thm: lasso}
			Under Assumptions \ref{assump: subGaussain} and \ref{assump: three assumptions}, suppose that we solve the Lagrangian Lasso with regularization parameter
			$\lambda_n = \frac{2C\sigma}{1-\alpha_1}\left\{\sqrt{\frac{2\log(d-s)}{n}}+\zeta\right\}$
			for $\zeta = \frac{\lowbeta}{4}\min\left\{\frac{\sqrt{c_{\min}}}{\sigma},\frac{\alpha_2(1-\alpha_1)}{2C\sigma}\right\}$. Then, the optimal solution $\hat{\theta}$ is unique, with its support contained within $S$, and it satisfies the $\ell_\infty$-error bound
			\[
			\|\htheta_S-\theta^*_S\|_\infty \leq \frac{\sigma}{\sqrt{c_{\min}}}\left\{\sqrt{\frac{2\log s}{n}}+\zeta\right\}+\frac{1}{\alpha_2} \lambda_n=:{B_n}\,,
			\]
			all with probability at least $1-4e^{-\frac{n\zeta^2}{2}}$.
		\end{theorem}

		\begin{remark}
			Lasso is not limited to the linear function but can also be applied to the nonlinear function class \citep{plan2016generalized} or to a nonparametric variable selection \citep{li2023dimension}. Moreover, there is a rich stream of literature on representation learning that can be applied in the first stage \citep{bengio2013representation}.
		\end{remark}

		\begin{remark}[When supervised learning misspecifies the representation]\label{remark: stage1 failure}
				When the SL stage uses Lasso, the two possible selection errors are not symmetric. A \emph{false positive} (a spurious feature $j\in\hat{S}$ with $\theta^*_j=0$), which Theorem~\ref{thm: lasso} rules out under our assumptions, would in any case be self-correcting: the comparison stage drives its coefficient toward zero, at the cost of an additive $O(\varepsilon^{-2\kappa})$ term in sample complexity per spurious feature. A \emph{false negative} (a true feature missed when the signal strength $\lowbeta$ is small relative to the noise) is not correctable, because comparisons refine parameters only within the selected feature space and cannot discover new features; false negatives therefore cap the achievable alignment, much as RLHF refines capabilities that supervised fine-tuning has already installed but does not create new ones. The practical implication is to keep Stage~1 conservative, favoring false positives over false negatives, for instance through a milder Lasso penalty or an elastic net with a small $\ell_2$ term; Section~\ref{subsec: kickstarter} illustrates this regime on a real high-dimensional dataset.
		\end{remark}

		Theorem \ref{thm: lasso} implies that as long as $B_n<\lowbeta$, the variable selection is exact. %
			The choice $\zeta = \frac{\lowbeta}{4}\min\left\{\frac{\sqrt{c_{\min}}}{\sigma},\frac{\alpha_2(1-\alpha_1)}{2C\sigma}\right\}$ makes each of the two $\zeta$-terms in $B_n$ at most $\lowbeta/4$, namely $\frac{\sigma}{\sqrt{c_{\min}}}\zeta\leq \frac{\lowbeta}{4}$ and $\frac{1}{\alpha_2}\cdot\frac{2C\sigma}{1-\alpha_1}\zeta\leq \frac{\lowbeta}{4}$, so it remains to make the two sample-size terms at most $\lowbeta/8$ each,
			$\frac{\sigma}{\sqrt{c_{\min}}} \sqrt{\frac{2\log s}{n}}\leq \frac{\lowbeta}{8}
			$  and $
			\frac{1}{\alpha_2}\cdot \frac{2C\sigma}{1-\alpha_1}\sqrt{\frac{2\log(d-s)}{n}}\leq \frac{\lowbeta}{8},$
			which together give $B_n\leq \frac{3}{4}\lowbeta<\lowbeta$. Thus, if there are $n$ noisy samples with
			\[
			n\geq \max\left\{\frac{128\,\sigma^2 \log s}{\lowbeta^2 c_{\min}},\frac{512\,C^2\sigma^2\log(d-s)}{(\lowbeta\, \alpha_2(1-\alpha_1))^2}\right\},
			\]
			we recover the $s$ true variables with probability at least $1-4e^{-\frac{n\zeta^2}{2}}$. The confidence term $4e^{-n\zeta^2/2}\leq\delta$ requires $n\geq 2\log(4/\delta)/\zeta^2$, and since $\zeta\propto\lowbeta$ this term also scales as $\sigma^2\log(1/\delta)/\lowbeta^2$. Define
			\[
			H_0(\delta,\varepsilon;\lowbeta) = \max\left\{\frac{128\,\sigma^2 \log s}{\lowbeta^2 c_{\min}},\frac{512\,C^2\sigma^2\log(d-s)}{(\lowbeta\, \alpha_2(1-\alpha_1))^2}, \frac{2\log(4/\delta)}{\zeta^2}\right\} = O\!\left(\frac{\sigma^2\,(\log d+\log(1/\delta))}{\lowbeta^2}\right).
			\]
			Then, when we have more than $H_0(\delta,\varepsilon;\lowbeta)$ noisy labeled data at hand, we recover the true support with probability at least $1-\delta$. We note that this guarantee is for the prescribed penalty $\lambda_n$ above; a penalty chosen by cross-validation, as in the numerical experiments of Section~\ref{sec:  numerical}, is a practical surrogate that does not automatically inherit it.
		{}For a chosen $0<\botheta\leq\lowbeta$, recalibrate both $\zeta$ and $\lambda_n$ in Theorem~\ref{thm: lasso} with $\botheta$ in place of $\lowbeta$; the confidence term in $H_0(\delta,\varepsilon;\botheta)$ uses this same recalibrated $\zeta$. Then $n\geq H_0(\delta,\varepsilon;\botheta)$ gives $\hat S=S$ and $\|\htheta^{(1)}_S-\theta^*_S\|_\infty\leq B_n\leq3\botheta/4$ with probability at least $1-\delta$; we write $\cE_1=\{\hat S=S,\ \|\htheta^{(1)}_S-\theta^*_S\|_\infty\leq B_n\leq3\botheta/4\}$ for this good Stage-1 event. In particular, the interval $\Theta_i=[\htheta^{(1)}_i-\botheta,\htheta^{(1)}_i+\botheta]$ contains $\theta_i^*$ with margin at least $\botheta/4$ on this event. The refinement half-width $\botheta$ and the Stage-2 budgets are chosen before observing the Stage-1 data.
		The trade-off exists:
		A smaller value of $\botheta$ demands more noisy labels in the initial stage. However, limiting the feasible region to a smaller area would conserve the samples that are required for human comparisons. As previously noted, when both options significantly deviate from the true model, distinguishing the superior one becomes challenging. This challenge implies a  need for more samples to ascertain the correct answer. Consequently, in certain scenarios, opting for $\botheta\leq \lowbeta$ to confine the refinement region may be optimal. Specifically, we select $\botheta$ to minimize the total sample complexity:
		{}
			For a nonempty finite set $\mathcal R\subset(0,\lowbeta]$ of candidate widths, specified before Stage~1, use
			\begin{equation}\label{eq: optimal threshold}
				\botheta^*\in\operatorname*{arg\,min}_{r\in\mathcal R}
				\left\{ H_0(\delta,\varepsilon;r)
				+s\,H(\delta/s,\varepsilon/s^{1/\varrho};r)\right\}.
			\end{equation}
			{}Here $H(\delta,\varepsilon;r)=Km$ is the comparison budget of one scalar run from Theorem~\ref{thm: total complexity}, written with the half-width $r$ as its third argument because the run takes place on an interval of half-width $r$: the width enters through the prior constant $c_0=1/(2r)$ in the horizon $K$ and through the admissible radius $a\leq r/4$ and the outer drift over that interval in the cap $m$. Distance is measured in the $\ell_{\varrho}$ norm. %
		{}
		Note that although the first part depends on the noise $\sigma$, the second part relies not on $\sigma$ but on the detection accuracy. Hence, the optimal $\botheta^*$ is contingent on the function of detection accuracy.

		\subsection{Complexity of \twostage{}}\label{subsec: SFT+FTPB}

		The optimal value of $\botheta^*$ can be obtained numerically by enumerating and then comparing the objectives in Equation \eqref{eq: optimal threshold}. To establish an upper bound on the sample complexity, we set $\botheta=\lowbeta$ throughout this subsection.
		{}
			Each coordinate run is the scalar problem of Section~\ref{sec: human} for coordinate $i$ with target $\theta_i^*$. %
			In particular the constant $c$ of Lemma~\ref{lemma: convergence} and the outer drift $\underline{\mu}_a$, now taken over $a\leq|\theta-\theta_i^*|\leq2\lowbeta$,
			are common to all coordinates. %
			The uniform prior satisfies Assumption~\ref{assump: prior} with $c_0=1/(2\lowbeta)$, $\varsigma=1$, and $a_0=\lowbeta/4$. The horizon $K$ and cap $m$ of Section~\ref{sec: human} are therefore the same for every coordinate; we write $\bar K(q,t)$ and $\bar m(q,t)$ for their values at confidence $q=\delta/s$ and precision $t=\varepsilon/s^{1/\varrho}$. %

		\begin{theorem}[Sample complexity of \twostage{}]\label{thm: complexity of HAI}
			{}Under Assumptions~\ref{assump: utility}--\ref{assump: three assumptions}, by setting
				$N_1=\lceil H_0(\delta,\varepsilon;\lowbeta)\rceil$, $\bar K=\bar K(\delta/s,\varepsilon/s^{1/\varrho})$, and $\bar m=\bar m(\delta/s,\varepsilon/s^{1/\varrho})$, it holds that
				\[
				\Pr\bigl(d_\varrho(\widehat\theta,\theta^*)\leq\varepsilon\bigr)\geq1-2\delta.
				\]
				The number $N_2$ of Stage-2 comparisons satisfies $N_2\leq|\hat S|\,\bar K\bar m$ on every sample path; since $|\hat S|=s$ on $\cE_1$,
				with probability at least $1-\delta$ the total sample complexity is
				\begin{equation}\label{eq: uniform two stage cost}
					N_1+N_2\leq N_1+s\bar K\bar m
					=\widetilde O\!\left(\frac{\sigma^2}{\lowbeta^2}+s\,\underline{\mu}_a^{-2}+\frac{c^{-2}s^{1+2\kappa/\varrho}}{\varepsilon^{2\kappa}}\right).
				\end{equation}
			{}
		\end{theorem}

			Theorem~\ref{thm: complexity of HAI} gives an accuracy guarantee and a cost bound that holds on every sample path in terms of $|\hat S|$. For $\varrho=2$, its rate is $\widetilde O(\sigma^2/\lowbeta^2+s\,\underline{\mu}_a^{-2}+c^{-2}s^{1+\kappa}/\varepsilon^{2\kappa})$.
		{}
		Instead, if Lasso is used exclusively for the supervised learning oracle to refine the estimator based on $n$ noisy labeled data (subject to certain regularity conditions), the estimation error can be bounded as  $\|\htheta_n-\theta^*\|_2\leq c'\sigma\sqrt{s\log d/n}$ with high probability for some constant $c'>0$. Equivalently, we need to acquire $\tilde{O}(\sigma^2 s\log d/\varepsilon^2)$ data points to ensure that the two-norm distance of the estimation error is within $\varepsilon$.

		To compare the sample complexity of \twostage{} as indicated by Theorem \ref{thm: complexity of HAI} for $\varrho=2$, we characterize the condition of \twostage{} that requires less sample complexity than SL in Proposition \ref{prop: LNCA condition}.

		\begin{proposition}\label{prop: LNCA condition}
			For sufficiently small $\varepsilon$, \twostage{} requires less sample complexity than pure \textsf{SL} for solving the $(\varepsilon,\delta)$-alignment problem (ignoring the logarithm term) when:
			\begin{equation}\label{eq: compare condition}
				\frac{\sigma^2 }{s^\kappa} \gtrapprox \varepsilon^{2-2\kappa}.
			\end{equation}
		\end{proposition}

		We call the ratio $\sigma^2/s^\kappa$ the \emph{label noise-to-comparison-accuracy ratio} (LNCA ratio), and we call Condition \eqref{eq: compare condition} the \emph{LNCA condition}. This condition involves two important parameters: the observational noise $\sigma$ and the convergence rate of the accuracy $\kappa$. When the comparison accuracy is bounded away from 1/2, we have $\kappa=0$, and the LNCA condition reduces to $\sigma \gtrapprox \varepsilon$, which naturally holds  true. In the worst case scenario, where $\kappa=1$, the LNCA condition becomes $\sigma^2 \gtrapprox s $. When the observational noise of the sample is higher, the benefit of the human comparison step becomes more apparent.

		As previously mentioned, our framework can be extended to scenarios where the tasks differ between the initial feature-learning stage and the final implementation stage, provided the low-dimensional representation remains consistent. This is because the sample complexity of the second refinement stage in our algorithm remains unaffected as long as the representation stays unchanged. Meanwhile, Lasso must relearn the varying parameters. Therefore, the sufficient condition outlined in Proposition \ref{prop: LNCA condition} continues to hold.

		\emph{Estimating the LNCA ratio in practice.} The LNCA condition involves two parameters, the observational noise $\sigma$ and the accuracy-decay rate $\kappa$, that are not known a priori. Both can be estimated from a small amount of pilot data before committing the full budget: $\sigma^2$ from the residual variance of the Stage-1 supervised fit (or a small labeled validation set), and $\kappa$, together with the comparison-accuracy parameters, by maximum likelihood from a small batch of pilot comparisons, regressing observed correctness on the predicted utility gap. This is exactly the estimation we carry out in the Mechanical Turk study (Appendix~\ref{appendix: mturk}). The precise value of the ratio matters most when comparisons are expensive; when comparison data is cheap to obtain relative to labels (for instance, when an LLM or a quick expert judgment can supply it), the two-stage approach is preferable across a wide range of $\sigma$ and $\kappa$, so a decision maker can adopt it without pinning the ratio down exactly. From a high level, the LNCA ratio plays a role analogous to the signal-to-noise ratio in supervised learning: it is seldom known exactly in advance, yet it remains a standard and informative guide to how much data is needed and what accuracy to expect, so even an order-of-magnitude estimate already indicates which regime one is in and whether acquiring comparisons is worthwhile.

		\section{Practical Implementations}\label{sec: practical implementations}
		In the preceding sections, we established a framework that assumes the opportunity for humans to compare any two parametric models. However, such a comparison could be challenging when humans have access only to the models themselves. In practice, comparing models based on their performance with sampled data is often more useful because this engagement allows humans to discern which model aligns better with the true underlying model. To address this specific challenge, in this section, we present the practical aspects of comparing two models using selected samples.

		\subsection{Active Learning: How to Select Samples for Comparison}

		Active learning focuses on the strategic selection of samples to be labeled or compared in our context, with the aim of constructing prediction models in a resource-efficient manner. The key idea is to assess the significance of a sample to determine the value of acquiring its labels.
		Algorithm \ref{Alg: RTB} actively selects a pair of models for comparison. In practical scenarios, this choice often involves soliciting human judgment to compare the responses predicted by two different models for the given sample data, rather than directly comparing their model parameters.

		We start with an example. The clinician needs to evaluate the atherosclerotic cardiovascular disease (ASCVD) risk of patients $f_{\btheta}$, which is a linear function of contextual information, denoted as $x$. Such information comprises demographics (age, gender, racial/ethnic group, and education level); baseline conditions (body mass index, history of ASCVD events, and smoking status); clinical variables (glycosylated hemoglobin (A1c) and systolic blood pressure (sbp)); anti-hypertensive agents; and anti-hyperglycemic agents.

		Given a patient's context $x$ with features $\bz=\varphi(x)$, the clinician judges a model by its predicted risk, so the utility is the prediction discrepancy $u_x(\btheta)=-|\bz^\top\btheta-\bz^\top\btheta^*|$ of Section~\ref{subsec: utility model}, and comparing two models on $x$ is a scalar comparison of their predictions with target the true prediction $y^*=\bz^\top\btheta^*$. If $\varphi(x)=e_i$, the prediction is the coordinate $\theta_i$ itself, the target is $\theta_i^*$ whatever the other coordinates are, and Stage~2 can refine the coordinates one at a time. For a general $\bz$, varying coordinate $i$ with the others held at their current estimates targets $(y^*-\sum_{j\neq i}z_j\widehat\theta_j)/z_i$ rather than $\theta_i^*$, so such a comparison is not a directional oracle for that coordinate. %

		However,
		in practice, the dataset $\mathcal{X}$ may be limited, potentially leading to situations where there are no samples $x \in \mathcal{X}$ that satisfy the condition $\varphi(x) = e_i$. In such cases, refining the model parameter along each dimension separately becomes impractical. The question then arises: How can we efficiently select samples to refine the model parameter? To address this challenge, we propose a procedure for constructing a new basis of comparison using the available samples in $\mathcal{X}$. This procedure aims to optimize the alignment process by leveraging the existing data to guide the selection of samples for parameter enhancement.

		\subsection{New Basis Construction through the Gram Schmidt Process}

		To facilitate the independent alignment of each coordinate, we aim to identify a new set of orthogonal basis, which can be constructed using $\Psi := \{\varphi(x) : x \in \mathcal{X}\}$. Let $\balpha_1,\cdots, \balpha_s$ represent these new bases, and our objective is to learn $\boldsymbol{\theta}^* = \sum_{i=1}^s \omega_i^* \boldsymbol{\alpha}_i$, where $\omega_i^*$ are the corresponding coefficients. We describe the process of constructing the orthogonal basis $\{\balpha_1,\cdots, \balpha_s\}$ from the sample space.

		First, we pick a set of samples that spans the space of $\Psi$, denoted as $\varphi(x_1),\cdots, \varphi(x_s)$. For simplicity, we define $\bz_i=\varphi(x_i)$, and we write $\bZ\in\mathbb{R}^{s\times s}$ for the matrix with rows $\bz_1^\top,\ldots,\bz_s^\top$. The samples must be chosen so that $\bZ$ is \emph{well conditioned}, that is, $\sigma_{\min}(\bZ)\geq\sigma_0$ for a prescribed $\sigma_0>0$.
			We assume that the finite pool $\cX$ contains such a set.
		{}
		Then, we use the %
		Gram Schmidt process to construct a set of orthogonal bases, based on $\bz_1,\cdots, \bz_s$:
		\[
		\balpha_k = \bz_k- \sum_{j=1}^{k-1} \frac{\bz_k^\top \balpha_j}{\balpha_j^\top \balpha_j} \balpha_j, \quad \forall 1\leq k\leq s.
		\]

		We align the coordinates $\omega_1,\ldots,\omega_s$ in sequence. At the $k$th step we evaluate on the sample $\bz_k=\balpha_k+\sum_{j<k}\frac{\bz_k^\top\balpha_j}{\balpha_j^\top\balpha_j}\balpha_j$, whose true prediction is
			$\bz_k^\top\btheta^*=\sum_{j<k}\omega_j^*\,\bz_k^\top\balpha_j+\omega_k^*\,\balpha_k^\top\balpha_k .$
			Algorithm~\ref{Alg: RTB} refines this prediction to a value $y(x_k)$ within $\varepsilon_y$ of $\bz_k^\top\btheta^*$ with high probability, and the $k$th coordinate is set to
			\[
			\hat{\omega}_k=\frac{y(x_k)-\sum_{j<k}\hat{\omega}_j\,\bz_k^\top\balpha_j}{\balpha_k^\top\balpha_k},
			\]
			the first coordinate being the case without a correction term, $\hat\omega_1=y(x_1)/\|\bz_1\|_2^2$.
		By asking a human to evaluate $x_1,\cdots, x_s$ and by refining $\omega_1,\cdots, \omega_s$ in sequence, we can refine the parameter $\btheta:=\sum_{i=1}^s \hat{\omega}_i \balpha_i$ within a certain accuracy level. We include the pseudo code in Algorithm~\ref{Alg: ACS}.

		Note that the set of orthogonal bases is not unique; multiple sets may exist, and recognizing this possibility is essential. Consequently, we can conduct multiple rounds of alignment, with the option to switch the basis from one round to another.

		\paragraph{Calibration.} For the guarantee below, use the Stage-1 event $\cE_1$ of Theorem~\ref{thm: complexity of HAI}, on which $\|\hat{\btheta}^{(1)}-\btheta^*\|_\infty\leq3\lowbeta/4$ in the selected feature space. For each row of the matrix, set
			\begin{equation}\label{eq: acs prediction intervals}
				w_k=\lowbeta\|\bz_k\|_1,\qquad
				I_k=[\bz_k^\top\hat{\btheta}^{(1)}-w_k,\ \bz_k^\top\hat{\btheta}^{(1)}+w_k],
				\qquad a_{0,k}=w_k/4.
			\end{equation}
			Run $k$ holds $x_k$ fixed and searches over the prediction $y=\bz_k^\top\btheta$ itself, whose target under the prediction discrepancy is the true prediction $y_k^*=\bz_k^\top\btheta^*$ (Section~\ref{subsec: utility model}); on $\cE_1$, $y_k^*$ lies in $I_k$ with margin $a_{0,k}$.
			The constant $c$ of Lemma~\ref{lemma: convergence} and the outer drift $\underline{\mu}_a$, now taken over $a\leq|y-y_k^*|\leq2\max_kw_k$, are then common to all runs, and the uniform prior on $I_k$ satisfies Assumption~\ref{assump: prior} with $c_0\geq1/(2\max_kw_k)$, $\varsigma=1$, and $a_0\geq a$. The constants are set as
			\begin{equation}\label{eq: acs calibrated budgets}
				\varepsilon_y=\frac{\varepsilon\,\sigma_{\min}(\bZ)}{\sqrt{s}},\qquad
				t_y=\min\{\varepsilon_y,a\},\qquad
				K_y=\bar K(\delta/s,t_y),\quad m_y=\bar m(\delta/s,t_y),
			\end{equation}
			with $\bar K,\bar m$ as in Section~\ref{subsec: SFT+FTPB}, evaluated at $c_0=1/(2\max_kw_k)$; the internal precision $t_y$ keeps every run within the range of the scalar theorems. For the absolute prediction error, $c$ and $\underline{\mu}_a$ are explicit (Appendix~\ref{appendix: acs calibration}).

			\begin{algorithm}
				\caption{\textsf{ACS} (Active Comparison Selection)}\label{Alg: ACS}
				\begin{algorithmic}[1]
					\State \textbf{Input}: finite pool $\cX$, dimension $s$, $\sigma_0>0$, confidence $\delta\in(0,1)$ and precision $\varepsilon>0$\;
					\State Set $I_k,a_{0,k}$ by \eqref{eq: acs prediction intervals}, and $t_y,K_y,m_y$ by \eqref{eq: acs calibrated budgets}
					\For{$k=1,\ldots,s$}
					\State $\displaystyle \balpha_k=\bz_k-\sum_{j<k}\frac{\bz_k^\top\balpha_j}{\balpha_j^\top\balpha_j}\balpha_j$
					\EndFor
					\For{$k=1,\ldots,s$}
					\State Hold $x_k$ fixed and run Algorithm~\ref{Alg: RTB} for prediction $y=\bz_k^\top\btheta$ on $I_k$, with the uniform prior, confidence $\delta/s$, precision $t_y$, and prescribed budgets $K_y,m_y$; obtain $y(x_k)$
					\State $\displaystyle \hat\omega_k=\frac{y(x_k)-\sum_{j<k}\hat\omega_j\bz_k^\top\balpha_j}{\balpha_k^\top\balpha_k}$
					\EndFor
					\State \Return $\widehat{\btheta}=\sum_{k=1}^s\hat\omega_k\balpha_k$
				\end{algorithmic}
			\end{algorithm}

		Prediction errors propagate across coordinates through the correction terms, and the propagation is governed by the conditioning of the selected samples: the updates of Algorithm~\ref{Alg: ACS} are forward substitution for the system $\bZ\btheta=\hat{\by}$ with $\hat{\by}=(y(x_1),\ldots,y(x_s))^\top$, so $\widehat{\btheta}=\bZ^{-1}\hat{\by}$ and
			\[
			d_2(\widehat{\btheta},\btheta^*)\leq\frac{\sqrt{s}\,\max_{k}|y(x_k)-y_k^*|}{\sigma_{\min}(\bZ)}\leq \frac{\sqrt{s}\,\max_{k}|y(x_k)-y_k^*|}{\sigma_0}.
			\]

			\begin{proposition}[Guarantee for Algorithm~\ref{Alg: ACS}]\label{prop: acs guarantee}
				Under Assumptions~\ref{assump: utility}--\ref{assump: three assumptions}, suppose Stage~1 uses $N_1=H_0(\delta,\varepsilon;\lowbeta)$ labeled samples as in Theorem~\ref{thm: complexity of HAI} and the pool contains a sample set with $\sigma_{\min}(\bZ)\geq\sigma_0$. Then Algorithm~\ref{Alg: ACS} returns $\widehat{\btheta}$ with
				\[
				\Pr\bigl(d_2(\widehat{\btheta},\btheta^*)>\varepsilon\bigr)\leq2\delta .
				\]
				Moreover, the sample complexity is of order $\widetilde O\big(s\,\underline{\mu}_a^{-2}+s^{1+\kappa}c^{-2}\sigma_0^{-2\kappa}\varepsilon^{-2\kappa}\big)$.
			\end{proposition}
			{}Thus the sample complexity carries over from the setting of Section~\ref{sec: human-AI}, where the samples $\varphi(x)=e_i$ isolate the coordinates, to a general sample pool.%
			{}
		{} %

		\section{Numerical Experiment}\label{sec:  numerical}

		This section provides a comprehensive numerical evaluation of the framework, organized in three parts. The main experiment (Section~\ref{subsec: kickstarter}) is a real data task, with high-dimensional features: Section~\ref{appendix: kickstarter parity} measures, under a fixed query budget, when trading labels for comparisons improves upon pure supervised learning. This section also benchmarks our active refinement against an RLHF-style preference learning. Section~\ref{appendix: kickstarter llm} repeats the same experiment, but utilizes real large language models to compare candidate predictions. Second, synthetic simulations, deferred to Appendix~\ref{appendix: synthetic}, evaluate the theoretical results of the paper in a controlled setting: the roles of the label noise $\sigma$, the expertise level $\gamma$, the accuracy-decay rate $\kappa$, and the dimension of non-zero features $s$.
			Third, to also include a setting where the comparisons are made by real humans, we ran a controlled study on Amazon Mechanical Turk in which human participants judged a point-counting task; estimating $\kappa$ and $\sigma$ directly from the human answers and feeding these estimates into the framework confirms the LNCA prediction: \twostage{} reaches a target accuracy with substantially fewer samples than pure \textsf{SL} when label noise is high and human comparisons are reliable, and its advantage grows as the precision target $\varepsilon$ tightens. Full details of the MTurk design, the estimation of $\kappa$ and $\sigma$, and the error-ratio curves are in Appendix~\ref{appendix: mturk}.

			\subsection{Real-Data Experiment: Kickstarter Funding Outcomes}\label{subsec: kickstarter}

			We demonstrate the \twostage{} pipeline on a real high-dimensional task, predicting Kickstarter funding outcomes from $d=9{,}828$ covariates, a regime where the Stage-1 Lasso could be misspecified.

			\subsubsection{The value of comparisons under a fixed query budget}\label{appendix: kickstarter parity}

			Here we study when trading labels for comparisons pays off, under an exact accounting of queries.

			\emph{Dataset, belief, and simulated comparator.} Kickstarter is a crowdfunding platform: a creator posts a project (a product, a game, a film, and so on) with a funding goal and a deadline, and the campaign is funded only if backers pledge at least the goal by the deadline. The task is to predict a campaign's funding outcome from its listing at launch: the target outcome for prediction is the \emph{log funding ratio}, logarithm of the ratio of pledges to the goal, and the $d=9{,}828$ covariates comprise listing attributes (e.g., goal, duration, category, country), market conditions on the launch date, TF--IDF features of the title and blurb, and category-by-country interactions. We use a Web Robots scrape of completed campaigns\footnote{\url{https://webrobots.io/kickstarter-datasets/}} launched between January 2024 and April 2026, $71{,}023$ in total, split at random into $n_{\text{train}}=56{,}818$ training and $n_{\text{test}}=14{,}205$ test campaigns. The human's expertise is modeled by a belief: a Lasso trained on the entire training set with the test set excluded (test $R^2=0.459$). This belief is the knowledge the \emph{simulated comparator} draws on: the comparator judges two candidate models by their agreement with the belief, following the random utility model of Section~\ref{sec: model}. Appendix~\ref{appendix: kickstarter details} details the data construction, the belief, and the mechanics of the simulated comparator.%

			\emph{Query budget.} Pure \textsf{SL} spends its entire budget of $N$ labels. SL+LHF spends a fraction $\rho=0.9$ of the budget on Stage-1 labels ($0.9N$) and the remaining $0.1N$ queries on Stage-2 comparisons, so the two methods spend exactly the same total of $N$ queries. This accounting is conservative since a comparison is typically cheaper than generating a label, which would favor SL+LHF further.

			\emph{Comparator and refinement.} Each comparison is the random utility model of Section~\ref{sec: model}: on a batch of $B=8$ projects, the evaluator judges which of two candidate set of $B$ predictions (granularity $\Delta=0.2$) yields a smaller mean squared error (MSE) against the estimate of the belief model described above, answering correctly with probability $\sigma\!\big((\text{MSE gap})/\gamma\big)$ where $\sigma(\cdot)$ is the Sigmoid function. The parameter $\gamma$ is the expertise/noise scale, a small $\gamma$ meaning a sharp, near-expert evaluator; we report $\gamma\in\{0.25,0.01\}$. Stage~2 uses a practical line-search version of MAPB and active comparison selection: each comparison is asked on the batch of $B=8$ pool projects with the largest projections on the current refinement direction (chosen within a random subsample of $200$, the discriminative-gap selection).
			We use $\rho=0.9$, a pool of $4{,}000$ unlabeled projects, the min-CV Lasso for both fits, and report means over $30$ independent repetitions, each drawing at random the $N$ budgeted campaigns (shared by pure \textsf{SL} and SL+LHF) and the pool from the training set; all $R^2$ are on the test set.

			\begin{table}[!th]
				\centering
				\caption{Kickstarter under a fixed query budget (both methods spend exactly $N$ queries): test $R^2$ (mean over $30$ seeds) of pure \textsf{SL} and SL+LHF, versus  $N$, for two expertise levels $\gamma$. Entries in \textbf{bold} match or exceed pure \textsf{SL}.}
				\label{tab: kickstarter parity}
				{\small
					\begin{tabular}{r ccc}
						\hline

						$N$ & Pure \textsf{SL} & SL+LHF ($\gamma=0.25$) & SL+LHF ($\gamma=0.01$) \\
						\hline
						$100$ & $-0.009$ & $\mathbf{0.060}$ & $\mathbf{0.062}$ \\
						$150$ & $0.062$ & $\mathbf{0.093}$ & $\mathbf{0.094}$ \\
						$200$ & $0.107$ & $\mathbf{0.119}$ & $\mathbf{0.120}$ \\
						$250$ & $0.149$ & $\mathbf{0.151}$ & $\mathbf{0.149}$ \\
						$300$ & $\mathbf{0.176}$ & $0.169$ & $0.170$ \\
						$500$ & $\mathbf{0.249}$ & $0.233$ & $0.234$ \\
						\hline
				\end{tabular}}
			\end{table}

			\emph{Findings.} From the results in Table~\ref{tab: kickstarter parity}: %
			\begin{enumerate}
				\item \textbf{Comparisons help most when labels are scarce.} At small budgets the pure-\textsf{SL} Lasso is very weak (test $R^2$ near zero at $N=100$), and the second stage delivers a large, significant gain: $+0.071$ at $N=100$ for $\gamma=0.01$ and $+0.069$ for $\gamma=0.25$ (winning on $87$--$90\%$ of seeds), still significant at $N=150$ (a gain of ~$+0.03$); the paired $95\%$ intervals for the gain over seeds are $[+0.047,+0.094]$ at $N=100$ and $[+0.014,+0.050]$ at $N=150$ for $\gamma=0.01$, and they include zero from $N=200$ on.
				\item \textbf{A crossover as $N$ increases.} As labels accumulate the advantage shrinks toward zero (neutral by $N\approx250$--$300$) and turns negative by $N=500$ ($-0.016$ at $\gamma=0.01$ and $-0.017$ at $\gamma=0.25$): once the supervised model is already accurate, diverting budget from labels to a handful of comparisons is harmful. This is how the LNCA condition shows up in the empirics: comparisons are worthwhile precisely in the label-scarce, high-noise regime.
				\item \textbf{Active selection softens the need for expertise.} A twenty-five-fold noisier evaluator costs surprisingly little: at $N=100$ the gain falls only from $+0.071$ ($\gamma=0.01$) to $+0.069$ ($\gamma=0.25$), and both evaluators stay at or above pure \textsf{SL} through $N=250$. The reason is that  the discriminative-gap batches used in this implementation keep the utility gap of Lemma~\ref{lemma: precision} large at almost every query, so even the noisy evaluator answers reliably, and a well-designed query partially substitutes for evaluator expertise.
			\end{enumerate}

			\emph{Comparison to an RLHF-style refinement.} A natural question is whether these gains are specific to our active refinement or would arise from any method that consumes the same comparisons. We therefore compare our method (SL+HF), whose second stage uses an adaptive line-search variant inspired by Section~\ref{sec: practical implementations}, against a refinement in the style of reinforcement learning from human feedback (RLHF) using direct preference optimization (SL+DPO), holding everything else fixed. Both are two-step and share the identical Stage-1 warm start and selected support, the identical random utility comparison procedure of Section~\ref{sec: model}, and the identical comparison budget of $(1-\rho)N$ queries; the only difference is how that budget is spent. ACS selects its comparisons \emph{actively}, by the probabilistic bisection of Section~\ref{sec: human}. DPO instead collects its comparisons offline around the fixed warm start and fits the parameter by a Bradley--Terry preference likelihood. To separate the effect of which comparisons are asked, we also fit the same DPO objective on the exact comparison pairs that SL+HF generated (``SL+DPO on ACS queries''). Because the Bradley--Terry log-odds between two candidate predictions is linear in the parameter (the quadratic terms cancel), the DPO step reduces to a warm-started logistic regression on the support, regularized toward the Stage-1 model by a prediction-space penalty that plays the role of DPO's KL-to-reference term; like most RLHF pipeline it uses only the binary comparison outcomes.

			\begin{table}[!th]
				\centering
				\caption{Kickstarter at fixed query parity (equal total query budget $N$): test $R^2$ (mean over $30$ seeds) of our refinement (SL+HF) and two RLHF-style offline preference fits, for two expertise levels $\gamma$. SL+DPO uses passively sampled comparisons; the ``on ACS queries'' variant is fit on the  comparison pairs that SL+HF generated. All share the same Stage-1 warm start, support, comparison belief, and budget; pure \textsf{SL} is $\gamma$-independent. Bold marks the best entry in each $\gamma$ block.}
				\label{tab: kickstarter dpo}
				{\small
					\begin{tabular}{r c ccc ccc}
						\hline
						& & \multicolumn{3}{c}{$\gamma=0.25$} & \multicolumn{3}{c}{$\gamma=0.01$} \\
						$N$ & \textsf{SL} & SL+HF & SL+DPO & \shortstack{SL+DPO\\(on ACS queries)} & SL+HF & SL+DPO & \shortstack{SL+DPO\\(on ACS queries)} \\
						\hline
						100 & $-0.009$ & $\mathbf{0.060}$ & $-0.022$ & $0.038$ & $\mathbf{0.062}$ & $-0.019$ & $0.042$ \\
						150 & $0.062$ & $\mathbf{0.093}$ & $0.055$ & $0.069$ & $\mathbf{0.094}$ & $0.057$ & $0.071$ \\
						200 & $0.107$ & $\mathbf{0.119}$ & $0.090$ & $0.081$ & $\mathbf{0.120}$ & $0.101$ & $0.087$ \\
						250 & $0.149$ & $\mathbf{0.151}$ & $0.106$ & $0.071$ & $\mathbf{0.149}$ & $0.130$ & $0.067$ \\
						300 & $\mathbf{0.176}$ & $0.169$ & $0.140$ & $0.124$ & $0.170$ & $0.157$ & $0.111$ \\
						500 & $\mathbf{0.249}$ & $0.233$ & $0.200$ & $0.171$ & $0.234$ & $0.218$ & $0.183$ \\
						\hline
				\end{tabular}}
			\end{table}

			Table~\ref{tab: kickstarter dpo} and Figure~\ref{fig: kickstarter parity} report the two refinements at same query budget. SL+HF beats SL+DPO at every budget and both expertise levels: SL+DPO fails even to reach pure \textsf{SL} (it is negative at $N=100$), because its passively sampled comparisons fall largely along directions where the warm start is already accurate and the update brings unnecessary noise, whereas SL+HF improves on \textsf{SL} when labels are scarce and stays close to it otherwise. Specifically, bisection concentrates its few queries on one direction at a time, which is ideal for a sequential local update but poorly conditioned for a global logistic fit. (DPO's reference-regularization strength was fixed at a value giving stable behavior; the results is qualitatively insensitive to this choice.)

			\begin{figure}[!th]
				\centering
				\includegraphics[width=.62\linewidth]{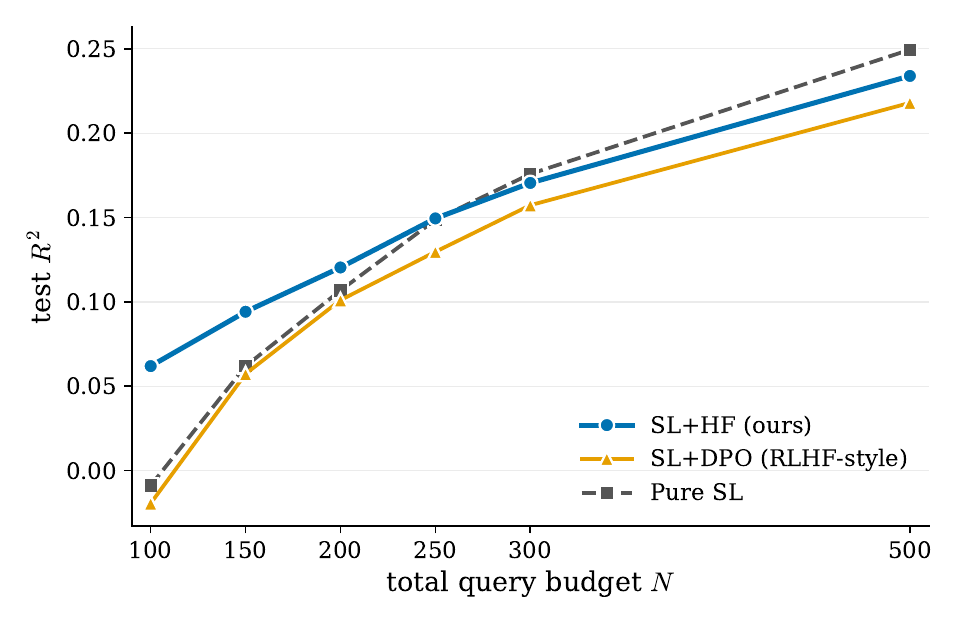}
				\caption{Kickstarter comparison ($\rho=0.9$, $\gamma=0.01$, means over $30$ seeds): test $R^2$ versus the total query budget $N$. At equal total budget our refinement (SL+HF) is above pure \textsf{SL} for $N\leq250$ and slightly below it at $N=300$ and $500$, whereas the RLHF-style SL+DPO is below pure \textsf{SL} at every budget. Both consume the same Stage-1 warm start, comparison oracle, and budget; they differ only in active bisection versus passive preference fitting.}
				\label{fig: kickstarter parity}
			\end{figure}

			These experiments isolate the positive side of the framework, when Stage~1 is informative enough that refinement helps; Section~\ref{appendix: kickstarter llm} next reproduces this value-of-comparisons picture with real large language models as the comparators.

			\subsubsection{A real LLM comparator}\label{appendix: kickstarter llm}

			In previous section, the comparisons were generated by a simulated belief model. We now replace the simulated evaluator with real LLMs and ask whether the value-of-comparisons finding of Section~\ref{appendix: kickstarter parity} persists. The task and the same refinement procedure will stay intact.

			\emph{How the LLM is queried.} Each comparison is a single API call to an LLM. For each project, in the batch of size $B=8$, the model sees the project's real text (name, blurb, category, country, funding goal, and launch date) and the two candidate models' predicted funding ratios, labeled A and B in an order randomized independently per project, together with a handful of labeled Stage-1 projects as in-context calibration (also known as few shot learning). The model answers with one letter per project (a comma-separated list of $B$ answers); each letter counts as a vote for one candidate, and the $B$ votes are aggregated by sign into the comparison outcome $Z$. Specifically, each letter is mapped back through the per-project randomization of the A/B order to a vote $v_k=+1$ if the model chose the candidate with the larger step on project $k$ and $v_k=-1$ if it chose the candidate with the smaller step ; the comparison outcome is the majority, with a tie broken randomly.

			\emph{Choosing the comparison granularity.} The geometry of the two candidates governs whether the comparison is answerable, through the granularity trade-off discussed in Section~\ref{subsec: PB}: accuracy degrades when the two options are too similar (the $\kappa$ regime of Assumption~\ref{assump: lip}) and also when both are far from the truth (Example~\ref{example: far-away answer}). We set the query geometry with two theory-guided choices: the batch consists of the projects with the largest projections on the current refinement direction (the discriminative-gap principle of ACS), and the granularity is set to $\Delta=0.6$, so that the two candidate ratios differ by a factor of $2.5$ to $4$.

			\emph{Procedure and accounting.} We follow Section~\ref{appendix: kickstarter parity}, except the comparison outcome now comes from the LLM's choices.
			We test the total query budget $N\in\{100,150,200,250,300,500\}$ as before. We report, in Table~\ref{tab: kickstarter llm}, the average over ten random seeds for each of Claude Haiku~4.5, Sonnet~4.6, and Opus~4.6.%

			\begin{table}[!th]
				\centering
				\caption{Kickstarter with real LLM comparators, at fixed query budget: test $R^2$ (mean over ten seeds) versus the total query budget $N$. Each API call shows a batch of projects with two candidate predicted ratios each, in per-project randomized order. Entries in \textbf{bold} match or exceed pure \textsf{SL}.}
				\label{tab: kickstarter llm}
				{\small
					\begin{tabular}{r r ccc}
						\hline
						$N$ & Pure \textsf{SL} & SL+HF (Haiku 4.5) & SL+HF (Sonnet 4.6) & SL+HF (Opus 4.6) \\
						\hline
						$100$ & $-0.019$ & $\mathbf{0.044}$ & $\mathbf{0.067}$ & $\mathbf{0.065}$ \\
						$150$ & $0.038$  & $\mathbf{0.084}$ & $\mathbf{0.102}$ & $\mathbf{0.092}$ \\
						$200$ & $0.102$  & $\mathbf{0.108}$ & $\mathbf{0.111}$ & $\mathbf{0.114}$ \\
						$250$ & $0.136$  & $0.133$          & $\mathbf{0.146}$ & $\mathbf{0.143}$ \\
						$300$ & $\mathbf{0.176}$ & $0.151$  & $0.172$          & $0.171$ \\
						$500$ & $\mathbf{0.249}$ & $0.109$  & $0.212$          & $0.192$ \\
						\hline
				\end{tabular}}
			\end{table}

			\emph{Findings.} Two conclusions.
			\begin{enumerate}
				\item \textbf{Similar crossover as $N$ increases.} At the smallest budget the supervised fit is near-useless (pure \textsf{SL} $R^2=-0.019$ at $N=100$) and all three models deliver a large gain of $+0.063$ (Haiku), $+0.086$ (Sonnet), and $+0.084$ (Opus), each beating pure \textsf{SL} on $9$ of $10$ seeds; the gain remains clearly positive at $N=150$ (Sonnet wins on $10$ of $10$ seeds). The advantage then shrinks to neutral by $N\approx200$--$250$ and turns negative past $N=300$ (at $N=500$ no model beats pure \textsf{SL} on more than $1$ of $10$ seeds), where diverting budget from labels to comparisons is harmful. Because comparisons and labels cost the same here, this is a demonstration of Proposition~\ref{prop: LNCA condition}, and it closely tracks the simulated comparator's own crossover (Table~\ref{tab: kickstarter parity}).
				\item \textbf{The weakest evaluator loses most.} The choice accuracy, measured against the true outcomes, orders the models as Haiku ($0.52$--$0.60$) below Sonnet ($0.61$--$0.69$) and Opus ($0.64$--$0.70$). Haiku is also the weakest refiner at every budget, whereas Sonnet and Opus are nearly tied, with Sonnet slightly ahead at five of the six budgets.%
				{} Moreover, Haiku's accuracy \emph{decays} with $N$, from $0.59$ at $N=100$ to $0.52$ at $N\geq300$: as Stage~1 improves, the bisection's candidates tighten around the truth and the comparison hardens, which is the accuracy-decay mechanism of Assumption~\ref{assump: lip} realized by a genuine evaluator, and its $N=500$ loss is correspondingly the largest ($-0.141$, versus $-0.037$ for Sonnet and $-0.057$ for Opus).
			\end{enumerate}

		\section{Discussions and Concluding Remarks}\label{sec: conclusion}

		This work is motivated by the ever-growing observations of human-AI interaction in LLM. We presented a theoretical two-stage framework for explaining the significance of using human comparisons for model improvement. In the first stage, the noisy-labeled data is fed into the SL procedure to learn the low-dimensional representation; we then ask human evaluators to make pair-wise comparisons in the second stage. To address the challenge of  efficient acquisition of valuable information from these human comparisons, we introduced a probabilistic bisection method, which factors in the uncertainty that arises from the accuracy of human comparisons. The newly introduced concept, the LNCA ratio, quantifies the relative scale between label noise and human comparison accuracy.
		{}%
			Real-data experiments on Kickstarter crowdfunding, with both simulated evaluators and real LLM comparators, confirm the LNCA prediction that comparisons add the most value when labels are scarce and noisy, and show that the same comparisons are worth more to our active refinement than to an RLHF-style preference fit.

		\subsection{Implications for Supervised Fine-Tuning}

		Supervised fine-tuning (SFT) is the process of refining a pre-trained model on a labeled dataset with specific input-output pairs tailored to a target task. Our proposed framework can fit into SFT in the following way. First, the pre-trained model produces a representation, denoted as $\hat{\varphi}(\cdot)$, along with initial parameters $\hat{\btheta}_0$. For a new task, if $\hat{\varphi}(\cdot)$ remains the same, we proceed directly to the second phase, where human feedback is solicited for model alignment through comparison data. However, if the representation $\hat{\varphi}(\cdot)$ also requires adaptation, we first update the embeddings by training on a combination of newly collected and existing noisy-labeled data, using techniques such as transfer learning \citep{pan2020transfer}. In this adapted embedding space, human comparisons are then actively acquired in the second stage to enhance model alignment.

		\subsection{Future Directions}

		Our work represents an early attempt to organically integrate two sources of data: noisy-labeled data and human comparison data. We
		position this paper as a prompt for an open discussion. Some potential extensions to our work
		are worth investigating. First, we extended the probabilistic bisection algorithm from one dimension to multiple dimensions by refining along each dimension. This expansion raises the question of whether a more efficient bisection algorithm exists in multi-dimensional spaces. As discussed in \cite{frazier2019probabilistic}, the development of a PBA tailored to multi-dimensional problems remains an open problem. Second, although our paper focuses on a prediction problem, the framework has the potential to be extended to action-based learning, where the goal is to select optimal actions for the decision-making problem. Third, exploring the optimal sample selection to maximize the learning algorithm's efficiency is a compelling area of inquiry. Although we have offered guidelines for practical sample selection, the quest for the most efficient proposed procedure remains open. %

		\bibliographystyle{informs2014}
		\bibliography{references}

		\newpage

		\begin{APPENDICES}
			\section{Proofs}
			\subsection{Proofs in Section \ref{sec: human}}\label{appendix: proof for human sec}
			\begin{proof}{Proof of Lemma~\ref{lemma: precision}.}
				{}
					Condition on the sample, both candidates, and the history before their Gumbel shocks are drawn. Write $v_j=u_x(\btheta_j)$ and let $\xi_1,\xi_2$ be independent Gumbel variables of scale $\gamma$ and common location, which can be set to zero. Then $e^{-\xi_j/\gamma}$ are independent unit-rate exponential variables, so $e^{-(v_j+\xi_j)/\gamma}$ are independent exponentials of rates $e^{v_j/\gamma}$. Selecting the larger random utility is equivalent to selecting the smaller exponential variable. Therefore
					\[
					\Pr(v_1+\xi_1>v_2+\xi_2)
					=\frac{e^{v_1/\gamma}}{e^{v_1/\gamma}+e^{v_2/\gamma}}
					=\frac{1}{1+e^{(v_2-v_1)/\gamma}}.
					\]
					If $v_1>v_2$, this is the probability of a correct strict choice and is larger than $1/2$. Exchanging the candidates proves the case $v_2>v_1$. If $v_1=v_2$, both probabilities are exactly $1/2$. A common additive evaluator bias cancels from $v_2-v_1$.
				{}
				\QED
			\end{proof}

			\begin{proof}{Proof of Proposition~\ref{prop: bisection}.}
				{}
					Initially the closed interval contains $\theta^*$. Suppose it contains the target before a query at its arithmetic midpoint $\theta$. If $\theta<\theta^*$, \eqref{eq: scalar comparison direction} makes the plus candidate strictly preferable, and the retained interval $[\theta,\theta^+]$ contains the target. If $\theta>\theta^*$, the analogous statement holds for $[\theta^-,\theta]$. If $\theta=\theta^*$, either tie-breaking choice retains the target as an endpoint. Thus containment is preserved on every round.

					Each retained numerical interval has exactly half the previous width, so after $j$ comparisons its width is $2\beta_\Theta2^{-j}$. For $j\geq\lceil\log_2(\beta_\Theta/\varepsilon)\rceil$ this is at most $2\varepsilon$ (no comparison is needed when $\beta_\Theta\leq\varepsilon$). Its midpoint has distance at most half that width from every point in the interval, in particular from $\theta^*$. This proves the claimed precision and number of comparisons.
				{}
				\QED
			\end{proof}

			\begin{proof}{Proof of Lemma~\ref{lemma: detection accuracy}.}
				{}
					Condition on the history before a fixed query. At $\theta=\theta^*$, \eqref{eq: scalar drift identity} gives independent fair signs. Set $Q_m=(S_m+m)/2$, the number of plus signs. Under this fair-sign law,
					\[
					M_m=\int_0^1(2v)^{Q_m}\{2(1-v)\}^{m-Q_m}\,dv
					=\frac{2^m}{(m+1)\binom{m}{Q_m}}
					\]
					is a nonnegative martingale with $M_0=1$: for each $v$, the integrand is the likelihood-ratio martingale for Bernoulli parameter $v$ against parameter $1/2$. If $|S_m|\geq\hbar_m$, then $|Q_m-m/2|\geq\sqrt{(m/2)\log((m+1)/\eta)}$. The appropriate one-sided Hoeffding bound gives
					\[
					\binom{m}{Q_m}2^{-m}\leq\frac{\eta}{m+1},
					\]
					and hence $M_m\geq1/\eta$. The maximal inequality for nonnegative martingales yields $\Pr(\tau_1^\uparrow(\theta^*)<\infty)\leq\eta$. Reflection of the fair walk shows that its first exit is through each boundary with probability at most $\eta/2$.

					For $\theta>\theta^*$, the plus-sign probability is below $1/2$ by \eqref{eq: scalar comparison direction}. Couple its signs with fair signs using common independent uniforms; its partial sums are everywhere no larger than those of the fair walk. A first exit through the upper boundary for the negative-drift walk entails a first exit through the upper boundary for the fair walk, so the wrong-direction probability is at most $\eta/2$. The positive-drift case follows by reflection. At any query other than the target the drift is nonzero, and $S_m/m$ converges almost surely to this drift while $\hbar_m/m\to0$. Therefore the uncapped test terminates almost surely there and has directional accuracy at least $1-\eta/2$.
				{}
				\QED
			\end{proof}

			The following statements set up the horizontal analysis of the uncapped reference sequence of Section~\ref{subsubsec: horizontal moves}; Section~\ref{subsec: horizontal move} uses them only through Theorem~\ref{thm: complexity}.

			We first characterize the complexity of horizontal moves.
			Fix some $a\in (0,\theta^*+\beta_\Theta)$. Define three regions: $A=[-\beta_\Theta, \theta^*-a]$, $B=(\theta^*-a,\theta^*]$, and $C=(\theta^*,\beta_\Theta]$. In words, $A$ is the \emph{outer} region, whose points lie at least an $a$-distance to the left of $\theta^*$, while $B$ is the left $a$-neighborhood of $\theta^*$ and $C$ is the entire right-hand side of $\theta^*$. These regions organize the convergence argument: we first confine the sequence of medians to the $a$-neighborhood of $\theta^*$ by showing it eventually leaves $A$ (and, symmetrically, the right-hand outer region $A'$ defined below) for good, and then analyze the finer behavior within $B$ and $C$. Define $T(a)=\inf\{k'\geq 0: \theta_k\in B\cup C, \forall k\geq k'\}$ to be the time required until the sequence of medians never reenters $A$. Let $\nu_k$ denote the (random) measure corresponding to the conditional posterior distribution, conditional on lying to the left of $\theta^*$, so that for any measurable $D$, $\nu_k(D)=\mu_k(D\cap [-\beta_\Theta,\theta^*])/\mu_k([-\beta_\Theta,\theta^*])$. Next, fix $\Delta\in(0,1/2)$ and define the stopping time $\tau(a)$ to be the first time that the conditional mass in $B$ lies above $1-\Delta$ and the median $\theta_k$ lies to the left of $\theta^*$; that is,
			\[
			\tau(a)=\inf\{k\geq 0: \nu_k(B)>1-\Delta, \theta_k<\theta^*\}.
			\]

			In line with the approach taken by \cite{frazier2019probabilistic} in establishing a connection between $T(a)$ and $\tau(a)$, we present a similar result in Lemma \ref{lemma: T(a) and tau(a)}.

			\begin{lemma}\label{lemma: T(a) and tau(a)}
				It holds that
				\begin{equation}\label{eq: Ta decomposition1}
					T(a)\leq_{s.t.} \tau(a) + R,
				\end{equation}
				where $R$ is a random variable that is independent of $\tau(a)$ and $a$.
			\end{lemma}

			\begin{proof}{Proof of Lemma \ref{lemma: T(a) and tau(a)}.}
				Let $U_0=\tau(a)$. Lemma 3 in \cite{frazier2019probabilistic} establishes that $U_0$ is finite a.s. Now, for $j\geq 1$, we recursively define
				\[
				V_j = \inf\{i>U_{j-1}: \nu_i(A)\geq 1/2\}, \text{ and } U_j=\inf\{i>V_j:\nu_i(B)\geq 1-\Delta\}.
				\]
				Here, $V_j$ represents the first time after time $U_{j-1}$ that the conditional mass in $A$ becomes at least $1/2$, and $U_j$ is the first time after $V_j$ that the conditional mass in $B$ is once again at least $1-\Delta$. $U_j$ and $V_j$ for $j\geq 1$ taking the value $\infty$ implicitly implies that the corresponding event does not occur. Let $\Gamma$ be the number of ``cycles'', i.e., $\Gamma=\sum_{j=1}^\infty \1(V_j<\infty)$. It holds that
				\begin{equation}\label{eq: T(a)}
					T(a) = U_0 + \sum_{j=1}^{\Gamma} [(V_j-U_{j-1})+(U_j-V_j)].
				\end{equation}
				In the above expression, $V_j-U_{j-1}$ represents the $j^{th}$ time required to increase the conditional  mass in $A$ from below $\Delta$ to 1/2 or more.  The quantity $U_j-V_j$ (conditional on $\Gamma\geq j$, i.e., conditional on both $U_j$ and $V_j$ being finite), represents the number of steps required to return the conditional mass to at least $1-\Delta$, starting from a point where the conditional mass in $A$ is at least $1/2$.

				Using Equation \eqref{eq: T(a)}, it can be shown that (Equation (6) in \cite{frazier2019probabilistic})
				\begin{equation}\label{eq: Ta decomposition}
					T(a)\leq_{s.t.} \tau(a) + R,
				\end{equation}
				where $R$ is a random variable that is independent of $\tau(a)$ and $a$.
				\QED

			\end{proof}

				\begin{lemma}\label{lemma: prior}
					{}Under Assumption \ref{assump: prior}, by setting $r=r_2\alpha/(4\varsigma)$,
						\[
						\Pr(\tau(\omega e^{-r k})>k/2)\leq e^{-r_2 \alpha k/4} \cdot \frac{1}{c_0\,\omega^\varsigma} +\beta_1 e^{-r_1 k/2},
						\]
						for some $r_1, r_2, \alpha, \beta_1 >0$ and any $\omega>0$ and $k\geq 0$ with $\omega e^{-rk}\leq a_0$.
				\end{lemma}

				Lemma \ref{lemma: prior} shows the light tail distribution of $\tau(\cdot)$. The tail is lighter when $c_0$ is larger, that is, when the initial prior places more mass in the neighborhood of the true parameter.

				\begin{proof}{Proof of Lemma \ref{lemma: prior}.}
					For ease of notation, let $\tau = \tau(a)=\inf\{k\geq 0: \nu_k(B)>1-\Delta, \theta_k<\theta^*\}$, where $0<a\leq a_0$.
					Recall that
					$\nu_k(D)=\mu_k(D\cap [-\beta_\Theta,\theta^*])/\mu_k([-\beta_\Theta,\theta^*])$.

					Fix $\iota\in(0,1/2)$. Define
					\[
					M_k = e^{r_2 \tilde{N}(k\wedge \tau)}/\nu_{k\wedge \tau}(B),
					\]
					where
					\[
					\tilde{N}(k) = \sum_{j=0}^{k-1} \1(\mu_j([-\beta_\Theta,\theta^*])\geq 1/2+\iota)
					\]
					is the number of instances from time $0$ to time $k-1$ when the mass of $A\cup B$ is at least $1/2+\iota$ (the count starts at the initial query, as in the occupation lemma of \cite{frazier2019probabilistic} imported below). Here and below, for odd $k$ we read $k/2$ as $\lfloor k/2\rfloor$ in $\tilde N(k/2)$ and in the events $\{\tau>k/2\}$ and $\{R>k/2\}$; since $\lfloor k/2\rfloor\geq(k-1)/2$, this only multiplies the exponential bounds by the fixed factors $e^{r_1/2}$, $e^{r_2\alpha/2}$, and $e^{r_R/2}$, which we absorb into $\beta_1$, $\beta_2$, and the leading constant.
					It has been shown in Lemma 2 in \cite{frazier2019probabilistic} that $\{M_k:k\geq 0\}$ is a supermartingale with respect to the filtration $\{\mathscr{F}_k:k\geq 0\}$, for some $r_2$ (which depends only on $\Delta$, $\eta$, $p_c$, and $p$). Its initial value is deterministic and, by the definition of $\nu_0$ and Assumption \ref{assump: prior},
					\[
					M_0=\frac{1}{\nu_0(B)}=\frac{\mu_0([-\beta_\Theta,\theta^*])}{\mu_0(B)}\leq \frac{1}{\mu_0(B)}\leq \frac{1}{c_0\,a^{\varsigma}},
					\]
					where the first inequality uses $\mu_0([-\beta_\Theta,\theta^*])\leq 1$. Since $\nu_{k\wedge\tau}(B)\leq 1$, the supermartingale property gives
					\[
					\E[e^{r_2 \tilde{N}(k\wedge \tau)}] \leq \E\left[\frac{e^{r_2 \tilde{N}(k\wedge \tau)}}{\nu_{k\wedge \tau}(B)}\right]=\E[M_k]\leq M_0 \leq \frac{1}{c_0\,a^\varsigma}.
					\]
					Since $\tilde{N}(\cdot)$ is nondecreasing, $e^{r_2 \tilde{N}(k\wedge \tau)}\uparrow e^{r_2 \tilde{N}(\tau)}$ as $k\rightarrow \infty$, and monotone convergence yields
					\[
					\E [e^{r_2 \tilde{N}(\tau)}] = \lim_{k\rightarrow \infty} \E[e^{r_2 \tilde{N}(k\wedge \tau)}]\leq \frac{1}{c_0\,a^\varsigma}.
					\]
					For $a=\omega e^{-r k}\leq a_0$, we obtain
					\begin{equation}\label{eq: lemma 4 inequality 1}
						\E[e^{r_2 \tilde{N}(\tau(\omega e^{-rk}))}]\leq  \frac{1}{c_0(\omega e^{-r k})^\varsigma} = \frac{e^{\varsigma rk}}{c_0\,\omega^{\varsigma}}.
					\end{equation}
					Lemma 1 in \cite{frazier2019probabilistic} shows that there exist $\alpha,\beta_1,r_1>0$ such that
					\begin{equation}\label{lemma 1 in frazier}
						\Pr(\tilde{N}(k/2)\leq \alpha k/2)\leq \beta_1 e^{-r_1 k/2}.
					\end{equation}
					Hence, we can bound
					\[
					\begin{aligned}
						\Pr(\tau(\omega e^{-r k})>k/2) &\leq \Pr(\tilde{N}(\tau(\omega e^{-r k}))\geq \tilde{N}(k/2))\\
						&\leq \Pr(\tilde{N}(\tau(\omega e^{-r k}))\geq \tilde{N}(k/2), \tilde{N}(k/2)>\alpha k/2) + \Pr(\tilde{N}(k/2)\leq \alpha k/2)\\
						&\stackrel{(*)}{\leq} \Pr(\tilde{N}(\tau(\omega e^{-r k}))>\alpha k/2) + \beta_1 e^{-r_1 k/2}\\
						&=\Pr(e^{r_2 \tilde{N}(\tau(\omega e^{-r k}))}>e^{r_2\alpha k/2}) + \beta_1 e^{-r_1 k/2} \\
						&\stackrel{(**)}{\leq}  e^{-r_2 \alpha k/2} \E[e^{r_2 \tilde{N}(\tau(\omega e^{-r k}))}] +\beta_1 e^{-r_1 k/2}\\
						&\stackrel{(***)}{\leq}  e^{-r_2 \alpha k/2} \cdot \frac{e^{\varsigma r k}}{c_0\,\omega^\varsigma} +\beta_1 e^{-r_1 k/2}\\
						&=e^{\varsigma r k-r_2 \alpha k/2} \cdot \frac{1}{c_0\,\omega^\varsigma} +\beta_1 e^{-r_1 k/2},
					\end{aligned}
					\]
					where (*) applies \eqref{lemma 1 in frazier}, (**) uses Markov's inequality, and (***) applies \eqref{eq: lemma 4 inequality 1}.
					By setting $r=\frac{r_2\alpha}{4\varsigma}$, we have
					\[
					\Pr(\tau(\omega e^{-r k})>k/2)\leq e^{-r_2 \alpha k/4} \cdot \frac{1}{c_0\,\omega^\varsigma} +\beta_1 e^{-r_1 k/2}.
					\]
					\QED
				\end{proof}

				\begin{proof}{Proof of Theorem \ref{thm: complexity}.}
					Lemma \ref{lemma: T(a) and tau(a)} has established that
					\[
					T(a) \leq_{s.t.} \tau(a) + R,
					\]
					where $\tau(a)$ and $R$ are independent random variables with appropriately light tails and the non-negative random variable $R$ does not depend on $a$.

					For arbitrary $\omega>0$, we have
					\begin{align}\label{eq: inequality 1 in Theorem 1}
						&\Pr(e^{rk}|\theta_k-\theta^*|\1(\theta_k\leq \theta^*)>\omega) \leq \Pr(\theta^*-\theta_k>\omega e^{-rk})\notag\\
						\leq &\Pr(T(\omega e^{-rk})>k) \leq \Pr(\tau(\omega e^{-rk})+R>k)\notag\\
						\leq & \Pr(\tau(\omega e^{-rk})>k/2) + \Pr(R>k/2)\,.
					\end{align}
					Lemma \ref{lemma: prior}
					establishes that, whenever $\omega e^{-rk}\leq a_0$,
					\begin{equation}\label{eq:bound tau}
						\Pr(\tau(\omega e^{-r k})>k/2)\leq e^{-r_2 \alpha k/4} \cdot \frac{1}{c_0\,\omega^\varsigma} +\beta_1 e^{-r_1 k/2},
					\end{equation}
					where $r=r_2\alpha/(4\varsigma)$. The same bound holds for the right-hand outer region $A'=[\theta^*+a,\beta_\Theta]$ by the mirror-image argument: let $T'(a)$ be the last-exit time after which the medians never reenter $A'$, let $\nu'_k(D)=\mu_k(D\cap[\theta^*,\beta_\Theta])/\mu_k([\theta^*,\beta_\Theta])$ be the posterior conditional on lying to the right of $\theta^*$, let $B'=[\theta^*,\theta^*+a)$, and let $\tau'(a)=\inf\{k\geq0:\nu'_k(B')>1-\Delta,\ \theta_k>\theta^*\}$. Lemma~\ref{lemma: T(a) and tau(a)} applies verbatim to the reflected problem and gives $T'(a)\leq_{s.t.}\tau'(a)+R'$ with $R'$ having the tail \eqref{eq: bound R} below, and the proof of Lemma~\ref{lemma: prior} applies with $B'$ in place of $B$, since it uses the prior only through $\mu_0(B')\geq c_0a^\varsigma$ (the second half of Assumption~\ref{assump: prior}) and the trivial bound $\mu_0([\theta^*,\beta_\Theta])\leq1$. This does not rely on any symmetry of the prior about $\theta^*$.

					Moreover, Lemma 9 in \cite{frazier2019probabilistic} shows that $R$ has an exponentially decaying tail, i.e.,
					\begin{equation}\label{eq: bound R}
						\Pr(R>k/2) \leq \beta_2 e^{-r_R k},
					\end{equation}
					where $r_R$ does not depend on $a$.

					Fix some $\omega>0$ (which we will define later) and define
					\[
					\tau^{\rightarrow}(\delta,\varepsilon) = \max \left\{\frac{\log\left(\omega/\varepsilon\right)}{r}, \frac{4\log\big(8/(c_0\,\delta\,\omega^\varsigma)\big)}{r_2 \alpha}, \frac{2\log(8\beta_1/\delta)}{r_1}, \frac{\log(8 \beta_2/\delta)}{r_R}\right\}.
					\]
					For any $k$ such that $k \geq \tau^{\rightarrow}(\delta,\varepsilon)$, the first term guarantees $\omega e^{-rk}\leq \varepsilon\leq a_0$, so \eqref{eq:bound tau} applies, and
					the probability of the precision $|\theta_k-\theta^*|$ being larger than $\varepsilon$ can be bounded by
					\[
					\begin{aligned}
						&\Pr\left(|\theta_k-\theta^*|>\varepsilon\right)\\
						=& \Pr\left(|\theta_k-\theta^*|  e^{r k}>\varepsilon  e^{r k }\right)\\
						\leq&\Pr(e^{rk}|\theta_k-\theta^*|>\omega)\\
						=&\Pr(e^{rk}|\theta_k-\theta^*|\1(\theta_k\leq \theta^*)>\omega)+ \Pr(e^{rk}|\theta_k-\theta^*|\1(\theta_k>\theta^*)>\omega)\\
						\stackrel{\eqref{eq: inequality 1 in Theorem 1}}{\leq} &2\left(\Pr(\tau(\omega e^{-rk})>k/2) + \Pr(R>k/2)\right)\\
						\stackrel{\eqref{eq:bound tau},\eqref{eq: bound R}}{\leq}  & 2 \left(e^{-r_2 \alpha k/4} \cdot \frac{1}{c_0\,\omega^\varsigma} +\beta_1 e^{-r_1 k/2}+\beta_2 e^{-r_R k}\right)\\
						\leq & \delta,
					\end{aligned}
					\]
					where the first inequality holds from the definition of $\tau^{\rightarrow}(\delta,\varepsilon)$ (namely $\varepsilon e^{rk}\geq \omega$); and the last inequality holds since $k\geq \tau^\rightarrow(\delta,\varepsilon)$ implies
					\[
					\quad e^{-r_2 \alpha k/4}\frac{1}{c_0\,\omega^\varsigma}\leq \delta/8, \quad \beta_1 e^{-r_1 k/2}\leq \delta/8, \quad  \beta_2 e^{-r_R k}\leq \delta/8.
					\]
					Specifically, we choose $\omega$ so that the two logarithms appearing in the first two terms coincide (the terms themselves differ by the factor $\varsigma$),
					\[
					\frac{\omega}{\varepsilon}= \frac{8}{c_0\,\delta\,\omega^\varsigma},
					\qquad\text{i.e.}\qquad
					\omega = \left(\frac{8\varepsilon}{c_0\,\delta}\right)^{\frac{1}{1+\varsigma}}.
					\]
					Writing $L:=\log\big(8/(c_0\,\delta\,\varepsilon^{\varsigma})\big)$, this choice gives $\omega/\varepsilon=\big(8/(c_0\delta\varepsilon^\varsigma)\big)^{1/(1+\varsigma)}$, hence $\log(\omega/\varepsilon)=L/(1+\varsigma)$ and $\log\big(8/(c_0\delta\omega^\varsigma)\big)=\log(\omega/\varepsilon)=L/(1+\varsigma)$. Substituting $r=r_2\alpha/(4\varsigma)$,
					\[
					\frac{\log(\omega/\varepsilon)}{r}=\frac{4\varsigma\,L}{(1+\varsigma)\,r_2\alpha},
					\qquad
					\frac{4\log\big(8/(c_0\delta\omega^\varsigma)\big)}{r_2\alpha}=\frac{4\,L}{(1+\varsigma)\,r_2\alpha}.
					\]
					Since $\varsigma\leq 1$, the second of these dominates the first, and therefore
					\[
					\tau^\rightarrow(\delta,\varepsilon) = \max \left\{\frac{4\log\big(\frac{8}{c_0\,\delta\,\varepsilon^{\varsigma}}\big)}{(1+\varsigma)\,r_2\alpha},  \frac{2\log(8\beta_1/\delta)}{r_1}, \frac{\log(8\beta_2/\delta)}{r_R}\right\}.
					\]
					\QED
				\end{proof}

			\begin{lemma}\label{lemma: finite cover}
				{}Under Assumption~\ref{assump: utility} and the RUM, let $I=[-\beta_\Theta,\beta_\Theta]$ and $a>0$. If $I_a=\{\theta\in I:|\theta-\theta^*|\geq a\}$ is nonempty, then
					\[
					d_a:=\min_{\theta\in I_a}|D_\Delta(\theta)|>0,\qquad
					\inf_{\theta\in I_a}\{2\tilde p(\theta)-1\}
					=\tanh(d_a/(2\gamma))>0.
					\]
					In particular the outer comparison accuracy is uniformly greater than $1/2$.
			\end{lemma}
			\begin{proof}{Proof of Lemma~\ref{lemma: finite cover}.}
				{}The set $I_a$ is compact and excludes $\theta^*$. The lower semicontinuous function $|D_\Delta|$ attains its minimum on this set, and the directional condition makes its value positive at every point of the set. Thus $d_a>0$. The formula for the drift magnitude in \eqref{eq: scalar drift identity}, and monotonicity of $\tanh$, give the displayed equality. This provides a model-specific positive bound; Section~\ref{subsec: SFT+FTPB} takes the minimum of this bound over coordinates on a fixed range, which does not depend on the Stage-1 data.
				\QED
			\end{proof}

			\begin{proof}{Proof of Lemma \ref{lemma: convergence}.}
				{}By the scalar directional condition and \eqref{eq: scalar drift identity}, the directional accuracy (with value $1/2$ at the target) is
				\[
				\tilde{p}(\theta) = \frac{1}{1+ \exp(-|u(c_\Delta^+(\theta))-u(c_\Delta^-(\theta))|/\gamma)}\,.
				\]
				For ease of notation, define $z=|u(c_\Delta^+(\theta))-u(c_\Delta^-(\theta))|$. Then the condition $\tilde{p}(\theta)-1/2\geq c|\theta-\theta^*|^\kappa$ is equivalent to
				\begin{align}\label{eq: eq condition}
					&\frac{1}{1+\exp(-z/\gamma)}-\frac{1}{2}\geq c|\theta-\theta^*|^\kappa \notag\\
					\Leftrightarrow & z/\gamma \geq -\ln\left(\frac{1-2c|\theta-\theta^*|^\kappa}{1+2c|\theta-\theta^*|^\kappa}\right) = -\ln\left(1-\frac{4c|\theta-\theta^*|^\kappa}{1+2c|\theta-\theta^*|^\kappa}\right)\,.
				\end{align}

				When $|\theta-\theta^*|\leq a$, we can choose  $c<1/(6a^\kappa)$ such that
				\[
				\frac{4c|\theta-\theta^*|^\kappa}{1+2c|\theta-\theta^*|^\kappa}<\frac{1}{2}.
				\]
				Note that $-\ln(1-x)\leq 2\ln (2) x$ when $x<1/2$. Then substituting $x$ with $\frac{4c|\theta-\theta^*|^\kappa}{1+2c|\theta-\theta^*|^\kappa}$, it holds that
				\[
				-\ln\left(1-\frac{4c|\theta-\theta^*|^\kappa}{1+2c|\theta-\theta^*|^\kappa}\right)\leq 2\ln (2) \frac{4c|\theta-\theta^*|^\kappa}{1+2c|\theta-\theta^*|^\kappa}.
				\]
				Therefore, to achieve the condition \eqref{eq: eq condition}, we only need
				\[
				z/\gamma\ge 8\ln(2)c|\theta-\theta^*|^\kappa \geq 2\ln(2)\frac{4c|\theta-\theta^*|^\kappa}{1+2c|\theta-\theta^*|^\kappa}\,,
				\]
				then we will have
				\[
				z/\gamma\ge 8\ln(2)c|\theta-\theta^*|^\kappa \geq 2\ln(2)\frac{4c|\theta-\theta^*|^\kappa}{1+2c|\theta-\theta^*|^\kappa}\geq  -\ln\left(1-\frac{4c|\theta-\theta^*|^\kappa}{1+2c|\theta-\theta^*|^\kappa}\right)\,.
				\]
				Thus, we only need
				\[
				|u(c_\Delta^+(\theta))-u(c_\Delta^-(\theta))|\geq 8\ln(2)c\gamma |\theta-\theta^*|^\kappa\,,
				\]
				then the condition \eqref{eq: eq condition} holds.
				Under Assumption \ref{assump: lip}, we choose $c=\min\{\lambda_\Delta/(8\gamma \ln(2)), 1/(6a^\kappa)\}$ and the conclusion holds.
				\QED
			\end{proof}

				The conditionally independent oracle and the standing assumption $\theta_k\neq\theta^*$ give the uncapped reference defined in Section~\ref{subsubsec: horizontal moves}. The following lemma transfers its horizontal convergence guarantee to the capped algorithm; it is stated and used in Section~\ref{subsec: vertical move}.

				\begin{lemma}[Coupling at reached queries]\label{lemma: horizontal coupling}
					Suppose the assumptions of Theorem~\ref{thm: complexity} hold, let $\delta\in(0,1)$ and $0<\varepsilon\leq a_0$, and fix deterministic integers $K,m\geq1$ with $K\geq\tau^{\rightarrow}(\delta/2,\varepsilon)$. Couple Algorithm~\ref{Alg: RTB}, run with budgets $K,m$, to the uncapped reference using the same prior, update parameter, and potential comparison signs. Let $R_k$ be the event that the algorithm reaches its median after $k$ horizontal updates; in particular, $R_K$ is the event that it returns through the horizontal horizon. Then, almost surely,
					\[
					R_k\subseteq\{\theta_k=\theta_k^{\mathrm{ref}},\ \mu_k=\mu_k^{\mathrm{ref}}\},
					\qquad 0\leq k\leq K,
					\]
					and consequently
					\begin{equation}
						\Pr\bigl(R_K\cap\{|\theta_K-\theta^*|>\varepsilon\}\bigr)\leq\frac{\delta}{2}.\label{eq: adaptive horizontal final}
					\end{equation}
					The statement also holds conditionally on a preceding history $\cG$ when its assumptions hold conditionally and $K,m$ are fixed given $\cG$.
				\end{lemma}

				\begin{proof}{Proof of Lemma~\ref{lemma: horizontal coupling}.}
					Construct the reference using conditionally independent potential outcome sequences at its successive queries. Equivalently, use independent uniform random variables and the conditional plus-sign probability at each query to generate the signs. The standing assumption $\theta_k^{\mathrm{ref}}\neq\theta^*$ and Lemma~\ref{lemma: detection accuracy} ensure that every reference test terminates almost surely. Thus all finite reference histories exist, and each reference direction has conditional accuracy at least $p_c$ before its signs are collected. Because the law of Algorithm~\ref{Alg: RTB} is determined by the same oracle, this construction realizes the algorithm in distribution, so the probability of any event determined by its path, such as the one bounded below, is unchanged by the coupling.

					The two processes start from the same prior, so their initial medians agree. Suppose they agree at a reached query $k<K$. Give the actual test the same signs as the reference test until it stops. On $R_{k+1}$, its first boundary crossing occurs at some sample $\ell\leq m$. This is also the first crossing of the uncapped reference test, with the same sign; both therefore make the same posterior update. This proves the displayed coupling by induction. A crossing at $\ell=m$ is included because the algorithm gives boundary crossing priority over early termination. If no crossing occurs by $m$, the actual algorithm returns and the reference alone continues; agreement is required only at reached queries.

					Theorem~\ref{thm: complexity} applies to this uncapped reference. Hence
					\[
					\begin{aligned}
						\Pr\bigl(R_K\cap\{|\theta_K-\theta^*|>\varepsilon\}\bigr)
						&=\Pr\bigl(R_K\cap\{|\theta_K^{\mathrm{ref}}-\theta^*|>\varepsilon\}\bigr)\\
						&\leq\Pr\bigl(|\theta_K^{\mathrm{ref}}-\theta^*|>\varepsilon\bigr)
						\leq\delta/2.
					\end{aligned}
					\]
					This uses event inclusion, not directional accuracy conditional on crossing before the cap. For the conditional version, perform the same construction under the conditional law given $\cG$ and apply the horizontal theorem with the resulting fixed prior and a sufficient horizon. The argument holds for almost every such preceding history.
					\QED
				\end{proof}

				\begin{proof}{Proof of Theorem~\ref{thm: micro stop}.}
					{}An output is inaccurate if its distance from $\theta^*$ exceeds $\varepsilon$. The algorithm returns either early, at a query whose directional test has not crossed by sample $m$, or through the final horizontal output at index $K$. We show that an early output is returned and is inaccurate with probability at most $\delta/2$; the final output is inaccurate with probability at most $\delta/2$ by \eqref{eq: adaptive horizontal final}, and the two routes together give the claim. The bound $Km$ on the number of comparisons holds on every path because each of the at most $K$ reached queries collects at most $m$ comparisons. Fix a query index $0\leq k<K$. Let $\cF_k$ contain the history before collecting comparisons at query $k$, including whether the query is reached and its location when reached. On $R_k$, set
					\[
					b_k=\E[\widetilde Z_{k,i}\mid\cF_k],\qquad
					\overline Z_{k,m}=\frac{1}{m}\sum_{i=1}^{m}\widetilde Z_{k,i},\qquad
					\mathcal B_k=R_k\cap\{|\theta_k-\theta^*|>\varepsilon\}.
					\]
					The query and $\mathcal B_k$ are $\cF_k$-measurable. Comparison outcomes through sample $m$ may be defined even when the directional test finishes earlier, without requiring the algorithm to collect them. At unreached queries, assign arbitrary values to these quantities; all error events under consideration include $R_k$.

					The sign of $b_k$ points toward $\theta^*$ and $|b_k|=2\tilde p(\theta_k)-1$. On $\mathcal B_k$ two cases arise. If $\varepsilon<|\theta_k-\theta^*|\leq a$, Lemma~\ref{lemma: convergence} gives
					\begin{equation}\label{eq: adaptive drift bound}
						|b_k|\geq2c|\theta_k-\theta^*|^\kappa\geq c\varepsilon^\kappa\geq g.
					\end{equation}
					If $|\theta_k-\theta^*|>a$, then $|b_k|\geq\underline{\mu}_a\geq g$ by the definition \eqref{eq: adaptive outer drift}. Hence $|b_k|\geq g$ on all of $\mathcal B_k$.

					Let $\mathcal A_k=R_k\cap\{\tau_1^{\uparrow}(\theta_k)>m\}$ denote early termination at query $k$. On $\mathcal A_k$ the directional boundary has not been crossed through sample $m$, so the boundary check at the cap ensures
					\[
					|\overline Z_{k,m}|<\frac{\hbar_m}{m}\leq\frac{g}{2}
					\]
					by the definition of $\tau_0$. Combining this with $|b_k|\geq g$, on $\mathcal A_k\cap\mathcal B_k$ we obtain
					\[
					\operatorname{sgn}(b_k)(\overline Z_{k,m}-b_k)
					\leq|\overline Z_{k,m}|-|b_k|
					\leq-\frac{g}{2}.
					\]
					Conditional on $\cF_k$, the sign of $b_k$ is fixed, and the new signs are independent with range of length $2$. The one-sided Hoeffding inequality and \eqref{eq: adaptive vertical budget} consequently imply
					\begin{align}
						\Pr(\mathcal A_k\cap\mathcal B_k\mid\cF_k)
						&\leq\1(\mathcal B_k)\exp\left(-\frac{m g^{2}}{8}\right)
						\leq\1(\mathcal B_k)\frac{\delta}{2K}.
						\label{eq: adaptive conditional hoeffding}
					\end{align}
					Taking expectations gives $\Pr(\mathcal A_k\cap\mathcal B_k)\leq\delta/(2K)$. This conditions only on the history before sampling the query, not on the event that it becomes terminal, and it uses no localization event: the case split above covers every reached query regardless of its distance from $\theta^*$.

					An inaccurate early output implies $\mathcal A_k\cap\mathcal B_k$ for some $k<K$. Hence
					\begin{align}
						\Pr\{\text{an inaccurate early output is returned}\}
						\leq\sum_{k=0}^{K-1}\Pr(\mathcal A_k\cap\mathcal B_k)
						\leq K\frac{\delta}{2K}=\frac{\delta}{2}.
						\label{eq: adaptive union bound}
					\end{align}
					This union bound covers the adaptively selected terminal index without substituting it into a fixed-query concentration bound, and independence between different query locations is not needed.

					If the algorithm instead reaches its deterministic horizontal horizon, an inaccurate output is contained in $R_K\cap\{|\theta_K-\theta^*|>\varepsilon\}$, whose probability is at most $\delta/2$ by Lemma~\ref{lemma: horizontal coupling} and \eqref{eq: adaptive horizontal final}; the lemma constructs the common reference and proves agreement at every reached query. The two output routes therefore have total error probability at most $\delta$. Finally, every directional test stops by sample $m$ and there are at most $K$ queries, so the algorithm terminates on every sample path after at most $Km$ comparisons, and its output is within $\varepsilon$ of $\theta^*$ with probability at least $1-\delta$. This is an unconditional guarantee.%
					\QED
				\end{proof}

				\begin{proof}{Proof of Corollary~\ref{thm: complexity pc}.}
					{}Set $g_0=2\lowp-1$, let $\tau_{0,p}=\min\{n\in\N:\hbar_s/s\leq g_0/2\text{ for all }s\geq n\}$, and let $m_p=\lceil\max\{\tau_{0,p},\,8g_0^{-2}\log(2K/\delta)\}\rceil$ be the cap \eqref{eq: adaptive vertical budget} with $g_0$ in place of $g$; the algorithm is run with the horizon $K=\max\{1,\lceil\tau^{\rightarrow}(\delta/2,\varepsilon)\rceil\}$ and this cap.
					The proof of Theorem~\ref{thm: micro stop} uses Assumption~\ref{assump: lip} and the radius $a$ only to establish $|b(\theta_k)|\geq g$ at every reached query with $|\theta_k-\theta^*|>\varepsilon$, through Lemma~\ref{lemma: convergence} inside the $a$-neighborhood and $\underline{\mu}_a$ outside it. Under the accuracy floor, $|b(\theta)|=2\tilde p(\theta)-1\geq g_0$ for every $\theta\neq\theta^*$, so this inequality holds with $g_0$ in place of $g$ at every inaccurate query, and the remaining steps of that proof (the boundary-crossing bound $\hbar_{m_p}/m_p\leq g_0/2$ from $\tau_{0,p}$, the one-sided Hoeffding bound with $m_p\geq8g_0^{-2}\log(2K/\delta)$, the union over the at most $K$ queries, and the final-output bound from Lemma~\ref{lemma: horizontal coupling}) apply. The comparison count is at most $Km_p=O\big(K\tau_{0,p}+Kg_0^{-2}\log(2K/\delta)\big)$ on every path; with $K=O(\log(1/\varepsilon)+\log(1/\delta))$ and fixed $\lowp,\eta$, $\tau_{0,p}$ is a constant, so $N=O\big(\log(1/(\delta\varepsilon))\cdot\log(\log(1/(\delta\varepsilon))/\delta)\big)=\widetilde O(\log(1/(\delta\varepsilon)))$.
					\QED
				\end{proof}

				\begin{proof}{Proof of Theorem~\ref{thm: total complexity}.}
					{}Let $L_k$ be the number of comparisons collected at query $k$, with $L_k=0$ if the query is not reached, so that $N=\sum_{k=0}^{K-1}L_k$.
					At every reached query the directional test either crosses its boundary by sample $m$ or the algorithm returns at sample $m$, so $L_k\leq m$ for every $k$, and unreached queries contribute $L_k=0$. Hence
					\begin{equation}\label{eq: bound in a2}
						N=\sum_{k=0}^{K-1}L_k\leq Km
					\end{equation}
					on every sample path; in particular the algorithm terminates almost surely and $\E[N]\leq Km=H(\delta,\varepsilon;a)$. By \eqref{eq: adaptive vertical budget},
					\[
					Km=O\left(K\tau_0+\frac{K}{g^{2}}\log\frac{2K}{\delta}\right),\qquad g=\min\{c\varepsilon^\kappa,\underline{\mu}_a\},
					\]
					and for fixed model parameters $\tau_0=\Theta(g^{-2}\log(1/g))$. Since $g^{-2}\leq c^{-2}\varepsilon^{-2\kappa}+\underline{\mu}_a^{-2}$, logarithmic horizontal budgets give $H(\delta,\varepsilon;a)=\widetilde O(\varepsilon^{-2\kappa})$.
					\QED
				\end{proof}

			\subsection{Proofs in Section \ref{sec: human-AI}}\label{appendix: }

			Before proving Theorem \ref{thm: lasso}, we first show the following result. It proves that, under mutual incoherence, the $\ell_\infty$-curvature condition implies a bound on $\|V_S^{-1}\|_\infty$, where $V_S=\frac{1}{n}\bX_S^\top\bX_S$ is the matrix appearing in the Lasso error bound.

				\begin{lemma}\label{lemma: bound inverse}
					Let $V=\frac{1}{n}\bX^\top\bX$ and $V_S=\frac{1}{n}\bX_S^\top\bX_S$. Suppose $\bX$ satisfies the mutual incoherence condition \ref{assum: MI} with $0\leq\alpha_1<1$, the lower-eigenvalue condition \ref{assum: LE} (so that $V_S$ is invertible), and the $\ell_\infty$-curvature condition
					\[
					\|V z\|_\infty \geq \gamma\|z\|_\infty \quad\text{for all } z \in C_{\alpha'}(S).
					\]
					Then
					\[
					\|V_S^{-1}\|_\infty \leq \frac{1}{\gamma}\,.
					\]
				\end{lemma}
				\begin{proof}{Proof of Lemma \ref{lemma: bound inverse}.}
					For any invertible matrix $A$,
					\begin{equation}\label{eq: matrix inverse}
						\|A^{-1}\|_\infty = \frac{1}{\min\{\|Ax\|_\infty:\|x\|_\infty=1\}}\,.
					\end{equation}
					Fix $z_S$ with $\|z_S\|_\infty =1$, let $z_{S^c}=0$, and set $z= (z_S, z_{S^c})$. Since $\|z_{S^c}\|_1=0$, we have $z\in C_{\alpha'}(S)$ for every $\alpha'\geq 0$. The vector $Vz$ has two blocks: its $S$-block is $(Vz)_S=V_S z_S$ and its $S^c$-block is $(Vz)_{S^c}=\frac{1}{n}\bX_{S^c}^\top\bX_S z_S$, so in general $\|Vz\|_\infty=\max\{\|V_Sz_S\|_\infty,\|(Vz)_{S^c}\|_\infty\}$ and the two norms need not coincide. Mutual incoherence controls the off-support block. For $j\in S^c$ let $w_j=(\bX_S^\top\bX_S)^{-1}\bX_S^\top x_j$, so that $\|w_j\|_1\leq \alpha_1$ by \ref{assum: MI} and
					\[
					\frac{1}{n}x_j^\top \bX_S z_S = w_j^\top \Big(\frac{1}{n}\bX_S^\top\bX_S\Big) z_S = w_j^\top V_S z_S\,.
					\]
					Hence
					\[
					\|(Vz)_{S^c}\|_\infty=\max_{j\in S^c}\big|w_j^\top V_S z_S\big|\leq \max_{j\in S^c}\|w_j\|_1\,\|V_S z_S\|_\infty\leq \alpha_1 \|V_S z_S\|_\infty\leq \|V_S z_S\|_\infty\,,
					\]
					and therefore $\|Vz\|_\infty=\|V_S z_S\|_\infty$. The curvature condition now gives
					\[
					\|V_S z_S\|_\infty = \|V z\|_\infty \geq \gamma\|z\|_\infty =\gamma
					\]
					for every $z_S$ with $\|z_S\|_\infty=1$, and Equation \eqref{eq: matrix inverse} yields
					\[
					\|V_S^{-1}\|_\infty \leq \frac{1}{\gamma}\,.
					\]
					\QED
			\end{proof}

			Now we are ready to prove Theorem \ref{thm: lasso}.

			\begin{proof}{Proof of Theorem \ref{thm: lasso}.}
				Corollary 7.22 in \cite{wainwright2019high} states that for any matrix $\bX$ that satisfies Assumption \ref{assump: three assumptions}\ref{assum: MI} (mutual incoherence) and \ref{assum: LE} (lower eigenvalue), then it satisfies the $\ell_\infty$-error bound
				\begin{equation}\label{eq: bound linfinite}
					\|\htheta_S-\theta^*_S\|_\infty \leq \frac{\sigma}{\sqrt{c_{\min}}}\left\{\sqrt{\frac{2\log s}{n}}+\zeta\right\}+\left\|\left(\frac{\bX_S^\top \bX_S}{n}\right)^{-1}\right\|_\infty\lambda_n,
				\end{equation}
				with probability at least $1-4 e^{-n\zeta^2/2}$. From Lemma \ref{lemma: bound inverse}, applied with $\gamma=\alpha_2$ under the mutual incoherence and lower-eigenvalue conditions of Assumption \ref{assump: three assumptions}, the $\ell_\infty$-curvature condition implies that
					\begin{equation}\label{eq: bound inverse}
						\left\|\left(\frac{\bX_S^\top \bX_S}{n}\right)^{-1}\right\|_\infty=\|V_S^{-1}\|_\infty \leq \frac{1}{\alpha_2}\,,
					\end{equation}
					which is exactly the factor multiplying $\lambda_n$ in \eqref{eq: bound linfinite}.
				Substituting \eqref{eq: bound inverse} into \eqref{eq: bound linfinite}, we conclude that
				\[
				\|\htheta_S-\theta^*_S\|_\infty \leq \frac{\sigma}{\sqrt{c_{\min}}}\left\{\sqrt{\frac{2\log s}{n}}+\zeta\right\}+\frac{1}{\alpha_2} \lambda_n.
				\]
				Theorem 7.21 in \cite{wainwright2019high} also implies no false inclusion. Thus we reach our conclusion.
				\QED
			\end{proof}

			\begin{proof}{Proof of Theorem~\ref{thm: complexity of HAI}.}
				{}Write $t=\varepsilon/s^{1/\varrho}$ and $q=\delta/s$.
				{}
					The width-dependent calibration of Theorem~\ref{thm: lasso} gives $\Pr(\cE_1)\geq1-\delta$. On $\cE_1$, $\hat S=S$ and $|\hat S|=s$, and for every $i\in S$,
					\[
					\min\{\theta_i^*-(\htheta_i^{(1)}-\lowbeta),
					(\htheta_i^{(1)}+\lowbeta)-\theta_i^*\}
					\geq\lowbeta-B_n\geq\lowbeta/4.
					\]
					The uniform prior on $\Theta_i$ therefore assigns mass $r/(2\lowbeta)$ to each one-sided $r$-neighborhood of $\theta_i^*$ for $0<r\leq\lowbeta/4$. In particular the prior parameters and the side-mass bounds stated in the deterministic calibration hold on $\cE_1$.

					Fix a true coordinate and condition on the history $\cG_i$ preceding its scalar run, on $\cE_1$. At any reached query with error greater than $t$, the absolute mean drift is at least $g=\min\{ct^\kappa,\underline{\mu}_a\}$: Lemma~\ref{lemma: convergence} gives $2\tilde p_i(\theta)-1\geq 2c|\theta-\theta_i^*|^\kappa\geq 2ct^\kappa$ for errors at most $a$, and the definition of $\underline{\mu}_a$ covers larger errors, since every point of $\Theta_i$ is within $2\lowbeta$ of $\theta_i^*$ on $\cE_1$. If the scalar run returns at its vertical cap, its empirical mean has absolute value less than $\hbar_{\bar m}/\bar m\leq g/2$. The conditional Hoeffding argument in the proof of Theorem~\ref{thm: micro stop} thus bounds an inaccurate return at any fixed query by
					\[
					\exp(-\bar m g^2/8)\leq\frac{\delta}{2s\bar K}.
					\]
					A union bound over the $\bar K$ possible queries gives at most $\delta/(2s)$ for an inaccurate early return. By construction $\bar K\geq\tau^{\rightarrow}(\delta/(2s),t)$ for the conditional scalar problem, whose prior constants are those defined in the calibration. The exact-query avoidance assumption holds conditionally on this history, so Lemma~\ref{lemma: horizontal coupling} bounds an inaccurate final horizontal output by $\delta/(2s)$ under the same conditional law. Thus each coordinate has conditional failure probability at most $\delta/s$, uniformly over preceding histories on $\cE_1$. Taking expectations and a union bound over the $s$ true coordinates gives
					\[
					\Pr\bigl(\cE_1\cap\{\|\widehat\theta-\theta^*\|_{\varrho}>\varepsilon\}\bigr)
					\leq\delta,
					\]
					because $s t^{\varrho}=\varepsilon^{\varrho}$ and all coordinates outside $S$ remain zero on $\cE_1$. Adding $\Pr(\cE_1^c)\leq\delta$ proves the accuracy claim.

					For sample complexity, consider every Stage-1 outcome. The algorithm performs one scalar run per selected coordinate, and each run uses the same precomputed budgets, so it uses at most $\bar K\bar m$ comparisons whether or not its interval contains the truth or its calibration conditions hold on that outcome. Consequently $N_2\leq|\hat S|\,\bar K\bar m$ pointwise, and on $\cE_1$, where $|\hat S|=s$, this is $s\bar K\bar m$; together with $N_1$ this gives \eqref{eq: uniform two stage cost} on $\cE_1$. On $\cE_1^c$ the count is still bounded, by $d\,\bar K\bar m$, since at most $d$ coordinates can be selected.

					Finally, $g^{-2}\leq c^{-2}t^{-2\kappa}+\underline{\mu}_a^{-2}$ and the burn-in in \eqref{eq: adaptive vertical budget} is $\widetilde O(g^{-2})$. Since $\bar K(q,t)=O(\log(1/t)+\log(1/q))$ by the threshold of Theorem~\ref{thm: complexity},
					\[
					s\bar K\bar m=\widetilde O\!\left(s\,\underline{\mu}_a^{-2}+c^{-2}s^{1+2\kappa/\varrho}\varepsilon^{-2\kappa}\right).
					\]
					Adding the prescribed label budget proves the rate.
				{}
				\QED
			\end{proof}

			\begin{proof}{Proof of Proposition \ref{prop: LNCA condition}.}
				First, if only utilizing Lasso for the supervised learning oracle to refine the estimator based on $n$ noisy labeled data, subject to certain regularity conditions, the estimation error can be bounded as  $\|\htheta_n-\theta^*\|_2\leq c\sigma\sqrt{s\log d/n}$ with high probability for some constant $c>0$ \citep{wainwright2019high}. Equivalently, to ensure that the two-norm distance of the estimation error is within $\varepsilon$, we need to acquire $O(\sigma^2 s\log d/\varepsilon^2)$ data points. Theorem \ref{thm: complexity of HAI} implies that for sufficiently small $\varepsilon$, the sample complexity scales with $s^{1+\kappa}/\varepsilon^{2\kappa}$.

				By comparing two terms,
				we have the condition of Algorithm \ref{Alg: SFT+FTPB} performing better than SL (ignoring the logarithm term) as
				\begin{equation*}
					\frac{\sigma^2 }{s^\kappa} \gtrapprox \varepsilon^{2-2\kappa}.
				\end{equation*}
				\QED
			\end{proof}

			\subsection{Calibration and proof for Algorithm~\ref{Alg: ACS}}\label{appendix: acs calibration}
			{}\paragraph{Prior, drift, and horizontal calibration.}
			{}
				On $\cE_1$, the prediction error at the center is at most $3w_k/4$, so $y_k^*$ lies at least $a_{0,k}>0$ from either endpoint. Thus the uniform prior on $I_k$ satisfies Assumption~\ref{assump: prior} with $c_{0,k}=1/(2w_k)$, $\varsigma=1$, and radius $a_{0,k}$; it has at least $1/8$ of its mass on each side of $y_k^*$. This uses the strict interior margin supplied by Stage~1, in addition to interval containment.

				For the absolute prediction error with a common admissible prediction separation $\Delta_y>0$ and $a\leq\min\{\min_ka_{0,k},\Delta_y/2\}$, the utility gap is $|D_{\Delta_y}(y)|=\min\{2|y-y_k^*|,\Delta_y\}$. Hence Assumption~\ref{assump: lip} holds with $\kappa=1$ and $\lambda_\Delta=2$ on the $a$-neighborhood of every $y_k^*$, Lemma~\ref{lemma: convergence} gives $c=\min\{1/(4\gamma\log2),\,1/(6a)\}$, and \eqref{eq: scalar drift identity} gives the outer drift $\underline{\mu}_a=\tanh(a/\gamma)$ at every $y$ with $|y-y_k^*|\geq a$; both are the same for every run and every Stage-1 realization in $\cE_1$. The uniform prior on $I_k$ has $c_{0,k}=1/(2w_k)\geq1/(2\max_kw_k)$, $\varsigma=1$, radius $a_{0,k}\geq a$, and side masses at least $1/8$, so the threshold of Theorem~\ref{thm: complexity} evaluated at $c_0=1/(2\max_kw_k)$ dominates that of every run, and the budgets \eqref{eq: acs calibrated budgets} are those of Section~\ref{subsec: SFT+FTPB} at these values. The input $x_k$ is held fixed within run $k$, comparison outcomes are conditionally independent given the past, and the candidate predictions must be realizable by admissible models. The horizontal guarantee of each run follows from Theorem~\ref{thm: complexity} and the conditional version of Lemma~\ref{lemma: horizontal coupling} under the exact-query avoidance assumption.

			\paragraph{Error propagation.}
				In Algorithm~\ref{Alg: ACS}, the estimate $\hat{\omega}_k$ depends on previous estimates $\hat{\omega}_1,\ldots,\hat{\omega}_{k-1}$ through the correction term $\sum_{j<k}\hat{\omega}_j\cdot\varphi(x_k)^\top\balpha_j$, so errors propagate across coordinates. The correct way to quantify this propagation is through the conditioning of the selected samples rather than through a count of coordinates. Let $\hat{\by}=(y(x_1),\ldots,y(x_s))^\top$ collect the refined predictions and $\by^*=\bZ\btheta^*$ the true ones, and suppose $|y(x_k)-\bz_k^\top\btheta^*|\leq\varepsilon_y$ for all $k$. The sequential updates of Algorithm~\ref{Alg: ACS} are exactly forward substitution for the triangular system obtained from $\bZ\btheta=\hat{\by}$ (orthogonalizing the rows of $\bZ$ is an LQ factorization of $\bZ$, equivalently a QR factorization of $\bZ^\top$), so $\hat{\btheta}=\sum_k\hat\omega_k\balpha_k=\bZ^{-1}\hat{\by}$ and
				\[
				\|\hat{\btheta}-\btheta^*\|_2=\|\bZ^{-1}(\hat{\by}-\by^*)\|_2\leq\frac{\|\hat{\by}-\by^*\|_2}{\sigma_{\min}(\bZ)}\leq\frac{\sqrt{s}\,\varepsilon_y}{\sigma_{\min}(\bZ)}\,.
				\]
				Refining each prediction to $\varepsilon_y=\varepsilon\,\sigma_{\min}(\bZ)/\sqrt{s}$ guarantees $\|\hat{\btheta}-\btheta^*\|_2\leq\varepsilon$, and substituting $\varepsilon_y$ into the vertical complexity $\widetilde O(\varepsilon_y^{-2\kappa})$ of Section~\ref{sec: human} multiplies the comparison cost of each coordinate by $\big(\sqrt{s}/\sigma_{\min}(\bZ)\big)^{2\kappa}=s^{\kappa}\,\sigma_{\min}(\bZ)^{-2\kappa}$.

			\begin{proof}{Proof of Proposition~\ref{prop: acs guarantee}.}
					The horizontal budget satisfies $K_y=O(\log(1/t_y)+\log(s/\delta))$ by the threshold of Theorem~\ref{thm: complexity}, which gives the stated order of $sK_ym_y$; for a data-dependent sample matrix this cost bound is conditional on the chosen calibration.
					{}This substitution describes the regime $\varepsilon_y\leq a$; Algorithm~\ref{Alg: ACS} otherwise uses the more accurate internal target $t_y=a$. In the former regime the factor $s^\kappa$ is the dimension factor for $\varrho=2$, and $\sigma_{\min}(\bZ)^{-2\kappa}\leq\sigma_0^{-2\kappa}$ accounts for conditioning. In both regimes, applying the conditional scalar guarantee with the calibration above and taking a union bound gives
						\[
						\Pr\!\left(\cE_1\cap
						\left\{\max_{1\leq k\leq s}|y(x_k)-y_k^*|>t_y\right\}\right)\leq\delta.
						\]
						On the complementary event within $\cE_1$, the displayed inverse-matrix bound gives $d_2(\hat{\btheta},\btheta^*)\leq\sqrt{s}\,t_y/\sigma_{\min}(\bZ)\leq\varepsilon$. Since $\Pr(\cE_1^c)\leq\delta$ for $N_1=H_0(\delta,\varepsilon;\lowbeta)$ by Section~\ref{subsec: SFT+FTPB}, the overall failure probability is at most $2\delta$, including the failure route for the sample-set search whenever availability is part of $\cE_1$. Under the well-conditioned pool assumption the exhaustive fallback finds a passing set before comparisons.

						The comparison count is at most $sK_y m_y$ on every path after the budgets have been fixed, including failed Stage-1 outcomes; finding the sample set uses covariates and no evaluator comparisons. With logarithmic $K_y$, the corresponding bound is
						\[
						sK_y m_y=\widetilde O\!\left(s\,\underline{\mu}_a^{-2}+s\,c^{-2}t_y^{-2\kappa}\right).
						\]
						For a data-dependent selected matrix, this is a bound conditional on the chosen calibration; an unconditional deterministic rate requires common budgets uniform over the possible matrices and Stage-1 histories, as in Section~\ref{subsec: SFT+FTPB}. Without that additional uniformity the unconditional cost bound is $\E[N_{\rm ACS}]\leq\E[sK_y m_y]$. %
					{}
					\QED
			\end{proof}

			\section{Numerical illustrations for Example \ref{example: pricing}}\label{appendix: pricing}

			In this experiment, each coordinate of $X_i\in \mathbb{R}^{1000}$ is generated from Gaussian distribution with mean 0 and standard deviation 100. To reflect our sparse setup, we set $\theta_1= -0.5$, $\theta_2 = 5$, and the remaining coefficients to 0. The observational noise follows Gaussian distribution with mean 0 and standard deviation 200. By training on a dataset comprising 2000 data points, we obtained estimated parameter $\hat{\theta}_1=0.47$ and $\hat{\theta}_2 = 5.87$. It is noteworthy that the estimated sign of the first parameter is incorrect, indicating a potential issue in pure supervised learning.

			\section{Synthetic Data Experiments}\label{appendix: synthetic}

			In the first experiment, we compare the performance of \textsf{SL+LHF} and \textsf{SL} using the same sample size to assess how parameter variations affect the performance of the two algorithms.

			\emph{Experiment setup.} Consider a high-dimensional linear regression problem where $d=100$ and $s=10$. The first ten coordinates of the true parameter are independently drawn from the uniform distribution, and the rest of the coordinates are zero. The observational noise is drawn from Gaussian distribution with a mean of zero and variance $\sigma\in\{1,2,5\}$; we choose $\kappa\in [0,0.95]$, where $\kappa$ comes from Assumption \ref{assump: lip}. Specifically, the difference of the utility in the two answers $c_\Delta^+(\theta)$ and $c_{\Delta}^-(\theta)$ is assumed to be $\|\theta-\theta^*\|^\kappa$. The selection precision (along each dimension) is set at $\varepsilon=0.1$, and the human expertise level in the choice model is $\gamma=1$. We conducted 30 repetitions for each experiment.

			In the first stage of the two-stage framework, we selected a dataset with a sample size proportional to $\sigma^2\log d$ to input into Lasso. According to Theorem \ref{thm: lasso}, this selection ensures that important features are identified with high probability. We chose the regularization parameter using cross-validation. The second stage is then executed, requiring human comparison samples, to reach the desired precision threshold $\varepsilon$. Once the \twostage{} algorithm completes, the combined sample size for both stages is recorded as the sample size for \textsf{SL+LHF}. Then, for a fair comparison, we provided the same sample size to \textsf{SL} (to train a lasso with the same number of  observations) and computed the estimation error. We then looked at the ratio of the estimation errors of the two approaches, denoted by Err(\twostage{})/Err(\textsf{SL}).

			\subsection{Experiment results.}
			Figure \ref{fig:experiment_sigma1} shows the median of the ratio Err(\twostage{})/Err(\textsf{SL}) using different values of $\kappa$ when $\sigma$ takes values $1$, $2$, or $5$. For all values of $\sigma$, when $\kappa$ is smaller, the human's selection accuracy is higher as $\theta$ approaches $\theta^*$. In particular, when $\kappa=0$, the selection accuracy is bounded away from 50\%. Therefore, the task is harder as $\kappa$ increases. It is consistent with the increasing trend in the estimation error ratio as $\kappa$ grows from 0 to 1. In addition, a larger value of $\kappa$ leads to a larger sample size, which leads to a decrease in the estimation error of lasso. When $\sigma=1$, the intersection point of the curve, with the horizontal line of the ratio equal to $1$ (red dashed line), is around $\kappa=0.2$. That is, when $\kappa\leq 0.2$, \twostage{} outperforms pure \textsf{SL} in more than half of the cases. In particular, when $\kappa=0$, the median of the error ratio is close to 0.

			\begin{figure}[!th]
				\centering
				{\includegraphics[width = .8\linewidth]{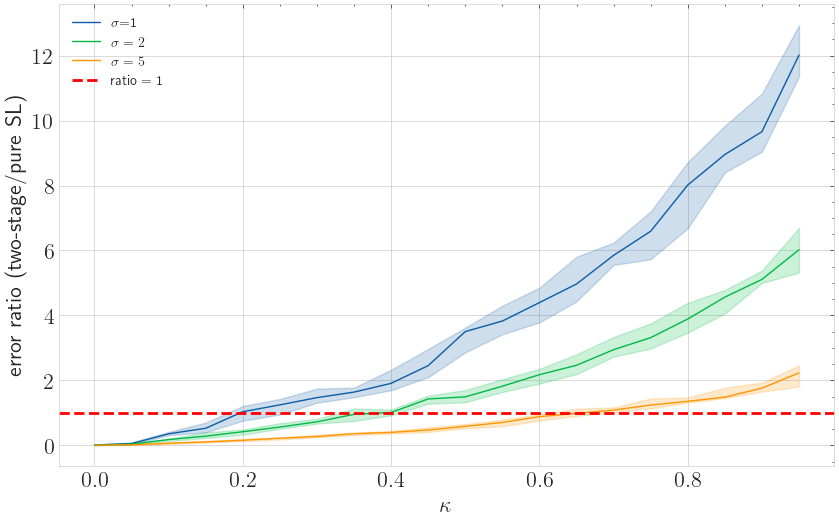}}
				\caption{Error ratio of the two-stage framework to the pure SL with different values of $\sigma$.}
				\label{fig:experiment_sigma1}
			\end{figure}

			The results for noise levels with $\sigma=2$ and $\sigma=5$ are also illustrated in Figure \ref{fig:experiment_sigma1}. Notably, the trends in the error curves are  similar. The threshold, defined as the intersection point with the red dashed line, is approximately 0.4 when $\sigma=2$ and approximately $0.65$ when $\sigma=5$.

			Overall, the three curves  show that the superiority of \twostage{} becomes more pronounced in the presence of higher observational noise. In conclusion, \twostage{} is more effective under increased noise or reduced $\kappa$.

			\subsection{The variation of expertise level $\gamma$.}  We vary the parameter $\gamma$ to model different levels of human expertise and to analyze how the accuracy of human comparisons affects the estimation error. For the experiment setup, we set $\sigma=2$ and $s=10$. The smaller value of $\gamma$ indicates higher accuracy in the comparison task. In particular, as $\gamma$ approaches 0, the human can always select the accurate answer. Figure \ref{fig:experiment_gamma} shows the performance of the two algorithms with $\gamma$ equal to $0.8$, $1$, and $1.1$. Comparing the three curves, when $\gamma=1.1$, \twostage{} performs better than pure \textsf{SL} if $\kappa$ is lower than approximately 0.3; when $\gamma=0.8$, \twostage{} performs better than pure \textsf{SL} if $\kappa$ is lower than approximately 0.5. Therefore, we find support for our theory that the performance of \twostage{} improves  when the human has a higher expertise level in the selection (i.e., the selection accuracy is higher).

			\begin{figure}[!th]
				\centering
				{\includegraphics[width = .8\linewidth]{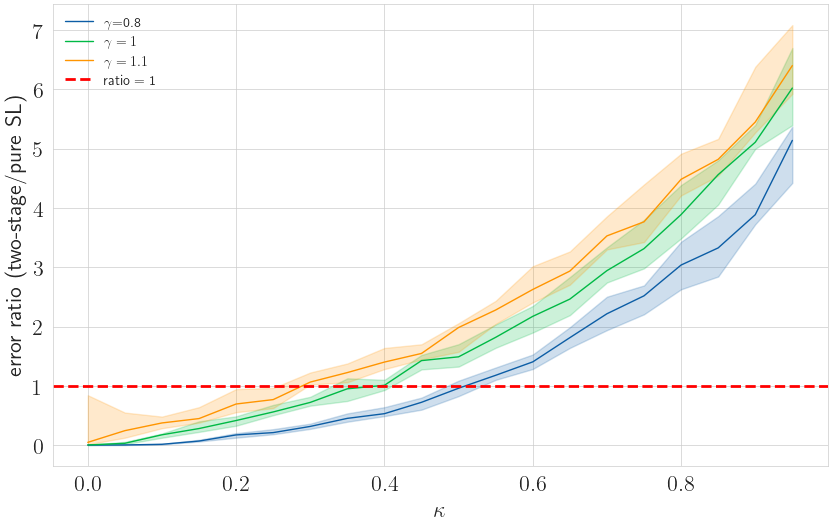}}
				\caption{Error ratio of the two-stage framework to the pure SL with different values of $\gamma$.}
				\label{fig:experiment_gamma}
			\end{figure}

			\subsection{The variation of important feature dimensions $s$.} The relative performance of two algorithms also depends on the dimension of the problem. For the experiment setup, we set $\gamma=1$ and $\sigma=2$. Figure \ref{fig:experiment_sigma2_gamma1_s20}  illustrates the algorithmic performance when important feature dimension $s$ takes values $10$, $20$, and $50$.  The figure reveals that the intersection point shifts to a smaller value as the important feature dimension value increases. This observation aligns with our theoretical condition \eqref{eq: compare condition}, indicating that as the value of the important feature dimension grows, the advantage of \twostage{} diminishes. This result highlights the importance of low-dimensional representations in the success of this method.

			\begin{figure}[!th]
				\centering
				{\includegraphics[width = .8\linewidth]{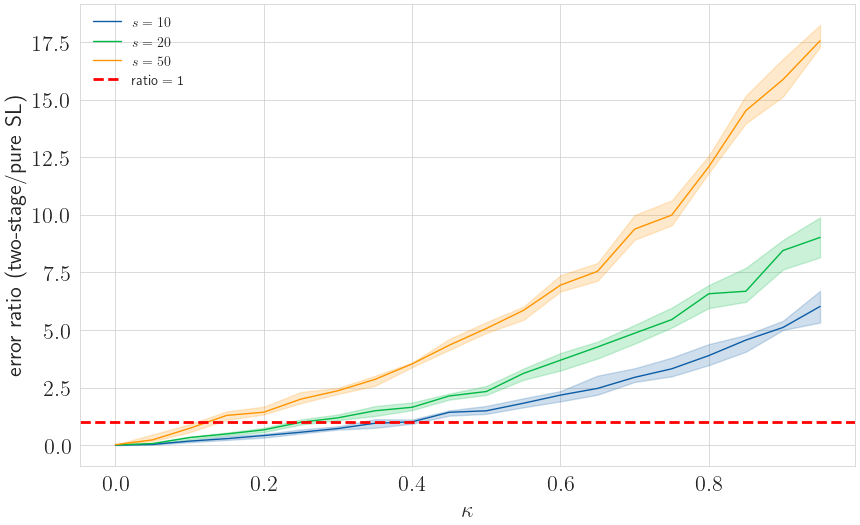}}
				\caption{Error ratio of the two-stage framework to the pure SL with different values of $s$.}
				\label{fig:experiment_sigma2_gamma1_s20}
			\end{figure}

			\section{MTurk Experiment on the LNCA Ratio}\label{appendix: mturk}

			In this section, we first outline the setup of our MTurk experiment\footnote{https://www.mturk.com/}. We then describe the estimation process for key parameters, including the accuracy convergence rate $\kappa$, expertise level $\gamma$, and noise level $\sigma$. After these parameters were estimated, we conducted simulations to compare the performance of the two-stage framework with that of the pure \textsf{SL} approach.

			We designed a task in which participants are asked to estimate the number of points within a square.
			We divided participants into two distinct groups. The first group was presented with one of the squares and two options for the estimated number of points. The second group was shown the same square and asked to provide their estimated number of points.  (See Figures \ref{fig:experiment_scenario1} and \ref{fig: Comparison vs. Estimation}.) The task for the first group was to select the option they believed was closer to the actual number. For instance, in Figure \ref{fig:experiment_scenario1}(b), the true number is 32. Human participants were provided with two options: 55 and 35; the correct choice is 35.

			With this dual-group approach, we aimed to explore the efficacy of human estimation when participants were asked to estimate the number of points versus to choose from comparative options. In the latter case, we varied the discrepancy between the two choices to assess how the ease or difficulty of comparison affected the accuracy of participants' selections. A larger disparity between the choices facilitated easier decision-making, while closer options posed a more challenging comparative task. Figure \ref{fig:experiment_scenario1} shows three testing samples. Notably, sample (c) presents the most challenging comparison because of the minimal disparity ($|u(c^+)-u(c^-)|)$ observed between them.\footnote{The gap between the first choice (45) and the true value (36) is 9 while the gap between the second choice (25) and the true value is 11. Thus, the difference between 9 and 11 is only 2.}

			To compare the performance of the two-stage framework and the pure supervised learning, we first estimated $\kappa$ and $\sigma$ from the online experiment in Amazon MTurk. Then we ran the simulation to compare two methods under different precision levels. We launched 50 different tasks, all structured in the same way as the task shown in Figure \ref{fig: Comparison vs. Estimation}, and collected 500 data points in total. We estimated $\kappa$ and $\sigma$ as described in the following.

			\begin{figure}[!th]
				\centering
				\subfigure[96 vs. 76. True: 64.]
				{\includegraphics[width = .3\linewidth]{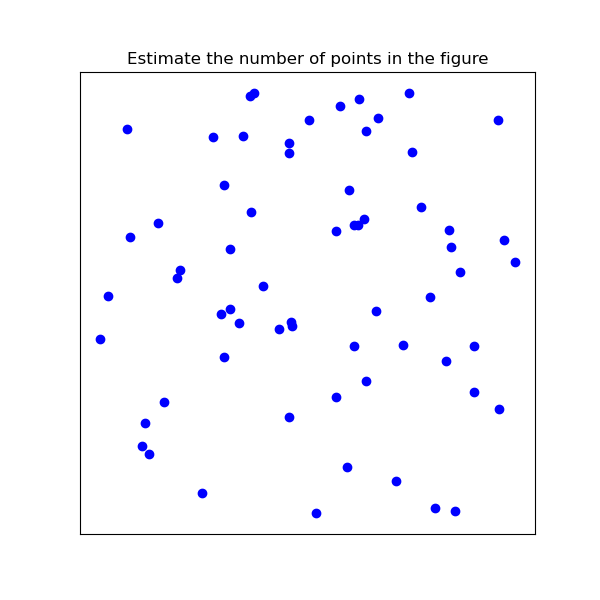}}
				\subfigure[55 vs. 35. True: 32.]
				{\includegraphics[width = .3\linewidth]{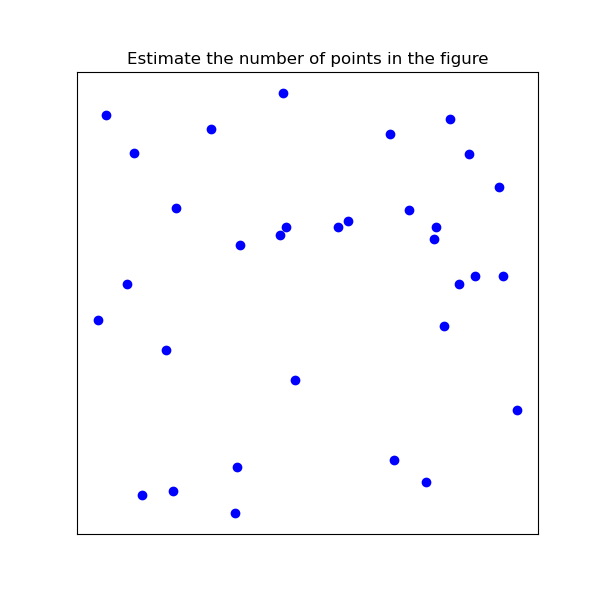}}
				\subfigure[45 vs. 25. True: 36.]
				{\includegraphics[width = .3\linewidth]{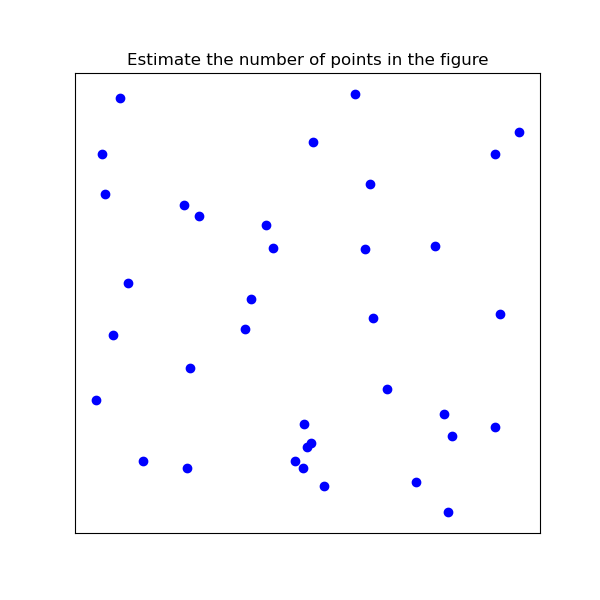}}
				\caption{Experiment design for Group 1.}
				\label{fig:experiment_scenario1}
			\end{figure}

			\begin{figure}[!th]
				\centering
				\subfigure[Estimate the number of points.]
				{\includegraphics[width = .5
					\linewidth]{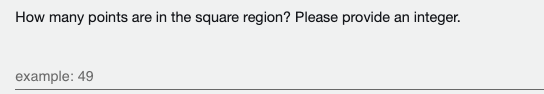}}
				\caption{Experiment design for Group 2.}
				\label{fig: Comparison vs. Estimation}
			\end{figure}

			\emph{Estimation process for $\kappa$.}
			Given figures with true parameter value $\theta^*$
			(i.e., the number of points in the figure), we estimate the accuracy convergence rate $\kappa$ between two options using a human choice model. We collect multiple samples for each figure and use these estimates to derive the relationship between the difference in  utility and the accuracy of human choices.

			Specifically, among two choices, assume the smaller one is $c^-$, the larger one is $c^+$, and the middle point is $\theta =(c^-+c^+)/2 $.
			According to the choice model, the human makes the correct choice with accuracy $p(\theta)=\frac{1}{1+\exp(-|u(c^-)-u(c^+)|/\gamma)}$. Note that we try to estimate $\kappa$ from Assumption \ref{assump: lip} that
			\[
			|u(c^+)-u(c^-)|/\gamma \approx \lambda_\Delta/\gamma\cdot\|\theta-\theta^*\|^\kappa\,.
			\]
			That is, the accuracy is $p(\theta)=\frac{1}{1+\exp(-\lambda_\Delta/\gamma \cdot |\theta-\theta^*|^\kappa)}$. Let $\tilde{\lambda}=\lambda_\Delta/\gamma$. For each data point with tested point $\theta_i$ and true parameter $\theta_i^*$, we define $y_i=1$ (the correct choice) with probability $p_i(\theta)$, and we define $y_i=0$ (the wrong choice) with probability $1-p_i(\theta)$. By maximizing the log-likelihood, we solve the optimization problem
			\[
			\max_{\tilde{\lambda}, \kappa}\quad \sum_{i=1}^K y_i\log p_i + (1-y_i)\log(1-p_i),
			\]
			which is equivalent to
			\[
			\max_{\tilde{\lambda}, \kappa} \quad \sum_{i=1}^K -y_i\log(1+\exp(-\tilde{\lambda}|\theta_i-\theta^*|^\kappa))-(1-y_i)\log(1+\exp(\tilde{\lambda}|\theta_i-\theta_i^*|^\kappa))\,.
			\]
			We solve this optimization problem simply by enumerating different values of $\tilde{\lambda}$ and $\kappa$.

			\emph{Estimation process for $\sigma$.} The other group is asked to directly estimate the percentage of total area. For the figure with true parameter values $\theta_i^*$, the human provides responses by $y_i=\theta_i^*+\epsilon_i$. We estimate the variance of the noise by $\hat{\sigma}^2=\sum_{i=1}^N \frac{(y_i-\theta_i^*)^2}{N-1}$.

			\emph{Empirical result.} We applied a bootstrap resampling method to the data collected from Amazon MTurk, generating 500 resamples to estimate key parameters, including $\kappa$, $\tilde{\lambda}$, and $\sigma$. The resulting mean estimates are $\kappa = 0.328$, $\tilde{\lambda} = 0.132$, and $\sigma^2 = 29.665^2$. The variances are 0.061, 0.006, and 4.600 for $\kappa$, $\tilde{\lambda}$, and $\sigma$, respectively.

			Figure \ref{fig:error_empirical_AMT} presents the median empirical error of the two-stage framework across varying precision levels, based on 100 repetitions for each value of $\varepsilon$. The results demonstrate that, for any given precision level, the empirical error remains within the targeted threshold. In addition, an upward trend in the empirical error is observed as the precision level increases. The green plot illustrates the sample size required to achieve a specified precision, showing a rapid decline in sample size as precision increases.

			For a clearer comparison with the pure \textsf{SL}, Figure \ref{fig:error_ratio_AMT} displays the error ratio between the two-stage framework and the pure \textsf{SL}. When $\varepsilon=0.3\%$, the median of the estimation error of the two-stage framework is approximately 16\% of the pure \textsf{SL}. Increasing the precision to $\varepsilon=4.3\%$, the median of the estimation error of the two-stage framework is approximately 45\% of the pure \textsf{SL}. This trend of a decreasing error ratio with tighter precision aligns with Proposition \ref{prop: LNCA condition}, highlighting the increasing advantage of the two-stage framework as $\varepsilon$ approaches 0.

			\emph{Comparison versus estimation, head to head.} The dual-group design lets us contrast the two data types directly, at matched task difficulty. Plugging the estimated primitives (label noise $\sigma\approx 29.7$, comparison-accuracy parameters $\kappa\approx 0.33$ and $\tilde{\lambda}\approx 0.13$) into the label-noise-to-comparison-accuracy ratio $\sigma^2/s^\kappa$ places this task well inside the region in which comparisons are the more efficient data type (Proposition~\ref{prop: LNCA condition}). The same contrast is visible operationally in Figure~\ref{fig:error_ratio_AMT}: under a matched query budget, the two-stage framework built on the \emph{comparison} group reaches an estimation error that is only $16\%$ to $45\%$ of what the \emph{estimation} group attains with direct labels, across precision levels. Reading this as an \emph{effective} $\sigma$, and using the fact that supervised error scales as $\sigma/\sqrt{n}$, matching the comparison group's precision would require the estimation group to collect roughly $5\times$ to $40\times$ as many direct labels; in this task, then, a single comparison is worth several to an order of magnitude more direct estimates. This is precisely the dual-group efficacy contrast, and it shows that the advantage of comparisons here is large and not an artifact of the modeling assumptions.

			\begin{figure}[!th]
				\centering
				{\includegraphics[width = .8\linewidth]{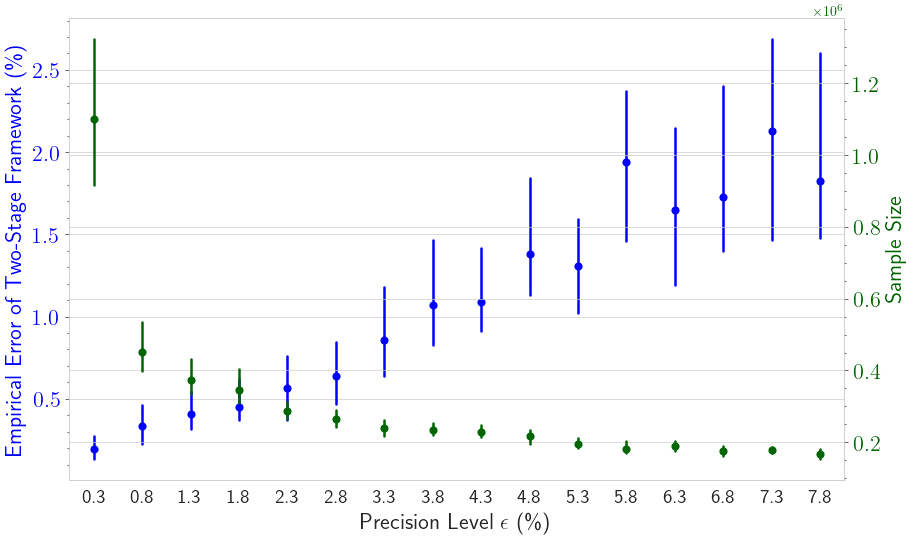}}
				\caption{Empirical error of the two-stage framework for different values of $\varepsilon$.}
				\label{fig:error_empirical_AMT}
			\end{figure}

			\begin{figure}[!th]
				\centering
				{\includegraphics[width = .8\linewidth]{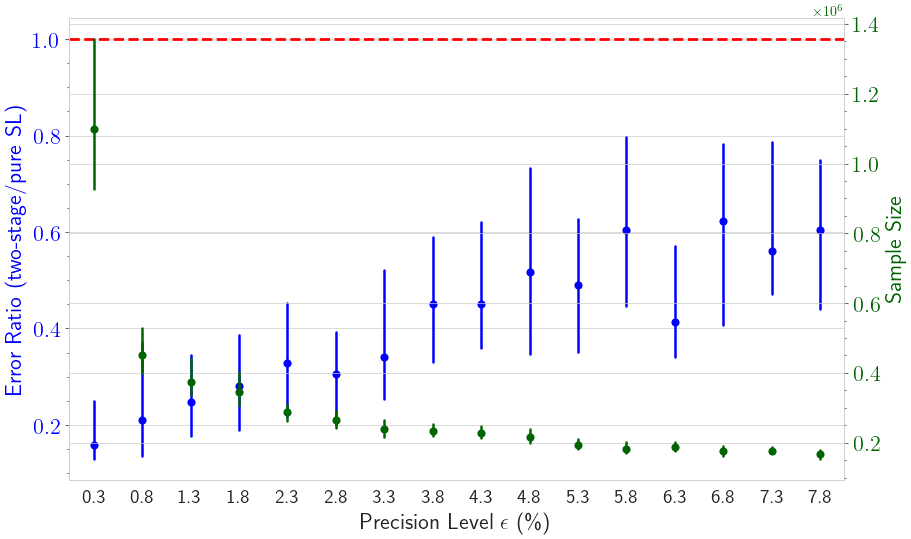}}
				\caption{Error ratio of the two-stage framework to the pure SL for different values of $\varepsilon$.}
				\label{fig:error_ratio_AMT}
			\end{figure}

			In conclusion, we acknowledge limitations in our numerical experiments. Ideally, our two-stage framework could be used to acquire responses sequentially from freelancers, first asking them which number is closer and and then asking them for an estimate of the number of points. Then the two results could be compared. However, implementing this comparison would be challenging because of the adaptive nature of our acquisition process, where each subsequent question depends on previously collected answers. Moreover, acquiring responses from the Amazon MTurk platform is costly. Given the complexity and expense of this procedure, we opted to use simulations to evaluate the performance of our framework.

			\section{Implementation Details of the Kickstarter Experiment}\label{appendix: kickstarter details}
			{}This appendix documents the data construction, the belief that models the evaluator's knowledge, and the exact mechanics of the simulated comparator and the refinement procedure used in Section~\ref{subsec: kickstarter}.

				\subsection{Data construction}
				Web Robots publishes monthly scrapes of the Kickstarter site; we use the April 2026 snapshot. We keep campaigns that have concluded (state \emph{successful} or \emph{failed}), were launched between January 2024 and April 2026, have a USD-converted goal between \$100 and \$1{,}000{,}000, and have a funding ratio ($\text{USD pledged}/\text{goal}$) between $0.01$ and $100$; the ratio bounds trim degenerate goals and viral outliers, and cap the target $y=\log(\text{USD pledged}/\text{goal})$ to $[-4.6,4.6]$. This yields $71{,}023$ campaigns, split uniformly at random into $80\%$ training ($n_{\text{train}}=56{,}818$) and $20\%$ test ($n_{\text{test}}=14{,}205$). The $d=9{,}828$ covariates come in four blocks: (i) $406$ listing attributes: log goal, campaign duration, preparation time between creation and launch, launch year, month, and day of week, title and blurb lengths, and one-hot encodings of $24$ countries, $15$ parent categories, $159$ subcategories, and location states; (ii) $123$ market covariates evaluated on the launch date: equity indices, sector ETFs, Treasury yields, exchange rates, the VIX, and commodity prices, capturing the macro environment a campaign launches into; (iii) $7{,}000$ TF--IDF text features ($5{,}000$ blurb terms and $2{,}000$ title terms), with the vocabulary fitted exclusively on campaigns from $2023$, which predate both splits, so no test text influences the representation; and (iv) $2{,}299$ category-by-country interaction indicators.

				\subsection{The belief and the simulated comparator}
				The evaluator's knowledge is represented by a \emph{belief} $b(\cdot)$: a cross-validated Lasso fit on all $56{,}818$ training campaigns (the test set is excluded, so the belief is a genuinely held-out predictor; its test $R^2$ of $0.459$ is the performance of this fitted teacher and serves as a benchmark; a single Lasso fit does not establish the best possible linear predictor). The belief plays the role of the internal assessment an expert consults when judging a prediction: it is far stronger than any model fit from the small labeled budgets $N$ in Table~\ref{tab: kickstarter parity} (whose test $R^2$ never exceeds $0.25$), yet imperfect, exactly the regime in which comparisons carry signal.

				The comparator instantiates the random utility model of Section~\ref{sec: model} with the belief standing in for the truth. When the refinement, moving along a direction $v$ from the current iterate, must choose between two candidate steps $t-\Delta/2$ and $t+\Delta/2$ (granularity $\Delta=0.2$), the evaluator receives a batch of $B=8$ projects, selected afresh for every comparison as those with the largest projections on the current direction within a random subsample of $200$ pool projects (a batch-selection heuristic), and assesses which candidate model predicts closer to its belief on that batch: the utility of a candidate is the negative mean squared distance between the candidate's predictions and $b(\cdot)$ over the batch, a batch analogue of the prediction discrepancy of Section~\ref{subsec: utility model}, and the Gumbel noise of \eqref{eq: rum} makes the evaluator select the better candidate with probability $\big(1+\exp(-\Delta'/\gamma)\big)^{-1}$, where $\Delta'$ is the batch mean-squared-error gap, per Lemma~\ref{lemma: precision}; with this utility the Bradley--Terry log-odds used by the DPO baseline are exactly linear in the parameter. The expertise parameter is $\gamma\in\{0.25, 0.01\}$. Note the two mechanisms this construction realizes: the accuracy is high when the two candidates are well separated relative to the belief and decays toward $1/2$ as the iterate closes in on the belief's own view (the $\kappa$ regime of Assumption~\ref{assump: lip}), and no aspect of the response uses the test labels.

				\subsection{The refinement procedure}
				Stage~1 fits a cross-validated Lasso on the $0.9N$ labeled campaigns and fixes the selected support; Stage~2 refines only coordinates in this support, using a pool of $4{,}000$ unlabeled training campaigns for the comparisons. %
					The step size is located by the probabilistic bisection of Section~\ref{subsec: PB} over the range $t\in[-0.6,0.6]$, each comparison outcome $Z=\pm1$ updates the posterior, and the committed step is the posterior median. The Stage-1 support has mean size $30$, $32$, $36$, $48$, $49$, $67$ at $N=100,\ldots,500$ (range $21$--$129$). {} %
				All results are means over $30$ independent repetitions; each repetition draws a fresh random permutation of the training campaigns, takes the first $N$ as the query budget (used by both pure \textsf{SL} and SL+LHF, so the comparison is paired) and the next $4{,}000$ as the pool, and all $R^2$ are computed on the fixed $14{,}205$ held-out test campaigns. The LLM experiment of Section~\ref{appendix: kickstarter llm} reuses this procedure, replacing only the simulated response $Z$ by the majority of the model's per-project choices and widening the granularity to $\Delta=0.6$.

		\end{APPENDICES}

	\end{document}